%% file: main.tex
\documentclass{article}

\usepackage{mlmodern}
\usepackage[T1]{fontenc}

\usepackage{hyperref}
\usepackage{url}
\usepackage{mathtools}
\usepackage{amsmath, amssymb, amsfonts, bm, amsthm, bbm,xfrac}
\usepackage{thmtools,thm-restate}
\usepackage{xcolor}
\usepackage{cleveref}
\usepackage{comment}
\usepackage{pgfplots}
\usepackage{enumitem}
\usepackage{authblk}
\usepackage[margin=1in]{geometry}
\usepackage{natbib}
\usepackage{placeins}
\setcitestyle{authoryear,round,citesep={;},aysep={,},yysep={;}}

\hypersetup{
    colorlinks = true,
    linkcolor = {blue},
    citecolor = {blue},
    urlcolor = {blue}
}

\numberwithin{equation}{section}
\theoremstyle{definition}
\newtheorem{theorem}{Theorem}
\newtheorem{proposition}{Proposition}
\newtheorem{lemma}{Lemma}[section]
\newtheorem{corollary}{Corollary}
\newtheorem{definition}{Definition}
\newtheorem{assumption}{Assumption}[section]

\newtheorem{remark}{Remark}
\newtheorem{conjecture}{Conjecture}
\DeclareMathOperator*{\argmin}{arg\,min}

\newcommand{\dd}[0]{\rm{d}}
\newcommand{\Tr}[0]{{\rm Tr}}
\newcommand{\He}[0]{{\rm He}}
\newcommand{\RR}[0]{{\mathbb{R}}}

\newcommand{\ZZ}[0]{{\mathbb{Z}}}
\newcommand{\NN}[0]{{\mathbb{N}}}

\author[1,2,3]{Lorenzo Rizzi}
\author[1]{Arie Wortsman Zurich}
\author[1]{Bruno Loureiro}

\affil[1]{\small Departement d'Informatique, \'Ecole Normale Sup\'erieure, PSL \& CNRS}
\affil[2]{\small Dipartimento di Fisica e Astronomia, Universit{\`a} degli Studi di Padova}
\affil[3]{\small Scuola Galileiana di Studi Superiori, Universit{\`a} degli Studi di Padova}

\title{Learning between the peaks: sharp asymptotics for kernel ridge regression under power-law anisotropy}

\date{}

\begin{document}

\maketitle

\begin{abstract}
We study kernel ridge regression under anisotropic Gaussian data, where the input covariance decays as a power law with exponent $\alpha\geq 0$ for polynomial inner-product kernels. We derive asymptotically sharp expressions for the kernel spectrum and the generalization error in the polynomial high-dimensional regime $n=\Theta(d^\kappa)$, revealing how anisotropy reshapes the learning curves. For weak anisotropy ($0<\alpha<1$), the problem remains effectively high-dimensional and retains some features of the isotropic case, while departing from it in others: the variance still peaks at integer sample complexities $\kappa\in\mathbb{N}$, but these peaks are progressively damped as $\alpha$ grows; meanwhile, for targets strongly aligned with the data's principal directions, the bias drops at fractional sample complexities, decoupling the bias transitions from the interpolation peaks. For strong anisotropy ($\alpha > 1$), the effective dimension of the problem is constant, and the variance stops depending on sample size altogether, plateauing under ridgeless interpolation or vanishing at an explicit rate under fixed ridge penalty. The bias undergoes a sharp transition governed by the target's decay rate: below a threshold, learning is abrupt rather than gradual; above it, the bias decays as a power law that recovers the classical source and capacity rates. We finally specialize these results to single-index targets, showing how the alignment of the index with the data's principal directions determines the effect of anisotropy on learning. Together, our results clarify how the input geometry shapes the kernel features and fundamentally impacts its generalization properties.
\end{abstract}

\input{1_Introduction}

\input{2_Setting}
\input{3_Spectrum}
\input{4_1_Weakly_Anisotropic_Regime}
\input{4_2_Strong_Anisotropic_Regime}
\input{5_Single_Index_Models}
\input{6_Conclusion}

\bibliography{references}
\clearpage
\clearpage
\bibliographystyle{abbrvnat}

\appendix
\input{Appendix/Kernel_operator_properties}
\input{Appendix/Deterministic_Equivalents}

\input{Appendix/Reparametrization_kernel_eigenvalues}
\input{Appendix/Weak_Anisotropy}
\input{Appendix/Strong_anisotropy}
\input{Appendix/Single_Index_Models}
\input{Appendix/HermiteAnsatzNumerical}
\input{Appendix/Technical_Lemmas}

\end{document}

%% file: 1_Introduction.tex
\section{Introduction \& motivation}

Kernel methods are a classical and widely used class of non-parametric statistical models. They allow for complex non-linear modelling by implicitly mapping the input data into a (possibly) infinite-dimensional Hilbert space, where the learning task is reduced to a convex problem.

Over the last decade, kernel methods have attracted renewed attention in the context of deep learning theory, following the discovery that sufficiently wide neural networks trained in the so-called lazy regime behave, to leading order, as kernel methods governed by the neural tangent kernel \citep{jacot_neural_2020, chizat2019lazy}. This correspondence established kernel methods as a tractable proxy for studying the mathematical properties of neural networks, and a substantial part of our current theoretical understanding of neural networks builds on this connection \citep{belkin_understand_2018, bartlett_2020, ghorbani2020neural, mallinar2022benign}.

More recently, this connection has gained further relevance through the study of neural scaling laws, the empirical observation that the performance of large-scale networks improves in a predictable, power-law fashion as a function of dataset size, number of parameters, and compute \citep{kaplan2020scaling, hoffmann2022training}. Scaling laws for kernel methods, by contrast, have a considerably longer history. In this context, the rate in which the excess risk of kernels vanish is known to follow a power-law in sample size, with the exponent controlled by the regularity of the task and underlying RKHS  \citep{caponnetto_optimal_2007}. This classical body of results has motivated a renewed interest in kernel methods as tractable models for studying neural scaling laws \citep{bahri2024explaining, maloney2022solvable, defilippis_dimension-free-2024, atanasov2026scaling}. A major limitation in this emerging literature, however, is that the power-law behavior of kernel features are either assumed (so-called \emph{capacity condition}) or derived in linear models. Understanding how the data geometry shapes the scaling behavior of the features is a question that only very recently started to be addressed \citep{wortsman_kernel_2025, karkada2025predicting, paquette2026power}, with the consequences for learning remaining a largely open question. In this work, we build on this recent progress to provide a sharp description of the generalization properties of kernel ridge regression under anisotropic input data. More precisely, our \textbf{main contributions} are:
\begin{itemize}
    \item \textbf{Sharp characterization.} We first show that for polynomial inner-product kernels under power-law anisotropic data, the high-dimensional kernel spectrum can be made asymptotically exact, and that this yields closed-form deterministic equivalents for the bias and variance of the kernel ridge estimator in the polynomial high-dimensional regime.
    % \item \textbf{Weak anisotropic}. For weakly anisotropic data, the problem remains effectively high-dimensional. We derive a closed-form characterization of the bias and variance through an explicit self-consistent equation in this regime. We show the variance exhibits the multiple-descent pattern known for isotropic data, but its peaks are systematically dampened by anisotropy. The bias displays a staircase pattern corresponding to a sequential learning of different target degrees. When the target decay is aligned with the data's principal directions, its features become learnable at fractional sample complexities that fall between the variance's integer peaks, decoupling bias and variance.    
     \item \textbf{Weak anisotropy.} For weakly anisotropic data ($0 < \alpha < 1$), the problem remains effectively high-dimensional, and we derive an explicit reduction of the deterministic equivalents in the polynomial regime $n=\Theta(d^{\kappa})$. This reveals a striking sample-complexity separation induced by anisotropy: for targets sufficiently aligned with the data's principal directions, a degree-$m$ component becomes learnable already with $n=\Theta(d^{m(1-\alpha)})$ fewer samples than the $n=\Theta(d^m)$ required in the isotropic case. As a consequence, higher-order target structure can be learned at fractional values $\kappa=m(1-\alpha)$, between the usual integer polynomial thresholds. Remarkably, these bias transitions no longer coincide with interpolation peaks: the variance retains the multiple-descent pattern with peaks at integer $\kappa$, although their magnitude is progressively suppressed as anisotropy increases.
    \item \textbf{Strong anisotropy}. For strongly anisotropic data, the effective dimension of the problem is finite. We show that in this case the variance decouples from the sample size entirely: it saturates to a constant without regularization, or vanishes at an explicit rate under constant ridge penalty. The bias undergoes a sharp transition governed by the target's decay: below a threshold, a given target degree is learned only abruptly (or, under regularization, gradually but at a delayed sample complexity); above it, the bias decays smoothly as a power-law, recovering the source \& capacity rates in the literature.
\end{itemize}
Finally, we discuss in the case of single-index models, showing how the alignment of the index direction with the data's principal directions translates into markedly different learning behaviours in feature space. Together, our work provides a comprehensive characterization of how data anisotropy shapes the generalization properties of kernel ridge regression, from the underlying kernel spectrum to the resulting learning curves.

\subsection*{Further related work}
\textbf{Source and capacity conditions.} Classical guarantees for KRR are typically formulated directly in terms of conditions in the kernel spectrum and target decay; under these source and capacity conditions, optimally regularized KRR is known to attain minimax rates \citep{caponnetto_optimal_2007, cui_generalization_2022}. This framework has since been extended in several directions, such as random features \citep{rudi2017generalization} and Nystr\"om approximations \citep{rudi2015less} and structured prediction \citep{cabannes2021fast}, following earlier work on specific kernels and regularization schemes \citep{niyogi1996relationship, evgeniou2000regularization}. The source and capacity exponents are typically treated as abstract quantities, decoupled from the specific mechanism through which the input distribution produces them. Our work departs from this literature by deriving the source and capacity exponents themselves from the explicit input covariance structure, rather than treating them as free parameters.

\noindent\textbf{High-dimensional asymptotics.} A separate line of work derives exact asymptotic characterizations of the KRR and random features risk in the high-dimensional regime, largely via random matrix theory and Gaussian equivalence arguments \citep{cui_generalization_2022, cui2023error, Canatar_2021, bordelon_spectrum_2021, spigler2020asymptotic, simon_eigenlearning_2021, xiao2022precise, mei2022generalization, gerace2020generalisation, goldt2022gaussian, loureiro2021learning, hu2022universality,schroder2023deterministic, schroder24asymptotics, lu2025equivalence}. More recently, dimension-free deterministic equivalents have provided non-asymptotic versions of these formulas for ridge regression, KRR, and random features \citep{cheng_dimension_2025, misiakiewicz2024non, defilippis_dimension-free-2024}. Our work builds on this recent line of results to provide an in-depth study of the learning curves of KRR under anisotropic inputs.

\noindent\textbf{Data geometry and the kernel spectrum.} In some classical cases, the capacity exponent is tied explicitly to a named geometric property of the input: for Sobolev kernels on bounded domains it is governed by the input dimension \citep{bach2017equivalence}, and this connection extends to Riemannian manifolds, where matching intrinsic and extrinsic constructions recover rates governed by the manifold's intrinsic dimension \citep{rosa20252}. The dependence of the kernel spectrum on the input density was discussed more broadly by \citet{williams2000effect}, who observed empirically that a rapidly varying density accelerates spectral decay. For translation-invariant kernels on $\mathbb{R}^{d}$, this dependence was characterized asymptotically by \citet{widom1963asymptotic, widom1964asymptotic}, who showed that the eigenvalue decay is governed jointly by the tails of the input density and of the kernel's Fourier transform (see \citealt{bach2024unraveling, bach2025unraveling} for a modern synthesis and new non-asymptotic bounds). Our setting can be seen as an instance of this general density-kernel interplay for polynomial, rather than translation-invariant, kernels, in the regime where both $n=\Theta(d^{\kappa})$.

Closest to our work are three recent papers that investigated KRR and RFRR under power-law anisotropic Gaussian data. \citet{wortsman_kernel_2025} derived the spectral bounds of Proposition 1 for inner-product kernels under anisotropic Gaussian data and provided upper bounds on risk in the weak anisotropy regime; we tighten these bounds to an asymptotic equality (Proposition 2) and, building on this, give the first complete closed-form description of the bias and variance across both the weak- and strong-anisotropy regimes, together with the resulting phase diagram. \citet{karkada2025predicting} introduced the Hermite ansatz, empirically showing that KRR generalization under orthogonally invariant kernels can be computed by treating the kernel eigenfunctions as Hermite polynomials. \citet{paquette2026power} studied the spectrum of random feature kernel matrices under power-law data with monomial activations via Random Matrix Theory, focusing on the strong anisotropy regime.

\subsection*{Notation}
For $n \in \mathbb{N}$, we write $[n] := \{1, \ldots, n\}$. We use $\langle\cdot,\cdot\rangle$, ${\rm Tr}$, $I$ for the inner-product, trace and identity matrix/operator on both the standard Euclidean space $\mathbb{R}^{d}$ and the Hilbert space $\mathbb{R}^{\infty}\simeq \ell^{2}(\mathbb{N})$. We denote by $\mathbb{Z}_{\geq 0}$ the set of non-negative integers. For a multi-index $\beta = (\beta_1, \ldots, \beta_d) \in \mathbb{Z}_{\geq 0}^d$, we write $|\beta| := \beta_1 + \cdots + \beta_d$ for its degree, $x^\beta := \prod_{j=1}^d x_j^{\beta_j}$ and $\beta! := \beta_1! \cdots \beta_d!$.  We denote by $\mathrm{He}_m$ the $m$-th (probabilists') Hermite polynomial, orthogonal with respect to $\mathcal{N}(0,1)$. For a multi-index $\beta \in \mathbb{Z}_{\geq 0}^d$ and $x \in \mathbb{R}^d$, we write the associated (normalized) Hermite tensor as $H_\beta(x) \coloneqq \prod_{j=1}^d \tfrac{1}{\sqrt{\beta_{j}!}}\mathrm{He}_{\beta_j}(x_j)$, which forms an orthonormal basis of $L^2(\mathcal{N}(0,I_d))$. We use standard Big-O notation $O(\cdot), \Omega(\cdot), \Theta(\cdot), o(\cdot), \omega(\cdot)$, making the asymptotic variable explicit when useful: $\Theta_d(\cdot)$ or $o_d(1)$ as $d \to \infty$. We write $\tilde{O}(\cdot)$, $\tilde{\Theta}(\cdot)$ when suppressing sub-polynomial (e.g. logarithmic) factors, and $x \asymp y$ to mean $x = \Theta(y)$.

%% file: 2_Setting.tex
\section{Setting \& Background}
\label{s:Setting}

Consider a supervised regression task with a training dataset of $n$ samples $\mathcal{D} = \{( x_i, y_i)\in\mathbb{R}^{d}\times \mathbb{R}: i\in [n]\}$, where each input-response pair $(x_i, y_i)$ is sampled independently from the joint probability distribution $p_{X,Y}(x, y) = p_X(x) p_{Y | X}(y | x)$ over $\mathbb{R}^d \times \mathbb{R}$. Given a symmetric and positive definite function $k : \RR^{d}\times \RR^{d} \to \RR$, the kernel ridge regression (KRR) estimator is defined as the unique minimizer of the $\ell_{2}$-regularized empirical risk:
\begin{equation} \label{eq:krrEstimator}
    \hat{f}_\lambda := \argmin_{f \in \mathcal{H}} \left\{\sum_{i=1}^n \left(y_i - f(x_i)\right)^2 + \lambda \|f\|_{\mathcal{H}}^2 \right\} ,
\end{equation}
where $\mathcal{H}$ is the \textit{reproducing kernel Hilbert space} (RKHS) associated with the kernel $k$ and $\|\cdot\|_{\mathcal{H}}$ the induced norm. Denoting by $f_{\star}(x)=\mathbb{E}[y|x]$ the \emph{target function} or \emph{Bayes predictor}, our goal in the following will be to characterize the generalization properties of the KRR predictor $\hat{f}_{\lambda}$, as quantified by the \emph{excess risk}:
\begin{equation}\label{eq:ExcessPopolationRisk}
    \mathcal{R}(f_\star, \mathcal{D}, \lambda) \coloneqq \mathbb{E}_{x}[(\hat f_\lambda(x)-f_\star(x))^2] = ||\hat{f}_{\lambda}-f_{\star}||^{2}_{L^{2}(p_{X})}.
\end{equation}
An important observation is that the data distribution $p_{X}$ defines not only the explicit geometry of the risk, as in \cref{eq:ExcessPopolationRisk}, but also implicitly defines the geometry of the RKHS $\mathcal{H}$. Indeed, we can associate to $k$ a Hilbert–Schmidt integral operator $T_{k}: L^{2}(p_{X}) \to L^{2}(p_{X})$:
\begin{equation}\label{eq:KernelOperator}
     f \in L^{2}(p_{X}) \mapsto T_{k}(f) (x) \coloneqq \int_{\RR^{d}} k(x,x')f(x')p_{X}(\dd{x'}). 
\end{equation}
Since $T_{k}$ is compact, it can be diagonalized in $L^{2}(p_{X})$ \citep{cucker_mathematical_2001}:
\begin{align}
\label{eq:Tk:diag}
    T_{k} = \sum\limits_{m\geq 0} \lambda_{m}e_{m}\otimes e_{m}
\end{align}
where $(e_{m})_{m\geq 0}$ is an orthonormal basis of $L^{2}(p_{X})$ and $\lambda_{m}\geq 0$ are the kernel eigenvalues, which we assume without loss of generality are given in non-increasing order. Defining the features $\phi_{m}\coloneqq \sqrt{\lambda_{m}} e_{m}$, we can explicitly write the RKHS:
\begin{align}
    \mathcal{H}=\Big\{f=\sum_{m\geq 0}a_{m}\phi_{m}:||f||_{\mathcal{H}}^{2}\coloneqq\sum_{m\geq 0}\tfrac{a_{m}^{2}}{\lambda_{m}}<\infty\Big\}\subset L^{2}(p_{X}).
\end{align}
Since both $(\lambda_{m}, e_{m})$ are implicitly defined with respect to $p_{X}$ in \cref{eq:KernelOperator}, the geometry of $\mathcal{H}$ is implicitly determined by $p_{X}$. A natural question steming from this observation is: 
\begin{center}
\emph{How does the data geometry $p_{X}$ impact the generalization properties of the kernel ridge estimator $\hat{f}_{\lambda}$?}      
\end{center}

Explicit results connecting the geometry of the data to the excess risk are scarce. Indeed, this requires explicitly diagonalizing the kernel operator $T_{k}$, which is a challenging problem in general. For this reason, results in this direction have mostly been specific to isotropic input distributions $p_{X}$, such as spherical, hypercube or isotropic Gaussian, for which the eigenfunctions that diagonalize the kernel operator are well-studied harmonic polynomials \citep{ghorbani_linearized_2020, mei2022generalization,misiakiewicz_spectrum_2022}. In these cases, the kernel eigenspectrum is degenerated: for each degree $m$ polynomial, there are $\Theta_{d}(d^{m})$ constant eigenvalues $\lambda_{m}=\Theta_{m}(d^{-m})$. As a consequence, the excess risk is radically different from a power-law decay, following an abrupt staircase pattern where degree $m$ components of the target function $f_{\star}$ require at least $n=O(d^{\kappa})$ samples in the high-dimensional limit $d\to\infty$ \citep{misiakiewicz_spectrum_2022}.

Real data, however, is far from isotropic: different directions exhibit different variance, typically aligned with the target at different scales \citep{mallat_theory_1989}. For instance, images contain regular structures which lead to power-law spectrum on an adapted basis \citep{simoncelli2001natural}, and the pairwise correlation of tokens in text is also well-approximated by a power-law \citep{cagnetta2026deriving}.

Motivated by this observation, in this work we investigate the generalization properties of KRR when the inputs are drawn from an \textit{anisotropic} Gaussian distribution $p_{X}=\mathcal{N}(0,\Sigma)$ where the eigenvalues  $\{\sigma_j\}_{j\in[d]}$ of the data covariance matrix $\Sigma\in\mathbb{R}^{d\times d}$ decay as a power-law:
\begin{align}
\label{eq:data:power}
\sigma_{j} = r_{\alpha}(d)^{-1} \cdot j^{-\alpha}, \quad j \in [d]. 
\end{align}
with the normalization $r_{\alpha}(d)=\sum_{j=1}^d j^{-\alpha}$ chosen such that 
${\rm Var}(p_{X}) = \Tr \Sigma = 1$. Note that the exponent $\alpha \geq 0$ quantifies the degree of anisotropy in the data, interpolating between the isotropic case for $\alpha = 0$ and an effectively low-dimensional distribution for $\alpha \gg 1$. 

Our main goal in the following is to derive a sharp description of the excess risk in \cref{eq:ExcessPopolationRisk} in the polynomial high-dimensional regime where $d\to\infty$ at fixed $\psi\coloneqq \tfrac{n}{d^{\kappa}} = \Theta_{d}(1)$ for $\kappa>0$. While our primary focus is on integer values $\kappa \in \mathbb{N}$, we will occasionally explore cases where $\kappa$ is fractional.

\paragraph{Assumptions --- } Our results will hold under the setting introduced above and the following assumptions on the target and kernel.

% By rotational invariance of the Gaussian distribution, we can assume without loss of generality that $\Sigma$ is a diagonal matrix with $\Sigma_{jj}=\sigma_{j}$. We will further assume that $f_{\star}\in L^{2}(p_{X})$, and that the label noise $\varepsilon = y - f_{\star}(x)$ has bounded variance $\sigma^{2}<\infty$. Moreover, in the following we will focus on a particular class of polynomial inner-product kernels.
\begin{assumption}[Polynomial inner-product kernel]
\label{assumption:h}
The kernel $k(x, x') = h(\langle  x, x' \rangle)$ with $h:\mathbb{R}\to\mathbb{R}$ given by a polynomial of arbitrary but finite degree $D$:
\begin{align}
\label{eq:def:polyk}
    h(t) = \sum_{m = 0}^D h_m t^m.
\end{align}
with coefficients $h_m\geq 0$.
\end{assumption}
\begin{assumption}[Target function]
\label{ass:target}
The target function $f_{\star}\in L^{2}(p_{X})$ is a $D$-degree polynomial, and that the label noise $\varepsilon = y - f_{\star}(x)$ are independent, mean-zero, $\sigma_{\varepsilon}^{2}$-sub-Gassian random variables.
\end{assumption}
\begin{remark}
Note that a kernel of degree $D$ is able to fit at most the degree $D$ components of a general target $f_{\star}\in L^{2}(p_{X})$, with the orthogonal part behaving essentially as an irreducible additive noise. Therefore, under \cref{eq:def:polyk}, the first part of \cref{ass:target} is, up to an adjustment on the noise variance, without loss of generality. For the purpose of the results that follows, we always consider $D$ to be large enough.
\end{remark}

%% file: 3_Spectrum.tex
\section{High-dimensional limit of the kernel spectrum}
\label{s:Spec}
As discussed in \Cref{s:Setting}, the statistical properties of the KRR estimator crucially depend on the kernel spectrum. The spectrum of inner-product kernels under anisotropic Gaussian data was recently studied in \citep{wortsman_kernel_2025}, who proved the following result.
\begin{proposition}[Upper and lower bounds on the kernel spectrum \citep{wortsman_kernel_2025}] \label{th:KernelSpectrum}
    Let $K(x, x') = h(\langle x, x' \rangle)$ be an inner-product kernel with $h(t)$ obeying Assumption~\ref{assumption:h} and assume the data covariance matrix is $\Sigma = \text{diag}(\sigma_1, \dots, \sigma_d)$. Then, the kernel eigenvalues can be indexed by multi-indices $\beta \in \mathbb{Z}^d_{\geq 0}$, and there exist constants, $C_{1}, C_2$, independent of $d$, such that
    \begin{equation}\label{eq:KernelSpectrum}
        C_1 h_{|\beta|} |\beta|! \, \sigma_1^{\beta_1} \dots \sigma_d^{\beta_d}  \leq \lambda_{\beta} \leq  C_2 h_{|\bm\beta|} |\beta|! \, \sigma_1^{\beta_1} \dots \sigma_d^{\beta_d} ,
    \end{equation}
    where $|\beta| = \beta_1 + \dots + \beta_d$.
\end{proposition}

\begin{remark}
    The intuition behind \Cref{th:KernelSpectrum} is the following: Since we are considering a polynomial kernel, for $x,x' \sim \mathcal{N}(0,\Sigma)$ we can write
    \[
    K(x,x') = h(\langle x,x'\rangle ) = \sum_{m=0}^{D} h_{m}\langle x,x'\rangle^{m} = \sum_{m=0}^{D} h_{m} \sum_{\beta \in \ZZ^{d}_{\geq 0 }: |\beta|=m} \binom{m}{\beta} x^{\beta} x'^{\beta},
    \]
    where in the last equality we expanded the inner product. For each multi-index $\beta$, the polynomial $x^{\beta}$ is contributes with a new independent component to the kernel. Therefore, intuitively, there should be an eigenvalue related to this new component. By further writing $z = \Sigma^{-\frac{1}{2}}x$, the whitened input, we have that $x^{\beta} = \prod_{k=1}^{d} \lambda_{k}^{\beta_{k}} z^{\beta}$, and therefore the eigenvalue should be related to  $\lambda^{\beta} = \prod_{k=1}^{d} \lambda_{k}^{\beta_{k}}$. 
\end{remark}

Our first result is to show that \cref{th:KernelSpectrum} can be made sharper in the high-dimensional limit considered here. More precisely, consider the \emph{participation ratio} of the covariance $\Sigma$:
\begin{equation}
    R_0(\Sigma) = \dfrac{(\sum_{i=1}^{d} \sigma_i)^{2}}{(\sum_{i=1}^{d} \sigma_i^{2})}.  
\end{equation}
Then, if $R_{0}(\Sigma)\to \infty$ in the high-dimensional limit $d\to\infty$, we can show that the upper- and lower-bound of \cite{wortsman_kernel_2025} is tight. Since in the considered setting $\Tr(\Sigma)=1$, this condition is equivalent to $\Tr(\Sigma^2) = o_{d}(1)$. 
\begin{proposition}[Asymptotic kernel spectrum] 
\label{prop:spec:highd}
Assume $\Tr(\Sigma^2) = o_{d}(1)$ and \Cref{assumption:h}. Then, the spectrum of the kernel in \cref{eq:def:polyk} indexed by multi-indices $\beta \in \ZZ^{d}_{\geq 0}$, $|\beta| \leq D $  satisfies: 
\begin{equation}
    \lambda_{\beta} = h_{|\beta|} |\beta|! \, \sigma_1^{\beta_1} \dots \sigma_d^{\beta_d}( 1+ o_{d}(1) ) . 
\end{equation}
In particular, this holds for $\alpha\in[0,1]$ under power-law data in \cref{eq:data:power}.
\label{proposition:asymptotic_matching_spectrum}
\end{proposition}

We prove \Cref{proposition:asymptotic_matching_spectrum} and other properties of the kernel operator in Appendix~\ref{section:kernel_operator}. By applying \Cref{proposition:asymptotic_matching_spectrum} to our setting with a power-law decay in the covariance, $\sigma_j = r_\alpha(d)^{-1} j^{-\alpha}$, we obtain:
\begin{equation}\label{eq:SpectrumSpec}
    \lambda_{\beta} = h_{|\beta|} |\beta|! \, r_\alpha(d)^{-|\beta|} \prod_{j=1}^d j^{-\alpha\beta_j}(1 + o_{d}(1)).
\end{equation}
This spectral structure is strongly determined by the normalization constant $r_\alpha(d)$, which naturally plays the role of the \textit{effective dimension} of the problem. Since $\Tr\Sigma = 1$, we have that to leading order in $d\to\infty$:
\begin{equation}\label{eq:EffectiveDimension}
    r_\alpha =  
    \begin{cases}
        \Theta_{d}\left ( (1-\alpha)^{-1}d^{1-\alpha} \right )& \text{if } 0 \leq \alpha < 1 , \\
        \Theta_{d}\left ( \log(d) \right ) & \text{if } \alpha = 1 , \\
        \zeta(\alpha) & \text{if } \alpha > 1 ,
    \end{cases}
\end{equation}
where $\zeta(\cdot)$ is the Riemann zeta function. Therefore, when $0 < \alpha < 1$, the effective dimension diverges as $d \to \infty$ (albeit sub-linearly), meaning the problem preserves a genuinely \emph{high-dimensional nature}. We refer to this as the \emph{weak anisotropy regime}.
Conversely, when $\alpha > 1$, the covariance decay is so rapid that $r$ saturates to a finite constant, and the problem is effectively finite dimensional. We refer to this as the \emph{strong anisotropy regime}.  A first consequence of this separation can be observed in the spectrum of the kernel. 
\begin{remark}
    A few remarks about the expression in \cref{eq:EffectiveDimension} are in order.
    \begin{itemize}
        \item For $\alpha = 0$ (isotropic case), rotational invariance implies that all features with a given frequency ``shell'' $m = |\beta|$ are equivalent, and therefore the spectrum in ~\cref{eq:KernelSpectrum} is highly degenerated. In particular, for each frequency $m$ there are $\Theta(d^{m})$ degenerated eigenvalues, each with a magnitude of $\Theta_{d}(d^{-m})$.
        \item For $\alpha > 1$ (strong anisotropy), \cite{wortsman_kernel_2025} showed that the kernel eigenvalues satisfies a capacity condition $\lambda_{m}=\tilde{\Theta}(m^{-\alpha})$ with the exponent $\alpha$ exactly inherited from the input data.
        \item For $0 < \alpha < 1$ (weak anisotropy), the spectrum becomes more complex, exhibiting gaps and disconnected sectors. Consequently, a simple capacity condition is insufficient to accurately characterize the true generalization behavior.
    \end{itemize}
\end{remark}
As we discuss next, this separation also has important consequences for the learning curves.

%% file: 4_1_Weakly_Anisotropic_Regime.tex
\section{Generalization properties}
\label{s:Gen}

We now move our attention to our core contribution, which is a sharp characterization of the excess risk in \cref{eq:ExcessPopolationRisk} under the anisotropic setting introduced in \cref{s:Setting}. Our results build on the recent progress in the non-asymptotic analysis of kernel methods, pioneered by \citet{cheng_dimension_2025, misiakiewicz2024non, defilippis_dimension-free-2024}, and which we quickly review in the following.

Consider the bias-variance decomposition of the excess risk:
\begin{equation}\label{eq:StandardDecompositionBiasVariance}
   \mathcal{R}(f_{\star},X,\sigma^{2},\lambda)\coloneqq \mathbb{E}_\varepsilon[\mathcal{R}(f_\star, X,\varepsilon,\lambda)] = \mathcal{B}(f_\star,  X, \lambda) + \mathcal{V}(X,\sigma^{2},\lambda)
\end{equation}
where $X = [x_1, \dots, x_n ]^{\top} \in \RR^{n \times d}$ is the data matrix, $y = (y_1, \dots, y_n)^{\top} \in \RR^{n}$ and $\varepsilon = (\varepsilon_1, \dots, \varepsilon_n)^{\top} \in \RR^{n}$ are the label and noise vectors and:
\begin{align}\label{eq:BiasVariance}
    \mathcal{V}(X, \sigma^{2}, \lambda) = \mathbb{E}_{ x}\left[\text{Var}_\varepsilon(\hat f_\lambda( x))\right] , \quad\quad
    \mathcal{B}(f_\star,  X, \lambda) = \mathbb{E}_{ x}\left[\left(f_\star(x) - \mathbb{E}_\varepsilon[\hat f_\lambda( x)]\right)^2\right] .
\end{align}
Note that both the bias and variance are random quantities, as they depend on the specific realization of the training data $\mathcal{D}=(X,y)$. \cite{misiakiewicz2024non} has proven that under specific assumptions on the kernel spectrum and eigenfunctions (see \cref{appendix:det_equivalents} for a detailed discussion), the following deterministic approximation of the bias and variance hold with high-probability:
\begin{align}
\label{eq:detequiv:rates}
\mathcal{B}(f_\star, X, \lambda) = \left(1+ \tilde{O}(n^{-1/2})\right) \mathsf{B}_n(\theta, \Lambda, \lambda) \, , \qquad \mathcal{V}(X,\Lambda, \lambda) = \left(1+ \tilde{O}(n^{-1/2})\right) \mathsf{V}_n(\sigma^{2},\Lambda,\lambda). 
\end{align}
where $\Lambda = \text{diag}(\lambda_1, \lambda_2, \dots) \in \mathbb{R}^{\infty\times\infty}$ is a diagonal matrix with the kernel eigenvalues and $ \theta=(\theta_{1},\theta_{2},\dots)\in\mathbb{R}^{\infty}$ are the coefficient of the target function in the $L^{2}(p_{X})$ basis. The functions $\mathsf{B}_{n}, \mathsf{V}_{n}$, which only depend only on these deterministic quantities, are known as \emph{deterministic equivalents} for the bias and variance, and are explicitly given by:
\begin{equation}\label{eq:DetEquivalent}
    \mathsf{V}_{n}(\Lambda,\lambda,\sigma^{2}_{\varepsilon}) = \sigma_\varepsilon^2 \frac{\Tr(\Lambda^2(\Lambda + \nu_{\star} I)^{-2})}{n-\Tr(\Lambda^2(\Lambda + \nu_{\star} I)^{-2})} \qquad \text{ and } \qquad 
    \mathsf{B}_n(\theta,\Lambda,\lambda) = \frac{\nu_{\star}^2\langle\theta, (\Lambda + \nu_{\star} I)^{-2}\theta\rangle}{1-\frac{1}{n}\Tr(\Lambda^2(\Lambda + \nu_{\star} I)^{-2})}, 
\end{equation}
where $I\in\mathbb{R}^{\infty\times \infty}$ is the infinite-dimensional identity matrix and the scalar quantity $\nu_{\star}(\lambda) > 0$ is the unique solution to the following self-consistency equation:
\begin{equation}\label{eq:SelfConsistencyEquation}
    n - \frac{\lambda}{\nu} = \Tr(\Lambda(\Lambda + \nu I)^{-1}).
\end{equation}
\begin{remark} 
Note that the quantity $\nu_{\star}(\lambda)\geq 0$ plays a similar role to the $\ell_{2}$ regularisation in the random expressions for the bias and variance. However, \cref{eq:SelfConsistencyEquation} implies that $\nu_{\star}(0)>0$, meaning that the interpolator $\hat{f}_{\lambda=0}$  will not necessarily overfit. For this reason, it is known as the \emph{effective regularization} \citep{bartlett_2020, bach_high-dimensional_2023}.
\end{remark} 
The conditions under which the deterministic equivalents in \Cref{eq:DetEquivalent} hold crucially depend on the spectrum and eigenfunctions of the kernel. Checking them is a non-trivial task, since they in principle require explicitly diagonalizing the kernel, which as previously discussed is a challenging problem in general. For this reason, these conditions were proven to hold only in the isotropic case where the spectrum and eigenfunctions are explicitly known \citep{misiakiewicz_spectrum_2022, misiakiewicz2024non}. Our first result in this Section is to show that the deterministic equivalents in \cref{eq:DetEquivalent} hold for the anisotropic power-law setting introduced in \cref{s:Setting}. A detailed discussion of the conditions 

\begin{proposition}[Validity of deterministic equivalents]
\label{prop:det}
Consider the KRR setting introduced in \Cref{s:Setting}. Define the deterministic equivalent for the excess risk:
\[
\mathsf{R}_n(\theta,\Lambda,\lambda,\sigma^2_\varepsilon) := \mathsf{B}_n(\theta,\Lambda,\lambda) + \mathsf{V}_{n}(\Lambda,\lambda,\sigma^{2}_{\varepsilon}),
\]
where $\mathsf{B}_{n}, \mathsf{V}_{n}$ given in \cref{eq:DetEquivalent}.
Under Assumptions \ref{ass:target} and \ref{assumption:h}, with probability going to $1$ as $n$ grows, we have
\[
\left|
\mathcal{R}(f_\star, X, \sigma_\varepsilon^2, \lambda)
-
\mathsf{R}_n(\theta,\Lambda,\lambda,\sigma^2_\varepsilon)
\right|
\leq
C_{x,\varepsilon,D,K}\,
\mathcal{E}_{R,n}(m)\,
\mathsf{R}_n(\theta,\Lambda,\lambda,\sigma^2_\varepsilon),
\]
where $\mathcal{E}_{R,n}(m)$ is a relative approximation error defined in \Cref{appendix:det_equivalents} that satisfies $\mathcal{E}_{R,n}(m) \to 0$ with $n$, 
\end{proposition}
A detailed discussion of the conditions and the proof of \cref{prop:det} are discussed in Appendix \ref{appendix:det_equivalents}. With this result in hand, we can move to the core of our contribution, which is to study how anisotropy impact the generalization properties of the kernel predictor.

\subsection{Weakly anisotropic regime}\label{s:WA}
As discussed in \Cref{s:Spec}, the behavior of the spectrum in the high-dimensional regime strongly depends on the level of anisotropy of the inputs. Therefore, it should be no surprise that the same will be true for the excess risk. We start by considering the \emph{weak anisotropy regime}, for which feature space is effectively high-dimensional (see \cref{eq:EffectiveDimension}). 

The starting point is to find the unique solution of the self-consistent \cref{eq:SelfConsistencyEquation} for the explicit setting introduced in \cref{s:Setting}. 
\begin{lemma}[\textit{Kernel state equation}]\label{def:KernelStateEquation}
    Consider the KRR setting introduced in \Cref{s:Setting} in the weak anisotropy regime $\alpha\in[0,1)$. Then, in the high-dimensional limit $d\to\infty$ with $\tfrac{n}{d^{\kappa}}\to\psi>0$ and $\kappa\in\mathbb{N}$, the solution $\nu_{\star}>0$ of \cref{eq:SelfConsistencyEquation} takes the form $\nu_{\star}=\xi_{\star}d^{-\kappa}$, where $\xi_{\star}>0$ is the unique positive root of the following equation:
    \begin{equation}
    \label{eq:KernelStateEquation}
        \psi = \frac{\lambda_{\text{eff}}}{\xi} + \frac{1}{\kappa! (\kappa-1)!} \int_{0}^{\infty} dt \> e^{-t} t^{\kappa-1} \frac{h_\kappa \kappa!}{h_\kappa \kappa! + (1-\alpha)^{-\kappa} \xi e^{-\alpha t}}
    \end{equation}
where $\lambda_{\text{eff}} \coloneqq  \lambda + \sum_{m > \kappa} h_m$ acts as an \textit{effective ridge regularization}. 
\end{lemma}
We now sketch the derivation of this result, and refer the reader to \Cref{appendix:KernelStateEquation} for the details. Considering the ansatz $\nu = \xi d^{-\kappa}$, for some $\xi >0$ and inserting in the self-consistent equations \cref{eq:SelfConsistencyEquation} together with the asymptotic spectrum given by \Cref{prop:spec:highd}, one can unroll and group the terms in the trace in the self-consistency \cref{eq:SelfConsistencyEquation} according to the value of $|\beta|$:
\begin{align}
\Tr(\Lambda(\Lambda + \nu I)^{-1}) = \sum_{m=0}^\infty \sum_{|\beta|=m} \frac{\lambda_{\beta}}{\lambda_{\beta}+\nu} \equiv \sum_{m=0}^\infty t^{(m)}. 
\end{align}
We call the group of $\beta \in \ZZ^{d}_{\geq 0}$ with $|\beta|=m$ the \emph{shell} of degree $m$. Taking the $d\to\infty$ limit, the spectrum separates into three groups:
    \begin{itemize}
        \item \textbf{Learned features ($m < \kappa$):} Terms in the sum which vanish in the limit. These correspond to fully resolved features. 
        \item \textbf{Unlearnable tail ($m > \kappa$):} Terms in the sum for which the contribution does not vanish and becomes independent of the sample complexity $\psi$. Their aggregate effect converges to a constant, which combined with the original $\ell_{2}$-regularization gives the \textit{effective ridge regularization} $\lambda_{\text{eff}} = \lambda + \sum_{m > \kappa} h_m$.
        \item \textbf{The threshold shell ($m = \kappa$):} This boundary shell is only partially resolved and its contribution strictly depends on $\psi$, yielding the continuous integral in the kernel state \cref{eq:KernelStateEquation}.
    \end{itemize}
    
\begin{remark}
    In this regime, the high-order tails of the kernel Taylor expansion $l(\kappa) = \sum_{m > \kappa}h_m$ act as an implicit regularization mechanism. This implies that even in the ``ridgeless'' limit $\lambda\to 0^{+}$, the kernel will still exhibit a self-regularizing effect which is completely determined by the high-degree components of the kernel $h$. This mechanism is exactly the same as for in isotropic case, first described by \cite{misiakiewicz_spectrum_2022}.
\end{remark}

Given a solution $\xi_{\star} \in \RR$ of \cref{eq:KernelStateEquation}, one can compute $\mathsf{V}_n(\Lambda,\lambda,\sigma^{2}_{\varepsilon})$ and $\mathsf{B}_n(\theta,\Lambda,\lambda)$. We start by discussing the variance term. 

\subsubsection{Variance}\label{ss:variancewa}
\begin{figure}[t]
\centering
\includegraphics[width=0.75\linewidth]{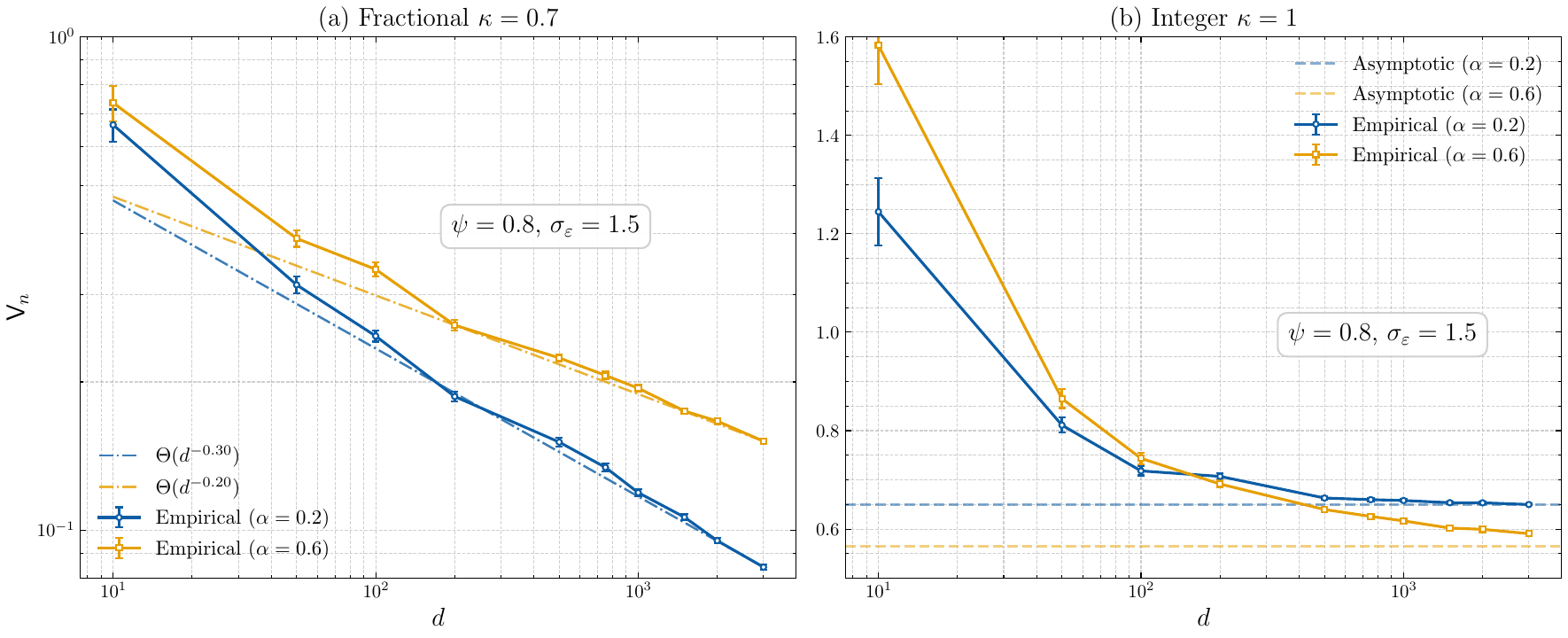}
\caption{Numerical illustration of the asymptotic variance of KRR. We use a null target function ($f_\star = 0$) to isolate the variance component of the risk. The experimental parameters are $h(t) = 1 + 8t + 3t^2 + 0.1 t^3 + t^4$, $\lambda = 0$, $\psi = 0.8$ and $\sigma_\varepsilon^2 = 1.5^2$. (\textbf{Left}) For fractional sample complexity ($\kappa = 0.7$), the empirical variance vanishes. The curve matches the theoretical power-law decay rate predicted by Eq.~\eqref{eq:DecayRatesVarianceFraction} (dashed line). (\textbf{Right}) For integer sample complexity, the variance instead plateaus to a finite, order-one constant given by Eq.~\eqref{eq:VarianceWA}.}
\label{fig:VarianceConvergence}
\end{figure}
Once the value of $\xi$ in \Cref{eq:KernelStateEquation} is found, we get the following result for the deterministic equivalent of the variance $\mathsf{V}_n(\lambda)$: 

\begin{theorem}[\textit{Variance in the weakly anisotropic regime}]\label{def:VarianceWeak} Consider KRR in the setting described in Section~\ref{s:Setting} in the weakly anisotropic regime ($0 \leq \alpha < 1$). Denote by $\xi_\star$ the unique root of the Kernel State Equation defined in \Cref{def:KernelStateEquation}. Under the sample complexity scaling $n = \psi d^{\kappa}$ (with $\kappa \in \mathbb{N}$), the deterministic equivalent for the variance in the high-dimensional limit ($d \to \infty$) is given by:
    \begin{equation}\label{eq:VarianceWA}
        \mathsf{V}_{n} = \sigma_\varepsilon^2 \frac{\tau}{1-\tau}
    \end{equation}
    where $\tau = \tau(\psi, \kappa, \alpha, \lambda, \{h_m\})$ is given by:
    \begin{equation}\label{eq:tauWA}
        \tau = \frac{1}{\psi} \frac{1}{\kappa!(\kappa-1)!} \int_0^\infty dt \> e^{-t} t^{\kappa - 1} \frac{(h_\kappa \kappa!)^2}{(h_\kappa \kappa! + (1-\alpha)^{-\kappa} \xi_\star e^{-\alpha t})^2}
    \end{equation}
    For fractional sample complexity $\kappa \not\in\mathbb{N}$, the variance asymptotically vanishes to zero as $\mathsf{V} = \Theta(d^{-\gamma_V})$ where:
\begin{equation}
        \gamma_V = 
        \begin{cases}
            \min\big(\kappa - \lfloor \kappa \rfloor, \; \lceil \kappa \rceil - \kappa \big) & \text{if } 0 \leq \alpha \leq \frac{1}{2} \\[6pt]
            \min\big(\kappa - \lfloor \kappa \rfloor, \; \frac{1-\alpha}{\alpha}(\lceil \kappa \rceil - \kappa) \big) & \text{if } \frac{1}{2} < \alpha < 1
        \end{cases}
        \label{eq:DecayRatesVarianceFraction}
    \end{equation}
\end{theorem}
\Cref{fig:VarianceConvergence} illustrates  \Cref{def:VarianceWeak} numerically, showing the good agreement between the deterministic equivalent for the variance and a finite-size simulations.
\begin{remark}
    The quantity $\tau$ in \Cref{eq:DecayRatesVarianceFraction} is closely tied to the effective dimensionality of the kernel operator. Formally, its value corresponds to
\[
\tau = \frac{\text{df}_2(\nu_{\star})}{\text{df}_1(\nu_{\star})+ \frac{\lambda}{\nu_{\star}}}
\]
where $\text{df}_1(\nu)$ and $\text{df}_2(\nu)$ denote respectively the the \emph{first and second degree of freedom} of the kernel eigenvalue matrix. The degrees-of-freedom offer an alternative measure of the kernel's effective dimensionality, which is diverging in the high-dimensional limit \citep{zhang_learning_2005}. 
Specifically, $\tau$ captures the interplay between these two distinct measures of capacity: in the ridgeless limit ($\lambda \to 0^+$), $\tau$ reduces to the ratio $\text{df}_2(\nu_{\star}) / \text{df}_1(\nu_{\star})$, thereby quantifying how dissimilar these metrics are. Indeed, whenever the first and second degree of freedom are approximately equal $\text{df}_1 \approx \text{df}_2$, then $\tau \approx 1$ (and the variance explodes to infinity). Conversely, when the second degree of freedom is negligible compared to the first ($\text{df}_2 \ll \text{df}_1$), then $\tau \to 0$ (and the variance collapses to zero). The specific behavior of $\tau$ as a function of the power law exponent $\alpha$, as defined in Eq.~\eqref{eq:tauWA}, is illustrated in \Cref{fig:tau}. Interestingly, as $\alpha$ approaches $1$, $\tau \to 0$ independently of $\psi$ and $\lambda$; this indicates that the second degree of freedom becomes strictly negligible when the data covariance spectrum approximately decays as $\sigma_j \propto j^{-1}$.
\end{remark}

\begin{figure}
    \centering
    \includegraphics[width=0.8\linewidth]{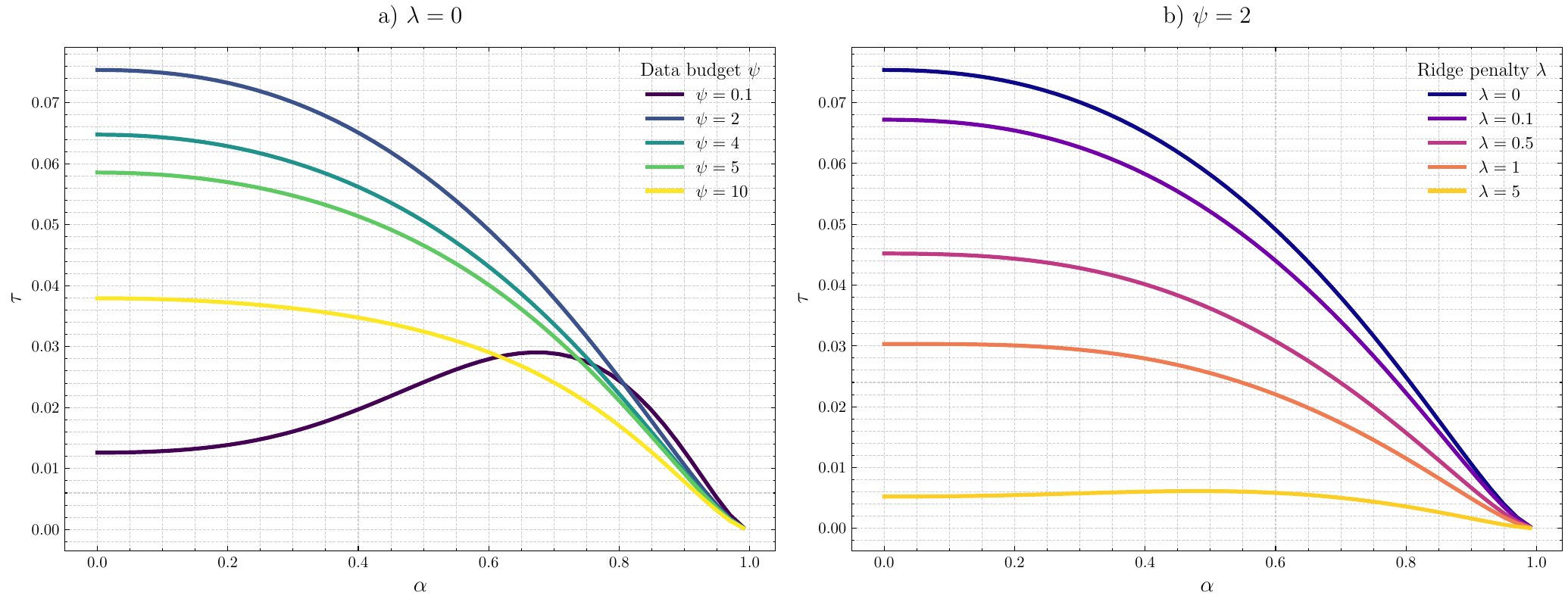}
    \caption{Behaviour of $\tau$ as a function of input covariance power-law exponent $\alpha \in [0,1)$, obtained by numerically solving the Kernel state \cref{def:KernelStateEquation}. \textbf{(Left)} Different sample complexities $n= \psi d^{\kappa}$ with $\psi \in \{0.1, 2, 4, 5, 10\}$, in the interpolation limit ($\lambda = 0$). \textbf{(Right)} Different regularization strengths $\lambda \in \{0, 0.1, 0.5, 1, 5\}$ at fixed sample complexity ($\psi = 2$). In both figures, we take $\kappa = 2$ ($n = \psi d^2$) under the fourth-order truncated exponential kernel $h(t) = \sum_{j=0}^4 t^j/j!$. 
    %Crucially, as the spectrum disperses ($\alpha \to 1^-$), the capacity parameter is systematically suppressed ($\tau \to 0$), shielding the estimator from the interpolation variance singularity.
    }
    \label{fig:tau}
\end{figure}
More fundamentally, \cref{eq:VarianceWA} reveals that the variance is strictly controlled by $\tau$ and, consequently, by the gap between $\rm{df}_{1}$ and $\rm{df}_{2}$ \citep{bach_high-dimensional_2023}. Indeed, given $\tau = \text{df}_2(\nu_{\star}) / \text{df}_1(\nu_{\star})$ for $\lambda = 0$, one can rewrite:
\[
\mathsf{V}_{n} = \sigma_\varepsilon^2 \left(\frac{\text{df}_1(\nu_{\star}) - \text{df}_2(\nu_{\star})}{\text{df}_2(\nu_{\star})}\right)^{-1}.
\]
Given $\tau$, the the behavior of the variance $\mathsf{V}_{n}$ as a function of the sample complexity $\psi$, at fixed scaling $\kappa\in\mathbb{N}$ can be easily evaluated. 
\begin{figure}[t]
    \centering
    \includegraphics[width=0.8\linewidth]{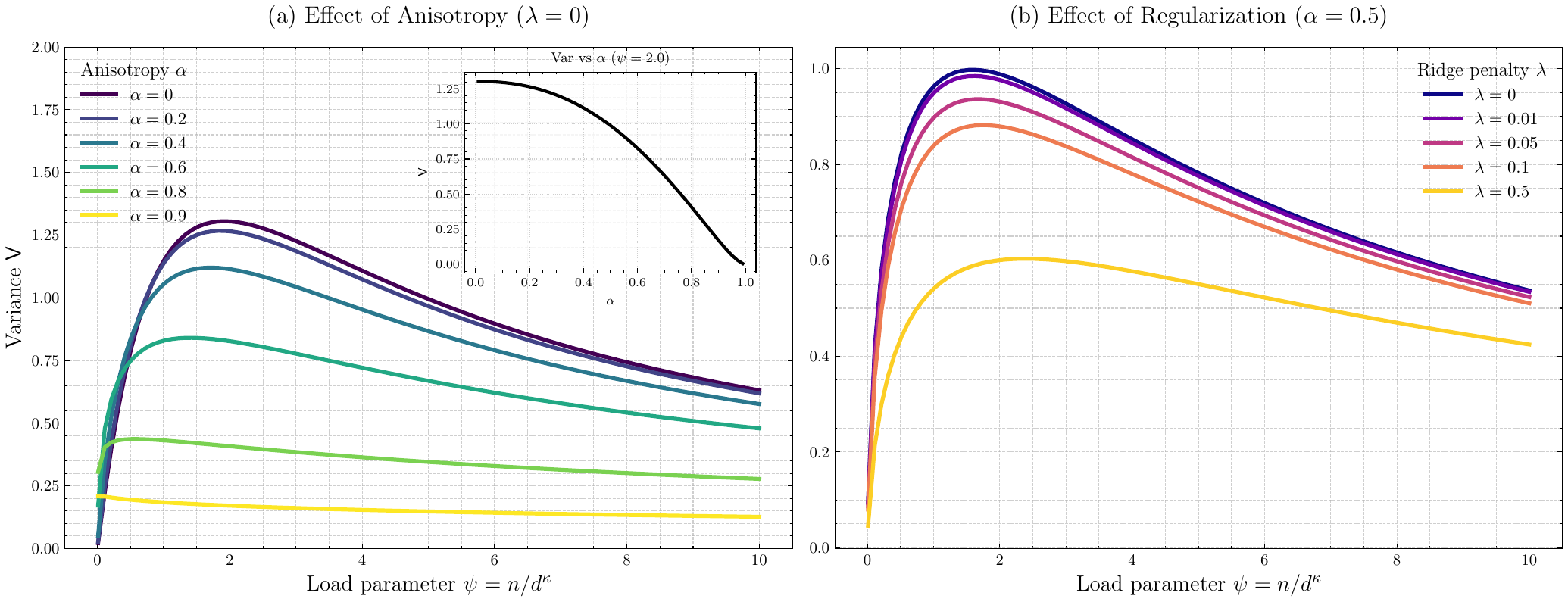}
    \caption{Deterministic equivalent of the variance as a function of the sample complexity for $\kappa=2$, $\sigma_\varepsilon = 4$ and polynomial kernel $h(t) = 1 + t + \frac{1}{2}t^2 + \frac{1}{6}t^3 + \frac{1}{24}t^4$.
    (\textbf{Left}) Fixed $\lambda=0$ for different values of anisotropy. The inset shows the variance curve for a fixed $\psi = 2$; (\textbf{Right}) Fixed $\alpha = 0.5$ for different regularization $\lambda$. As expected, stronger regularization suppresses the variance.}
    \label{fig:VarianceWA}
\end{figure}
\Cref{fig:VarianceWA} (Left) shows that the variance exhibits a distinct peak at a finite value $\psi^\star > 0$, and systematically vanishes at both extremes of the domain ($\psi \to 0, \psi \to \infty$). This behavior is consistent with the previously discussed result showing that the variance vanishes for fractional values of $\kappa$. 

By horizontally concatenating the variance curves across different scales $\kappa$, one obtain a global picture given by a \textit{multiple descent} curve. Each variance peak emerges at sample size $n$ matching the number of degrees of freedom in the next polynomial feature space (i.e., whenever $n \asymp d^{\kappa}$ for $\kappa \in \mathbb{N}$), while the variance vanishes for intermediate values. This multiple descent pattern in the variance is akin to the one shown by \cite{misiakiewicz_spectrum_2022, Canatar_2021, liang_multiple_2020} for isotropic data, and the fractional decay rates in \cref{eq:DecayRatesVarianceFraction} are consistent with the results of \citet{liang_multiple_2020} (Theorem 1) in the isotropic limit $\alpha \to 0$.

An important different, however, is that increasingly has a suppressing effect on the peaks. Indeed, as $\alpha\in[0,1)$ increases, the entire variance curve is progressively suppressed, flattening the interpolation peak towards zero and shifting its position toward smaller sample complexity. Indeed, in the limit $\alpha \to 1^{-}$, \Cref{eq:tauWA} implies that $\mathsf{V}_{n}\to 0$ identically since $(1-\alpha)^{-\kappa}\to\infty$. This is also a consequence of the fact that, as $\alpha \to 1^-$, $\tau \to 0$ and $\text{df}_2(\nu)$ becomes negligible compared to the first degree of freedom $\text{df}_1(\nu)$. 

\begin{remark}[\textit{Anisotropy suppresses multiple descent}]   
    Consider the variance decay for fractional $\kappa$. In the limit $\alpha \to 1^-$, the second argument of the minimum in \cref{eq:DecayRatesVarianceFraction} vanishes, analytically yielding $\gamma_V \to 0$. This implies that as we approach the strongly anisotropic regime, the fast power-law decay of the variance at fractional sample complexities $\kappa$ progressively slows down, closing the gap with the integer case. As a consequence, the multiple descent phenomenon gets gradually suppressed: the peaks at integer $\kappa$ are progressively flattened (to zero) while the vanishing valleys in the intermediate fractional regimes disappear. In other words, the variance curve effectively flattens out as we approach the strongly anisotropic regime, washing out the sharp distinction between integer and fractional regimes.
\end{remark}

\subsubsection{Bias}\label{sssection:BiasWA}
We now turn our attention to the bias. \Cref{eq:DetEquivalent} depends on the coefficients of the target function $f_\star$ into the basis that diagonalises the kernel operator, defined in \cref{eq:Tk:diag}. For the polynomial inner-product kernel, this defines a basis of degree $D$ polynomials in $\mathbb{R}^{d}$, which we denote $(e_{\beta})_{\beta\in\mathbb{Z}_{\geq 0}^{d}: |\beta|\leq D}$:
\begin{equation}\label{eq:decompo1}
    f_{\star}(x) = \sum_{\beta \in \mathbb{Z}^d_{\geq 0}} \theta_{\beta} e_{\beta}(x). 
\end{equation}
Working directly with the sum over the multi-index $\beta$ is analytically challenging, however, we can still group the indices according to the degree of $\beta$ (which does not necessarily coincides with the degree of the polynomial $\phi_{\beta})$.
\begin{lemma}
    Let $f_\star \in L_2(p_X)$ denote the target function. The decomposition in \cref{eq:decompo1} can be equivalently rewritten using a sequence of indices $(m, i_1, \dots, i_m)$ as:
    \begin{equation}\label{eq:DecompositionIndices}
            f_\star(x) = \sum_{m=0}^{\infty} \> \sum_{1 \leq i_1 \leq \dots \leq i_m \leq d} \theta_{i_1, \dots, i_m} e_{i_1, \dots, i_m}(x)
    \end{equation}
    where the indices $i_k$ map to the multi-index components via $\beta_j = \sum_{k=1}^m \delta_{i_k, j}$.
\end{lemma}
However, working with ordered summations $1 \leq i_1 \leq \dots \leq i_m \leq d$ can be difficult. When $d\to\infty$, we can switch to a free sum.
\begin{corollary}
        Assume $f_\star$ is not a sparse function, meaning that $\theta_{i_1, \dots, i_m}$ are non-zero. By extending $\theta_{i_1, \dots, i_m}$ to be a completely symmetric tensor with respect to index permutations, the expansion of the target function in the $d\to\infty$ limit can be rewritten as:
    \begin{equation}\label{eq:DecompositionTargetGeneric}
        f_\star(x) \asymp \sum_{m=0}^\infty \frac{1}{m!} \sum_{i_1, \dots, i_m = 1}^d \theta_{i_1, \dots, i_m} e_{i_1, \dots, i_m}(x) 
    \end{equation}
\end{corollary}
The actual formula for the bias will depend on the specific shape and values of the tensor $\theta_{i_1, \dots, i_m}$. In the following paper, we focus on the following case:
\begin{equation}\label{eq:DecompositionTargetFunction}
    \theta_{i_1, \dots, i_m} = \theta_m (i_1 \dots i_m)^{-\omega}  \quad \quad    \omega \in \mathbb{R}_{\geq 0},
\end{equation}
where the scalar $\theta_m$ depends only on $m$, and is chosen by imposing the normalization of the target function:
\begin{equation}
    ||f_{\star}||_{L^{2}(p_{X})}\coloneqq \mathbb{E}[f_\star(x)^2] = 1
\end{equation}
\begin{remark}
    Notice that the functional form of \cref{eq:DecompositionTargetFunction}  resembles a \textit{source condition} often assumed in the kernel literature, where the target coefficients are ordered according to the decreasing eigenvalues such that $\theta_k =\Theta( k^{-g})$. However, mapping this multi-dimensional tensor to an exact source condition is not straightforward, as sorting the product $(i_1 \cdots i_m)$ across all possible multi-indices into a single univariate sequence involves non-trivial combinatorial constraints, a challenge similarly highlighted by \citet{simon_webpage}.
\end{remark}
\begin{figure}[t]
    \centering
    \includegraphics[width=0.8\linewidth]{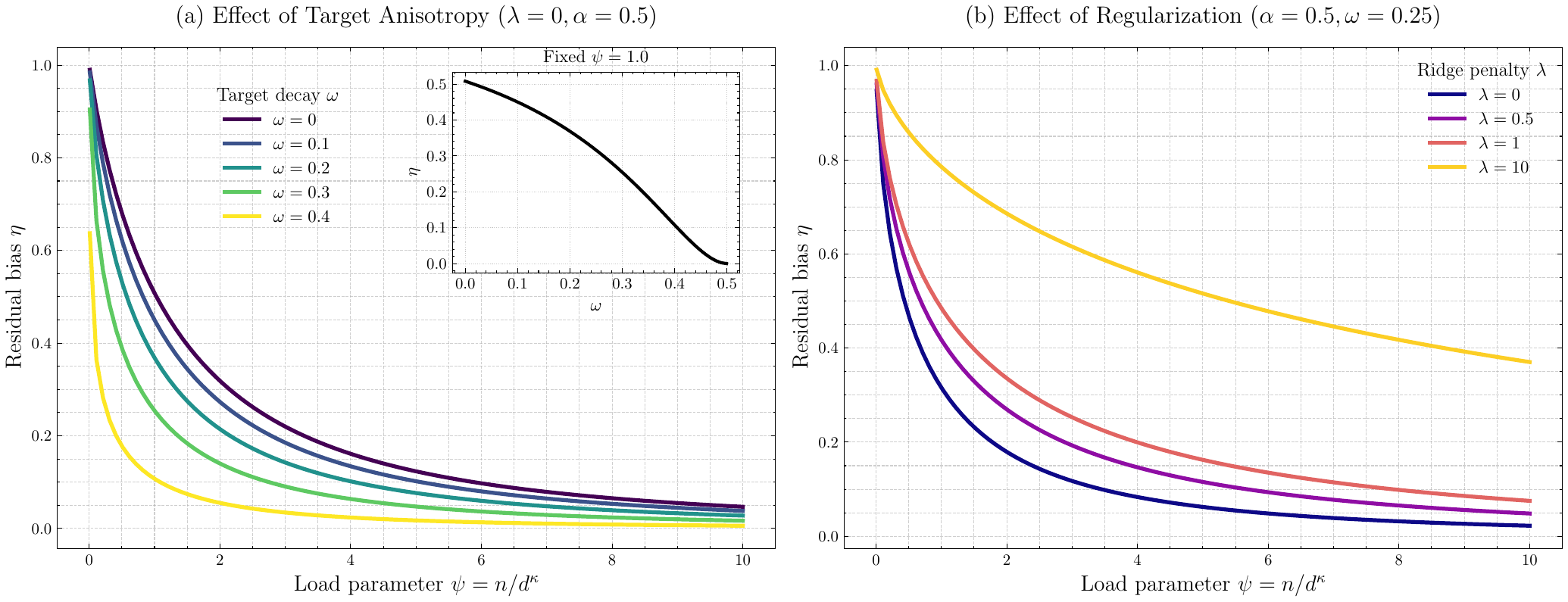}
    \caption{Residual bias $\eta$ as a function of the sample complexity $\psi$ for $\alpha = 0.5$, $\kappa=2$, $\sigma_\varepsilon = 4$ and polynomial kernel $h(t) = 1 + t + \frac{1}{2}t^2 + \frac{1}{6}t^3 + \frac{1}{24}t^4$, evaluated by solving eqs.~\cref{eq:KernelStateEquation}, \eqref{eq:VarianceWA}, and \eqref{eq:eta}. (\textbf{Left}) Varying $\omega$ for fixed $\lambda=0$. The \textit{inset} shows a cross-section of $\eta$ as a function of $\omega$ at a fixed load $\psi = 1.0$, highlighting how a stronger target anisotropy accelerates the learning of the $\kappa$-th shell. (\textbf{Right}) Fixed $\omega = 0.25$ for varying ridge $\lambda$. As expected, regularization increases the bias.}
    \label{fig:BiasWA}
\end{figure}
Denote by $E(m)$ the ``energy'' of the target function contained in the $m$-th shell, that is:
\begin{equation*}
    E(m) = \sum_{|\bm{\beta}| = m} \theta_{\bm{\beta}}^2 = \frac{1}{m!} \sum_{i_1, \dots, i_m = 1}^d \big(\theta_{i_1, \dots, i_m}\big)^2 = \|\text{P}_{m} f_\star\|^2_{L_2}
\end{equation*}
where $\text{P}_{m}$ is the projection onto the subspace spanned by the polynomials of degree $m$ in the kernel orthonormal basis. Under the assumption of Eq.~\eqref{eq:DecompositionTargetFunction}, we have the following result for the bias.
\begin{theorem}[\textit{Bias in the weakly anisotropic regime}]
\label{thm:bias:weak}
Consider the KRR setting from \Cref{s:Setting}, and let $f_\star \in L_2(p_X)$ denote a target function whose spectral decomposition satisfies \cref{eq:DecompositionTargetFunction} with a given decay rate $\omega \geq 0$. Let $\xi_\star$ denote the solution of the kernel state equation \cref{eq:KernelStateEquation} and $\tau$ given by \cref{eq:tauWA} as before. We have two distinct regimes for the bias:
    \begin{description}
        \item[a) Case $\bm{\omega > \frac{1}{2}}$:] Denoting $\kappa_{\text{eff}} = \frac{\kappa}{1-\alpha} \notin \mathbb{N}$, the bias is given by:
        \begin{equation}\label{eq:BiasWA>}
            \mathsf{B}(\omega, \lambda, \alpha) = \frac{1}{1-\tau(\alpha, \lambda)} \sum_{m > \kappa_{\text{eff}}} E(m)
        \end{equation}
        \item[b) Case $\bm{0 \leq \omega < \frac{1}{2}}$:] The bias is given by:
        \begin{equation}\label{eq:BiasWA<}
            \mathsf{B}(\omega, \lambda, \alpha) = \frac{1}{1-\tau(\alpha, \lambda)} \Bigg( \eta(\alpha, \omega) E(\kappa) + \sum_{m > \kappa} E(m) \Bigg) 
        \end{equation}
    where the quantity $0 \leq \eta \leq 1$, called \textit{residual bias fraction}, is defined as:
    \begin{equation}\label{eq:eta}
        \eta(\alpha, \omega) = \frac{1}{(\kappa-1)!} \int_0^\infty dt \> e^{-t} t^{\kappa-1} \frac{e^{2t\omega}(1-2\omega)^{\kappa}}{\big(1+ h_\kappa \kappa! \xi_\star^{-1} (1-\alpha)^{\kappa} e^{t\alpha}\big)^2}
    \end{equation}
    \end{description}
\end{theorem}
This theorem provides a detailed picture of how the kernel learns different degrees of the target function in the weak anisotropic regime.
\begin{figure}[t]
    \centering
    \includegraphics[width=0.8\linewidth]{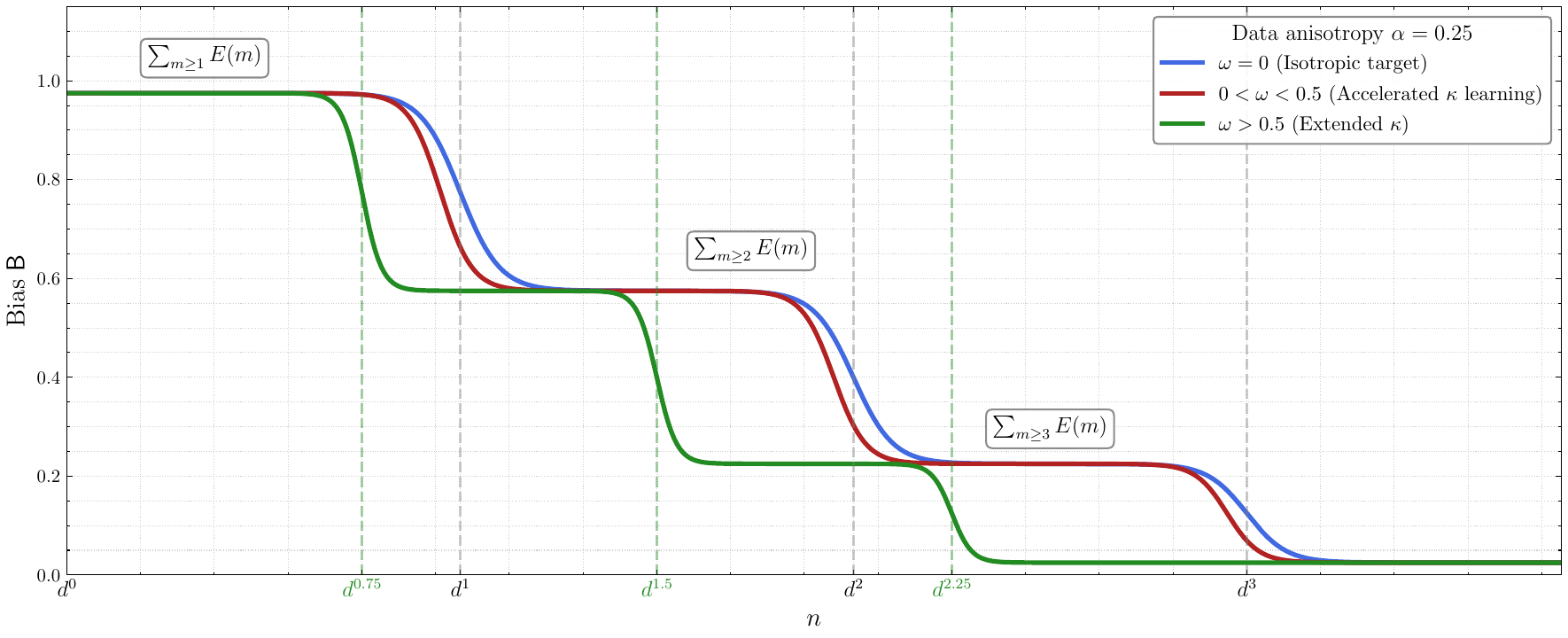}
    \caption{\textit{Schematic learning curves for the bias $\mathsf{B}_{n}$ in the weakly anisotropic regime, comparing different target functions with different $\omega$ but same energies $E(m)$ and fixed power-law exponent $\alpha = 0.25$. When $\omega < 1/2$, polynomial features are still learned at integer transitions $\kappa \in \mathbb{N}$. However, higher value of $\omega$ causes the $\kappa$-th sector to be learned more rapidly. When $\omega > 1/2$, the model can learn higher-order features at fractional scaling exponents $\kappa = m(1-\alpha) = 0.75m$. This accelerates the learning process: for example, given a sample size $n=O(d^{2-\epsilon})$, the isotropic model has only learned the linear part of the target, whereas the anisotropic model can learn a fraction of the quadratic components of the target.}}
    \label{fig:LearningCurveWA}
\end{figure}
\begin{itemize}
    \item When the target coefficients slowly decays ($0 \leq \omega < \frac{1}{2}$), the kernel can learn at best the degree $\kappa=\tfrac{\log{n}}{\log{d}}$ components of the target function $f_{\star}$,\footnote{Notice that we strictly compare target functions with the same total norm $||f_{\star}||_{L^{2}(p_{X})}$ and $\alpha$. This is because the kernel basis $e_{i_1, \dots, i_m}$ depends on the kernel operator, which in turn depends on the specific data geometry and, finally, on $\alpha$. Hence, comparing the excess risk across different values of $\alpha$ at a fixed decay rate $\omega$ makes function comparisons analytically ill-posed}
    \emph{irrespective of the data anisotropy $\alpha\in[0,1)$}. The exact fraction of degree $\kappa$ components learned is quantified by the fraction $\eta\in[0,1]$, which depends on the sample complexity $\psi$. \Cref{fig:BiasWA} (left) illustrates the effect of increasing $\omega$, showing that the faster the target decays, the faster the kernel is able to learn a fraction of the degree $\kappa$ features. In the data-scarce limit ($\psi \to 0$), the sample size is insufficient to resolve the $\kappa$-th degree, yielding $\eta \to 1$ regardless of $\omega$. Conversely, in the data-rich limit ($\psi \to \infty$, while still maintaining $n=\Theta(d^\kappa)$), the model fully learns the $\kappa$-th order eigenfunctions, driving $\eta \to 0^{+}$. Horizontally concatenating the bias behavior across different scales $\kappa$ yields a staircase behavior of the bias where different degrees of the target function are resolved at integer $\kappa \in \mathbb{N}$, akin to the isotropic case \citep{misiakiewicz_spectrum_2022}.
    \item If the target coefficients strongly decay $\omega > \frac{1}{2}$, the kernel is able to learn features up to order $\lfloor \frac{\kappa}{1-\alpha} \rfloor$, even though $n=\Theta(d^\kappa)$. In other words, learning the target features of degree $m$ strictly requires $n=\Theta(d^{m(1-\alpha)})$ samples in this regime, instead of $n=\Theta(d^{m})$ in the $\omega<1/2$ or isotropic case $\alpha=0$. As we approach the strong anisotropy regime $\alpha\to 1^{-}$, this result suggests the predictor is able to resolve all degrees independently of $\kappa$.
\end{itemize}
\Cref{fig:LearningCurveWA} summarize this discussion on the bias. 

Note that since the variance is by definition independent of the target, this decoupling admits a natural reading in terms of overfitting. In eq.~\eqref{eq:DetEquivalent}, a direction $\beta$ shifts from contributing to the bias to contributing to the variance as its eigenvalue $\lambda_\beta$
crosses the effective regularization $\nu_\star$: for $\lambda_\beta \ll \nu_\star$ the direction
is unresolved and its target component is missed, while for $\lambda_\beta \gg \nu_\star$ it is fully fit and contributes a unit to the variance trace, with both effects overlapping in a narrow window around $\lambda_\beta \approx \nu_\star$. In the isotropic case, degeneracy forces an entire degree-$m$ shell through this window at once, so capturing its signal and paying its full variance cost coincide at $n = \Theta(d^m)$. Anisotropy spreads the shell's eigenvalues instead: a target aligned with the data's principal directions concentrates on the few directions that cross the window earliest, resolving them at a small variance cost, while the bulk of the shell, largely orthogonal to the target, only crosses later, at the same threshold $n = \Theta(d^m)$, producing a variance peak no longer accompanied by a matching gain in bias. This yields two qualitatively different learning curves depending on the target's alignment, which we illustrate in \cref{fig:RiskWA}.
\begin{figure}[t]
    \centering
    \includegraphics[width=0.8\linewidth]{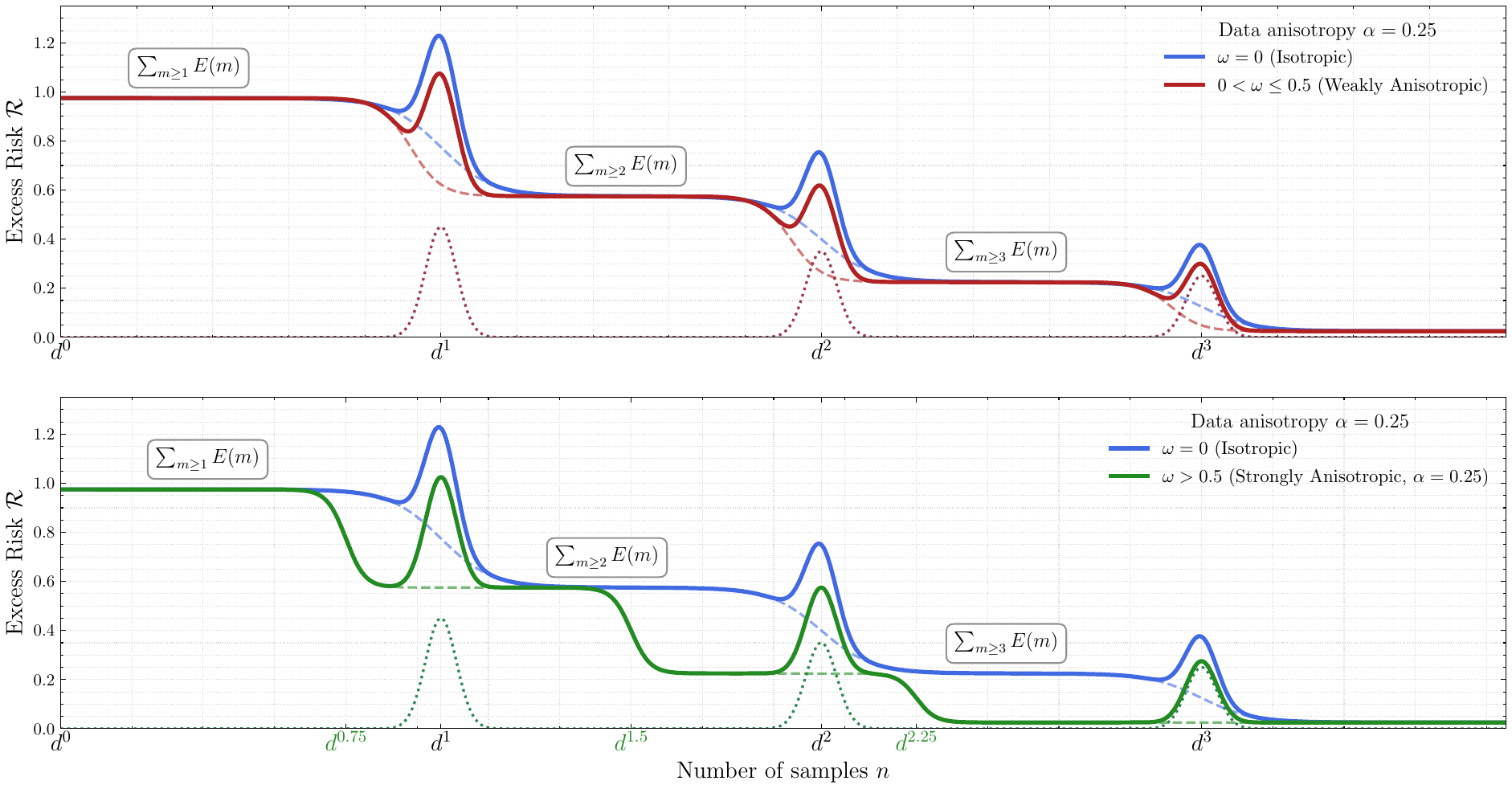}
    \caption{Schematic learning curves for the excess risk in the weakly anisotropic regime ($\alpha < 1$). The figure compares a  target function with a slow decay $0 < \omega < \frac{1}{2}$ (top panel) and a fast decay $\omega > \frac{1}{2}$ (bottom panel) against the isotropic baseline ($\omega = 0$). Solid lines represent the total excess risk, dashed lines indicate the bias contribution, and dotted lines show the variance. (\textbf{Top}): For $0 < \omega < \frac{1}{2}$, the variance peaks exactly at integer polynomial scaling ($n \asymp d^m$), in-phase with the resolution of new polynomial features. Here, the target decay $\omega$ accelerates the learning of the active $\kappa$-th shell. (\textbf{Bottom}): For $\omega > \frac{1}{2}$, a decoupling occurs. The bias drops at fractional scales $n \asymp d^{m(1-\alpha)}$, while the variance continues to peak at integer scalings.}
    \label{fig:RiskWA}
\end{figure}

%% file: 4_2_Strong_Anisotropic_Regime.tex
\subsection{Strongly anisotropic regime}
We now turn our attention to the strongly anisotropic regime ($\alpha > 1$). A fundamental feature of this regime is that the effective dimension $r_\alpha$, introduced in Eq.~\eqref{eq:EffectiveDimension}, no longer diverges but instead converges to a finite constant $r_\alpha = \zeta(\alpha)$ as $d \to \infty$. This reflects the fact that even though the data is supported in $\mathbb{R}^{d}$, it is effectively finite dimensional. 

\subsubsection{Variance}
\begin{proposition}[Variance in the strongly anisotropic regime]\label{th:VarianceSA}
    Consider the KRR problem in the setting introduced in \Cref{s:Setting} in the strong anisotropy regime $\alpha > 1$. Then, in the high-dimensional limit ($d \to \infty$), the deterministic equivalent of the variance is given by:
    \begin{itemize}
        \item \textbf{Ridgeless limit ($\lambda = 0^{+}$):} The variance reduces to exactly:
        \begin{equation}
            \mathsf{V}_{n}(\lambda=0^{+},\sigma^{2}_{\varepsilon}) = C\sigma_\varepsilon^2 (\alpha - 1)
        \end{equation}
        for a constant $C$ independent of $\psi$ and $\kappa$ (and consequently, independent of $n$).
        
        \item \textbf{Regularized regime ($\lambda \asymp 1$):} If the ridge penalty is strictly positive and independent of $d$, the variance vanishes is asymptotically of order:
        \begin{equation}
            \mathsf{V}_{n}(\lambda, \sigma^{2}_{\varepsilon}) = \tilde{\Theta}\left(d^{-\kappa\left(1-\frac{1}{\alpha}\right)}\right) = \tilde{\Theta}\left(n^{-1+\frac{1}{\alpha}}\right)
        \end{equation}
    \end{itemize}
\end{proposition}
Here, we retrieve two expected limits from the source \& capacity literature on kernels \citep{defilippis_dimension-free-2024}, but from a different angle.
Indeed, for $\alpha > 1$ the effective dimension of the data $r_{\alpha}=\Theta_{d}(1)$ is constant, meaning that the polynomial growth of samples ($n \sim \psi d^\kappa$) effectively results in an infinitely large dataset compared to the dimension, and thus the classical large-$n$ asymptotics in the statistical analysis of kernel methods.

\subsubsection{Bias}\label{sssection:BiasSA}
Consider a setting similar to the one discussed in \Cref{sssection:BiasWA} where the decay of the target coefficients is dictated by $\omega>0$, see \cref{eq:DecompositionTargetGeneric} and \cref{eq:DecompositionTargetFunction}. As hinted by the extrapolation of \Cref{thm:bias:weak}, the bias in the strong anisotropic regime can be essentially characterized by a continuous rate in which the target coefficients are progressively learned as the sample complexity is increased.
\begin{theorem}[Bias in the strongly anisotropic regime]
\label{thm:bias:strong}
Let the target decay rate be $\omega \geq 0$. Depending on the explicit ridge penalty $\lambda$, the macroscopic bias $\mathsf{B}(\omega, \alpha)$ exhibits the following asymptotic behaviour:
\begin{itemize}
    \item \textbf{Ridgeless limit ($\lambda = 0$):} 
    \begin{equation}\label{eq:BiasSARidgless}
        \mathsf{B} =
        \begin{cases}
            C\alpha \sum\limits_{m \geq \kappa} E(m) & \text{if} \quad 0 \leq \omega < \frac{1}{2} \\[8pt]
            \tilde{\Theta}\left(d^{-\kappa(2\omega - 1)}\right) & \text{if} \quad \frac{1}{2} < \omega \leq \alpha + \frac{1}{2} \\[8pt]
            \tilde{\Theta}\left(d^{-2\kappa\alpha}\right) & \text{if} \quad \omega > \alpha + \frac{1}{2}
        \end{cases}
    \end{equation}

    \item \textbf{Regularized regime ($\lambda =\Theta_{d}(1)$):} 
    \begin{equation}\label{eq:BiasSARegularized}
        \mathsf{B} =
        \begin{cases}
            C_{1}\mathbb{I}_{\left\{\frac{\kappa}{\alpha} \in \mathbb{N}\right\}} \eta' E\left(\frac{\kappa}{\alpha}\right) + C_{2}\sum_{m > \frac{\kappa}{\alpha}} E(m) & \text{if} \quad 0 \leq \omega < \frac{1}{2} \\[8pt]
            \Theta\left(d^{-\frac{\kappa(2\omega - 1)}{\alpha}}\right) & \text{if} \quad \frac{1}{2} < \omega < \alpha + \frac{1}{2} \\[8pt]
            \Theta\left(d^{-2\kappa}\right) & \text{if} \quad \omega > \alpha + \frac{1}{2}
        \end{cases}
    \end{equation}
    for constants $C, C_{1},C_{2}$, where $\mathbb{I}_{\{\cdot\}}$ denotes the indicator function, and $\eta' \in (0,1)$ is given by:
    \begin{equation}\label{eq:eta'}
            \eta' = \frac{(1-2\omega)^{\kappa'}}{(\kappa'-1)!} \int_0^{\infty} dt \> e^{-t} t^{\kappa'-1} \frac{e^{2t\omega}}{(1+ h_{\kappa'} \kappa'! \zeta(\alpha)^{-\kappa'} \frac{\psi}{\lambda} e^{\alpha t} )^2}
    \end{equation}
    with $\kappa' = \frac{\kappa}{\alpha}$
\end{itemize}
\end{theorem}
Again, \Cref{thm:bias:strong} give us a fine analysis of what components of the target are learned by the kernel, as the sample complexity is increased.
\begin{itemize}
    \item For a slowly decaying target $\omega\in[0,1/2)$, a finite portion of the target function cannot be learned. As in the weakly anisotropic case, in the ridgeless case $(\lambda=0)$, the model fails to resolve target components $m > \kappa$; however, different from $\alpha<1$, it also fails to learn the $m=\kappa$ components. Learning these $\kappa$-th degree features requires an amount of data $n\sim d^{\kappa+\epsilon}$ for arbitrarily small $\epsilon>0$, indicating a sharp learning transition. The fast spectral decay of the kernel effectively means its features are low-dimensional, meaning the kernel is not able to capture a high-dimensional target characterized by a very slow decay $\omega<1/2$. 

    A similar picture holds for finite regularization $\lambda= \Theta(1)$, with the crucial difference that now the kernel resolves features up to order $\lfloor\kappa/\alpha\rfloor < \kappa$. Notably, the finite ridge restores a smooth transition between polynomial degrees mediated by the fraction $\eta'(\omega)$, which activates exactly at the resonant conditions $\kappa = r\alpha, r\in\mathbb{N}$. Ultimately, this dictates a slower learning process compared to the ridgeless case, with the $m$-th degree features being learned at a delayed sample complexity of $n\sim d^{m\alpha}$. This is a manifestation of the classical bias-variance trade-off: while the explicit ridge penalty drives the variance to zero (Theorem~\ref{th:VarianceSA}), it simultaneously increases the bias.
    \item When $\omega > 1/2$, the target rapidly decaying coefficients are aligned with the kernel's. In this regime, regardless of whether $\lambda = 0$ or $\lambda = \Theta(1)$, the bias drops to zero with a decay rate that depends on $\omega$. Notably, a crossover occurs at $\omega = \alpha + 1/2$, which alters the convergence rate. When expressed in terms of $n=\Theta(d^{\kappa})$, the rates in this regime recover precisely to known rates under source \& capacity conditions by \citet{cui_generalization_2022}.
\end{itemize}
This discussion is summarized in \Cref{fig:SchematicaBiasSA}.
\begin{figure}[t]
    \centering
    \includegraphics[width=0.8\linewidth]{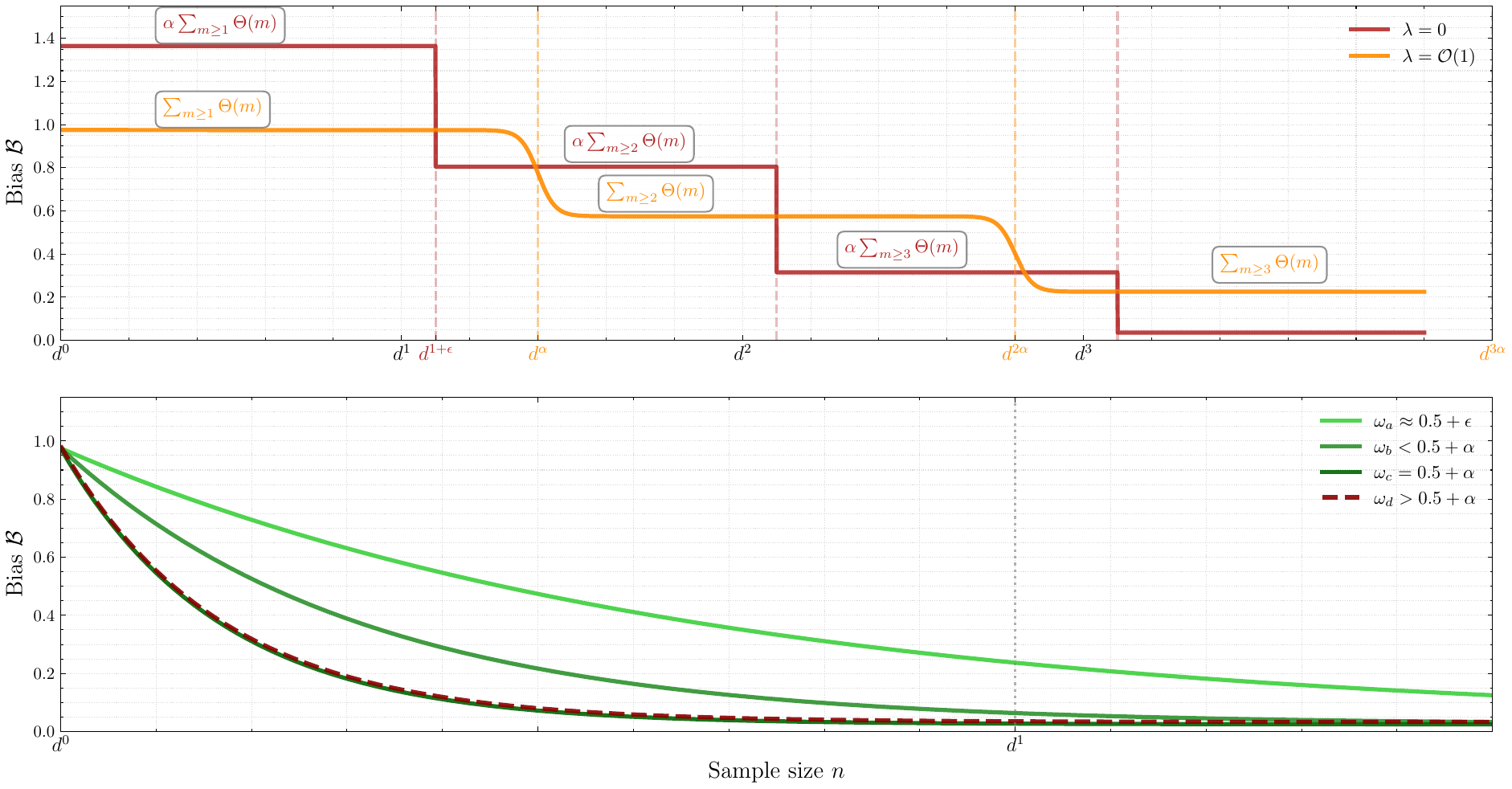}
\caption{Schematic learning curves for the bias $\mathsf{B}$ in the strongly anisotropic regime. (\textbf{Top}) $0 \le \omega < 1/2$: In the ridgeless limit ($\lambda = 0$, red), learning is instantaneous and occurs abruptly at $n =\Theta(d^{m+\epsilon})$. In the regularized regime ($\lambda = \Theta(1)$, orange), the ridge penalty shifts the activation of the $m$-th shell to a further delayed sample complexity $n=\Theta(d^{m\alpha})$. (\textbf{Bottom}) $\omega > 1/2$: The bias collapses to zero as a power-law (using $\lambda = 0$). Below the critical threshold $\omega < \alpha + 1/2$ (light and forest green), the learning rate continuously accelerates as the target energy compresses. Once the threshold is crossed at $\omega = \alpha + 1/2$ (dark green, red), the asymptotic decay rate saturates to $-2\kappa$}
    \label{fig:SchematicaBiasSA}
\end{figure}
\begin{remark}
    Note that the weakly and strongly anisotropic regimes continuously match at the critical point $\alpha = 1$. For instance, the unregularized variance, given by \cref{eq:VarianceWA} for $\alpha < 1$, is progressively suppressed as $\alpha \to 1^{-}$, ultimately vanishing. Analogously, in the strongly anisotropic regime ($\alpha > 1$), the variance is proportional to $\sigma_\varepsilon^2 (\alpha - 1)$, which also smoothly converges to zero as $\alpha \to 1^{+}$. The same can be said about the bias. Consider, for example, the case where $\omega < \frac{1}{2}$. For $\alpha < 1$, the learning of the $\kappa$-th shell depends on the parameter $\eta$. Taking the limit $\alpha \to 1^-$ in \cref{eq:eta} yields $\eta \to 1$, indicating that the $\kappa$-th shell remains entirely unlearned and behaves similarly to higher-order shells (which is the case for the $\alpha > 1$ case, Eq.~\eqref{eq:BiasSARidgless}). Thus, the theoretical predictions remain perfectly consistent across the boundary $\alpha = 1$.
\end{remark}
A comprehensive visual summary of the various analysed regimes, illustrating the different behaviours for $\alpha \lessgtr 1$ and $\omega \lessgtr 1/2$ for $\lambda = 0$, is presented in \Cref{fig:PhaseDiagram}.
\begin{figure}
    \makebox[\textwidth][c]{
        \input{figures/phaseDiagram}
    }
    \caption{\textit{Phase diagram summarizing the asymptotic bias $\mathsf{B}_{n}$ in the ridgeless limit ($\lambda=0$), highlighting the various regimes as $\alpha$ (data anisotropy) and $\omega$ (target decay rate) are varied. The red dashed lines represent the critical values where new qualitative behaviours appear.}}
    \label{fig:PhaseDiagram}
\end{figure}
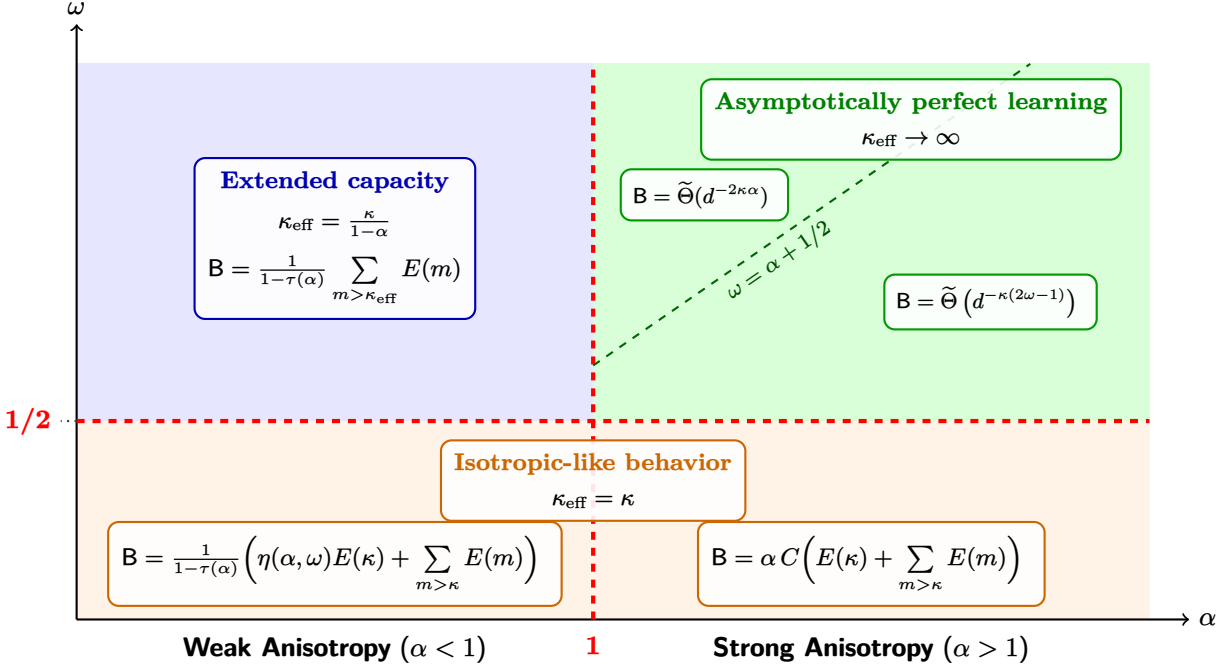

%% file: figures/phaseDiagram.tex
\begin{tikzpicture}[
    scale=1.05, 
    every node/.style={font=\sffamily},
    boxBase/.style={
        thick,
        fill=white,
        fill opacity=0.9,
        text opacity=1,
        rounded corners=4pt,
        inner sep=5pt,
        align=center,
        font=\small
    },
    boxBlue/.style={boxBase, draw=blue!70!black},
    boxGreen/.style={boxBase, draw=green!60!black},
    boxOrange/.style={boxBase, draw=orange!80!black}
    ]

    \fill[blue!10] (0, 2.5) rectangle (6.5, 7);
    \fill[green!15] (6.5, 2.5) rectangle (13.5, 7);
    \fill[orange!10] (0, 0) rectangle (13.5, 2.5);
    
    \draw[thick, ->] (0,0) -- (14,0) node[right, align=center] {$\alpha$};
    \draw[thick, ->] (0,0) -- (0,7.5) node[above, align=center] {$\omega$};

    \draw[dashed, red, ultra thick] (0, 2.5) -- (13.5, 2.5);
    
    \draw[dashed, red, ultra thick] (6.5, 0) -- (6.5, 7);
    
    \draw[dotted, thick, darkgray] (0,2.5) -- (-0.2,2.5) node[left, red, font=\bfseries] {1/2};
    \node[below, red, font=\bfseries] at (6.5, -0.1) {1};
    
    \node[below] at (3.25, -0.1) {\textbf{Weak Anisotropy} ($\alpha < 1$)};
    \node[below] at (10, -0.1) {\textbf{Strong Anisotropy} ($\alpha > 1$)};
    
    \node[boxBlue] at (3.25, 4.8) {
        \textbf{\textcolor{blue!70!black}{Extended capacity}}\\[1ex]
        $\kappa_{\text{eff}} = \frac{\kappa}{1-\alpha}$\\[1.5ex]
        $\mathsf{B} = \frac{1}{1-\tau(\alpha)} \sum\limits_{m > \kappa_{\text{eff}}} E(m)$
    };

    \node[boxOrange] at (6.5, 1.75) {
        \textbf{\textcolor{orange!80!black}{Isotropic-like behavior}}\\[0.5ex]
        $\kappa_{\text{eff}} = \kappa$
    };
    
    \node[boxOrange] at (3.25, 0.7) {
        $\mathsf{B} = \frac{1}{1-\tau(\alpha)}\Big( \eta(\alpha,\omega)E(\kappa) + \sum\limits_{m > \kappa} E(m) \Big)$
    };
    
    \node[boxOrange] at (10, 0.7) {
        $\mathsf{B} = \alpha \, C \Big( E(\kappa) + \sum\limits_{m > \kappa} E(m) \Big)$
    };
    \draw[thick, dashed, green!40!black] (6.5, 3.2) -- (12, 7) node[pos=0.4, sloped, above=-15pt, font=\footnotesize, text=green!30!black] {$\omega = \alpha + 1/2$};
    \node[boxGreen] at (10.5, 6.3) {
        \textbf{\textcolor{green!50!black}{Asymptotically perfect learning}}\\[0.5ex]
        $\kappa_{\text{eff}} \to \infty$
    };
    \node[boxGreen, scale=0.9] at (7.9, 5.35) {
        $\mathsf{B} = \widetilde{\Theta}(d^{-2\kappa\alpha})$
    };
    \node[boxGreen, scale=0.9] at (11.5, 4) {
        $\mathsf{B} = \widetilde{\Theta}\left(d^{-\kappa(2\omega-1)}\right)$
    };
    
\end{tikzpicture}

%% file: 5_Single_Index_Models.tex
\section{Case study: single-index models}
In \Cref{s:Gen}, we characterized the generalization properties of the kernel starting from assumptions on the target function $f_\star$ decomposition onto the kernel eigenbasis. We now discuss a particular case which has drawn recent interest in the literature, \textit{single-index models} (SIMs). This will allow us to illustrate how some of the behavior described in \Cref{s:Gen} can naturally emerge from a function directly defined in input space.  A SIM is a function that depends on the inputs $x\in\mathbb{R}^{d}$ only through a scalar function of its projection in one direction $v\in\mathbb{R}^d$:  
\begin{equation}\label{eq:SIMdefinition}
    f_\star( x) = g(v^{\top}  x).
\end{equation}
We refer to $v\in\mathbb{R}^d$ as the \emph{index}, and assume without loss of generality that it is normalized such that $\mathbb{E}[(v^\top  x)^2] = 1$. The scalar function $g : \mathbb{R}\to\mathbb{R}$ is known as the link function, and we assume it is a continuous function $g \in L_2(p_X)$. For this class of functions, we can explicitly compute their decomposition onto the multivariate Hermite polynomial basis in $L_2(p_X)$. This expansion is analogous to the one in Eq.~\eqref{eq:DecompositionIndices}, though the latter is carried out over the kernel eigenfunctions (which themselves form an orthonormal basis).

\begin{theorem}\label{thm:SIMDecomposition}
    Let $f_\star(x) = g(v^{\top}  x)$ be a single-index target function with $\mathbb{E}[(v^\top  x)^2] = 1$. Let $g \in L_2(p_X)$ admit the 1D orthonormal Hermite expansion $g(t)  = \sum_{m=0}^\infty \frac{g_m}{\sqrt{m!}}\He_m(t)$. Let $\alpha \geq 0$ and consider the generic decomposition of $f_\star$ onto the multi-variate Hermite eigenbasis:
    \[
        f_\star(x) = \sum_{m=0}^{\infty} \> \sum_{1 \leq i_1 \leq \dots \leq i_m \leq d} \theta_{i_1, \dots, i_m} H_{i_1, \dots, i_m}(\Sigma^{-1/2} x)
    \]
    Then, the tensor coefficients $\theta_{i_1, \dots, i_m}$ are given in terms of $v$ and $\alpha$ by
    \begin{equation}\label{eq:conversionSparse}
        \theta_{i_1, \dots, i_m} = g_m \sqrt{\frac{m!}{\beta_1!\dots \beta_d!}} \> r_\alpha(d)^{-m/2} (i_1 \dots i_m)^{-\alpha/2} \prod_{k=1}^m v_{i_k}  
    \end{equation}
    with $\beta_j = \sum_{k=1}^m \delta_{i_k, j}$. For dense functions and when $d\to\infty$, this expression can be further approximated with:
    \begin{equation}\label{eq:conversionDense}
        \theta_{i_1, \dots, i_m} \simeq g_m \sqrt{m!} \> r_\alpha(d)^{-m/2} (i_1 \dots i_m)^{-\alpha/2} \prod_{k=1}^m v_{i_k}       
    \end{equation}
\end{theorem}

\subsection{Hermite ansatz}
Even though we can explicitly write the decomposition of SIMs in the Hermite basis, writing them in terms of the kernel's basis is a challenging problem, as we don't have the explicit eigenfunctions of the kernel. We therefore consider the following ansatz, which postulates that the kernel
behaves, as far as the risk is concerned, as if its eigenfunctions were the Hermite
tensors $H_\beta$ with eigenvalues given by the leading term of Proposition~\ref{prop:spec:highd}. 
More precisely, let $c_\beta:=\langle f_\star,H_\beta\rangle_{L^2(p_X)}$ denote the Hermite coefficients of the target and $\theta_\beta=\langle f_\star,e_\beta\rangle$ the true kernel coefficient for the eigen-pair $(\lambda_{\beta}, \Phi_{\beta})$ in \Cref{proposition:asymptotic_matching_spectrum}. The deterministic equivalents \eqref{eq:DetEquivalent} depend on $\theta_\beta$ only through the cumulative target energy below a given eigenvalue: $A(u):=\!\!\sum_{\beta:\lambda_\beta\le u}\!\!\theta_\beta^{2}$. Denote by $\widehat A$ the same function built
from $c_\beta$.
\begin{conjecture}[Hermite ansatz]\label{conjecture:hermite_ansatz}
Let $\alpha\ge0$ and let $k$ satisfy Assumption~\ref{assumption:h}. Then there exists
$\varepsilon(d)$, such that $\varepsilon(d) \to 0$ as $d\to\infty$ such that, for every $f_\star\in L^2(p_X)$ with
$\|f_\star\|_{L^2(p_X)}=1$,
% \begin{equation}
%   \widehat A\big(u\,e^{-\varepsilon_d}\big)-\varepsilon_d
%   \;\le\; A(u)\;\le\;
%   \widehat A\big(u\,e^{\varepsilon_d}\big)+\varepsilon_d
%   \qquad\text{for all } u>0,
% \end{equation}
\begin{equation}
  A(u)=(1+o_{d})\widehat A(u)
  \qquad\text{for all } u>0,
\end{equation}
\end{conjecture}
In words: the distribution of squared target coefficients is
asymptotically the same in the true kernel eigenbasis and in the Hermite tensor basis. Under Conjecture~\ref{conjecture:hermite_ansatz}, the deterministic equivalents in \cref{eq:DetEquivalent} may be evaluated on $(c_\beta)$ in place of
$(\theta_\beta)$, with $o_d(1)$ relative error.

We provide numerical evidence for \Cref{conjecture:hermite_ansatz} in Appendix \ref{appendix:HermiteAnsatz}. Moreover, in a recent work, \cite{karkada2025predicting} has provided extensive numerical support for for a more general version of this statement. More precisely, they have shown that the alignment of the target function with Hermite polynomials can capture the learning curves of orthogonally invariant kernels in the high-dimensional setting, what they refer to as the \emph{Hermite ansatz}. 
%We expect \Cref{conjecture:hermite_ansatz} to hold in the particular case where $\Tr(\Sigma^2) = o_{d}(1)$, as in \Cref{proposition:asymptotic_matching_spectrum}. 

\subsection{Learning a SIM under anisotropy}\label{ssec:LearningSIM}
Under \cref{conjecture:hermite_ansatz}, we now discuss the consequences of our results to learning SIM under power-law anisotropic Gaussian data. This will strongly depend on how the index $v\in\mathbb{R}^{d}$ is aligned with the principle directions of the data.
\paragraph{Localized target}
Consider the case where the index is aligned with a principle direction $j\in[d]$ of the data $v = \frac{1}{\sqrt{\sigma_j}} e_j$. Since this target is strictly localized (sparse), we must use the exact formulation of Eq.~\eqref{eq:conversionSparse} so that the tensor coefficients collapses to:
\begin{equation}
    \theta_{i_1,\dots,i_m} = g_m \delta_{i_1, j} \dots \delta_{i_m, j} 
\end{equation}
This implies that for each polynomial of degree $m$, only a single term in the eigendecomposition of $f_{\star}$ remains, i.e., the fully diagonal element where $i_1 = \dots = i_m = j$. The spectral energy contained within each shell is thus trivially $E(m) = g_m^2$.

To evaluate the learning behavior, we must distinguish not only between the data anisotropy regimes (weak or strong) but also based on the target direction $j$. Specifically, we analyze whether the target aligns with the leading principal components ($j = O(1)$) or is one of the last components ($j = O(d)$). For the sake of simplicity, we discuss the two absolute extremes: $j = 1$ and $j = d$.

\vspace{0.5em}

\textbf{a) Weakly anisotropic regime ($0 \leq \alpha < 1$):} 
Assuming the non-resonant condition $\frac{\kappa}{1-\alpha} \notin \mathbb{N}$, and letting $\xi_\star, \tau$ be the solutions to Eqs.~\eqref{eq:KernelStateEquation} and~\eqref{eq:tauWA} (which are independent of the target function), the bias evaluates to:
\begin{subequations}
\begin{align}
    \mathsf{B}(j=1, \alpha) &= \frac{1}{1-\tau} \sum_{m > \frac{\kappa}{1-\alpha}} g_m^2 \label{eq:BiasSparse_HeadWA} \\
    \mathsf{B}(j=d, \alpha) &= \frac{1}{1-\tau} \left( \frac{g_\kappa^2}{\big(1+\xi^{-1}_\star h_\kappa \kappa! (1-\alpha)^\kappa\big)^2} + \sum_{m > \kappa} g_m^2 \right) \label{eq:BiasSparse_TailWA}
\end{align}
\end{subequations}

\vspace{0.5em}

\textbf{b) Strongly anisotropic regime ($\alpha > 1$):} 
In this regime, the equations for the bias reduce to:
\begin{subequations}
\begin{align}
    \mathsf{B}(j=1, \alpha, \lambda = 0) &= \widetilde{\Theta}(d^{-2\kappa\alpha})  \label{eq:BiasSparse_HeadSA1} \\
    \mathsf{B}(j=1, \alpha, \lambda \asymp 1) &= \Theta(d^{-2\kappa}) \label{eq:BiasSparse_HeadSA} \\
    \mathsf{B}(j=d, \alpha, \lambda = 0) &=  C \alpha \sum_{m \geq \kappa} g_m^2 \\
    \mathsf{B}(j=d, \alpha, \lambda \asymp 1) &= \alpha \left( C_1 \mathbb{I}_{\left\{\frac{\kappa}{\alpha} \in \mathbb{N}\right\}}\frac{g_{\kappa /\alpha}^2}{\left(1+h_{\kappa / \alpha} (\kappa / \alpha)! \zeta(\alpha)^{-\kappa/\alpha}\psi \lambda^{-1}\right)^2} + C_2\sum_{m \geq \kappa / \alpha} g_m^2 \right) \label{eq:BiasSparse_TailSA}
\end{align}
\end{subequations}
\begin{remark}
    Note that the above results can also be recovered using the formalism of Sec.~\ref{sssection:BiasWA} where we modeled $\theta_{i_1, \dots, i_m}\propto (i_1 \dots i_m)^{-\omega}$. In particular, substituting $g_m^2 = E(m)$, one can easily show that the case $j = 1$ formally coincides with $\omega \to \infty$. Equivalently, the case $j = d$ coincides with the limit $\omega \to -\infty$. Indeed, one can show that $\eta$ as defined in Eq.~\eqref{eq:eta} converges to the prefactor of $g_\kappa^2$ in Eq.~\eqref{eq:BiasSparse_TailWA} when $\omega \to -\infty$. The same reasoning applies for Eq.~\eqref{eq:BiasSparse_TailSA} where the prefactor of $g_{\kappa / \alpha}$ is equivalent to the limit $\omega\to-\infty$ of $\eta'$ defined in Eq.~\eqref{eq:eta'}.
    % \bl{$-\infty$ or $0$?}\lr{Should be $-\infty$. Here I'm cheating a bit, I know. The theory above is stated for $\omega \leq 0$. Surprisingly enough, it also works if we formally extend $\omega < 0$, if you just trust the formulas (Arie told me it's something that one can consider, a sort of adversarial target). In addition, when $\omega = 0$ all components are equally important. But here, the only important component is the last one $j = d$. This should correspond to a power-law decay which actually grows as power law (so that the only main component is the last one, upon proper normalization)}
\end{remark}

\paragraph{Power-law target}
Let us now consider a scenario where the index follows a power-law:
\[
v_{j} = Z^{-1} j^{-\gamma/2}
\]
where $Z$ is a normalization constant ensuring the index is correctly normalized $\mathbb{E}[(v^\top  x)^2] = 1$, and $\gamma$ dictates the decay rate of the target features. Substituting this into the dense asymptotic approximation of Eq.~\eqref{eq:conversionDense}, we obtain:
\begin{equation}
    \theta_{i_1, \dots, i_m} \simeq g_m \sqrt{m!} \> r_{\alpha+\gamma}(d)^{-m/2} \> (i_1 \dots i_m)^{-\frac{\alpha+\gamma}{2}}
\end{equation}
where $r_{\alpha+\gamma}(d) = \sum_{k=1}^d k^{-(\alpha+\gamma)}$. Notice that this exact functional form for the tensor $\theta_{i_1, \dots, i_m}$ perfectly matches the power-law condition proposed in Eq.~\eqref{eq:DecompositionTargetFunction}, yielding an effective decay rate $\omega$ driven by both the intrinsic data anisotropy $\alpha$ and the target weight decay $\gamma$:
\[
    \omega = \frac{\alpha+\gamma}{2}
\]
Consequently, this case is essentially equivalent to the ones analyzed in Sec.~\ref{sssection:BiasSA} and Sec.~\ref{sssection:BiasWA}, and the bias is strictly dictated by the specific interplay between $\alpha$ and $\gamma$. In particular, if the data is strongly anisotropic ($\alpha > 1$) and assuming $\gamma \geq 0$, we automatically have $\omega > 1/2$ (the fast learning regime). In this scenario, the asymptotic generalization bias completely vanishes in the thermodynamic limit. Furthermore, the spectral energy contained within each polynomial shell is simply given by $E(m) = g_m^2$

\paragraph{Delocalized target}
Finally, we consider a scenario where the index is completely delocalized. More precisely, we take:
\[
    v\sim \mathcal{N}(0, I_{d}),
\]
In this case, the exact value of the bias will depend on the specific realization of the random index vector. However, we can meaningfully study the average bias, denoted as $\mathbb{E}_{v}[\mathsf{B}]$.

Note that coefficients $\theta_{i_1, \dots, i_m}$ also become a random tensor. The fundamental quantity required to compute the averaged bias is their second moment, given by:
\[
\mathbb{E}_{v}\left[\left(\theta_{i_1, \dots, i_m}\right)^2\right] = g_m^{2} m!\> r_\alpha(d)^{-m} (i_1 \dots i_m)^{-\alpha}
\]
This shows that the effective decay rate is exactly $\omega = \alpha / 2$, which falls under the the power-law target scenario with $\gamma = 0$, representing a signal with no intrinsic decay. Physically, this equivalence is highly consistent: when $\gamma = 0$, all deterministic components of $\theta_\star$ contribute uniformly and isotropically, a condition that the random delocalized target satisfies on average.

Ultimately, these results establish a general dictionary that maps the various single-index model configurations into the unified power-law framework $\theta_{i_1, \dots, i_m} \sim (i_1\dots i_m)^{-\omega}$ analysed in the previous sections, with an effective exponent $\omega$ strictly determined by the localization or decay properties of the target vector ${\theta}_\star$ (Fig.~\ref{fig:SIMMapping}).
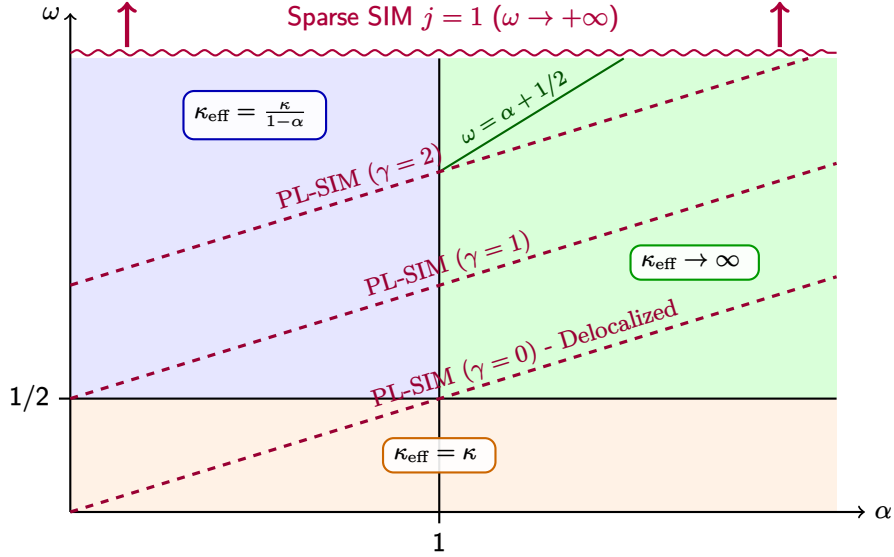
\begin{figure}
    \makebox[\textwidth][c]{
        \input{figures/SIMMapping}
    }
    \caption{\textit{Mapping of the various Single-Index Model (SIM) scenarios onto the general phase diagram of Fig.~\ref{fig:PhaseDiagram}, assuming $\lambda = 0$. Power-law single-index models (PL-SIMs) with target decay $\gamma$ yield an effective decay rate $\omega = \frac{\alpha+\gamma}{2}$, translating into a series of diagonal lines that intersect the different generalization regimes. The sparse SIM limit ($j = 1$) corresponds to $\omega \to \infty$, while the perfectly delocalized scenario is captured by $\omega = \alpha / 2$.}}
    \label{fig:SIMMapping}
\end{figure}

%% file: figures/SIMMapping.tex
\begin{tikzpicture}[
    scale=0.75, 
    every node/.style={font=\sffamily},
    boxBase/.style={
        thick,
        fill=white,
        fill opacity=0.9,
        text opacity=1,
        rounded corners=4pt,
        inner sep=4pt,
        align=center,
        font=\small
    },
    boxBlue/.style={boxBase, draw=blue!70!black},
    boxGreen/.style={boxBase, draw=green!60!black},
    boxOrange/.style={boxBase, draw=orange!80!black},
    simLine/.style={draw=purple!80!black, dashed, very thick}
    ]

    \fill[blue!10] (0, 2) rectangle (6.5, 8);
    \fill[green!15] (6.5, 2) rectangle (13.5, 8);
    \fill[orange!10] (0, 0) rectangle (13.5, 2);
    
    \draw[thick, ->] (0,0) -- (14,0) node[right, align=center] {$\alpha$};
    \draw[thick, ->] (0,0) -- (0,8.8) node[left, align=center] {$\omega$};

    \draw[thick, black] (0, 2) -- (13.5, 2);
    \draw[thick, black] (6.5, 0) -- (6.5, 8);
    
    \draw[thick] (0,2) -- (-0.2,2) node[left] {1/2};
    \draw[thick] (6.5,0) -- (6.5,-0.2) node[below] {1}; 
    
    \node[boxBlue] at (3.25, 7) {
        $\kappa_{\text{eff}} = \frac{\kappa}{1-\alpha}$
    };

    \node[boxOrange] at (6.5, 1) {
        $\kappa_{\text{eff}} = \kappa$
    };
    
    \node[boxGreen] at (11, 4.4) {
        $\kappa_{\text{eff}} \to \infty$
    };
    
    \draw[thick, green!40!black] (6.5, 6) -- (9.75, 8) node[pos=0.45, sloped, above=-2pt, font=\footnotesize, text=green!30!black] {$\omega = \alpha + 1/2$};

    \draw[ultra thick, ->, purple!90!black] (1, 8.2) -- (1, 9);
    \draw[ultra thick, ->, purple!90!black] (12.5, 8.2) -- (12.5, 9);
    \node[above, text=purple!90!black] at (6.75, 8.3) {Sparse SIM $j=1$ ($\omega \to +\infty$)};
    \draw[thick, purple!90!black, decorate, decoration={snake, amplitude=1pt, segment length=8pt}] (0, 8.1) -- (13.5, 8.1);

    \draw[simLine] (0,0) -- (13.5, 4.15) node[pos=0.6, sloped, above=-2pt, font=\small, text=purple!80!black] {PL-SIM ($\gamma = 0$) - Delocalized};

    \draw[simLine] (0, 2) -- (13.5, 6.15) node[pos=0.5, sloped, above=-2pt, font=\small, text=purple!80!black] {PL-SIM ($\gamma = 1$)};

    \draw[simLine] (0, 4) -- (13, 8) node[pos=0.4, sloped, above=-2pt, font=\small, text=purple!80!black] {PL-SIM ($\gamma = 2$)};

\end{tikzpicture}

%% file: 6_Conclusion.tex
\section{Conclusion}
In this work, we provided a high-dimensional characterization of the generalization properties of kernel ridge regression under anisotropic data, with a covariance with power-law decaying eigenvalues. By leveraging the exact asymptotic spectrum of inner-product kernels, we derived closed-form deterministic equivalents for both the bias and the variance in the ridgeless regime ($\lambda = 0$) or in the macroscopic ridge case $\lambda = \Theta_d(1)$. Depending on the power-law exponent $\alpha$, we identified two qualitatively distinct regimes. 

In the weakly anisotropic regime ($\alpha < 1$), the variance depends on the sample size $n$, exhibiting the characteristic multiple descent profile with interpolation peaks occurring at integer polynomial transitions ($n \sim d^\ell, \ell \in \mathbb{N}$). Within this regime, increasing $\alpha$ systematically dampens the magnitude of these peaks. In the strongly anisotropic regime ($\alpha > 1$), the variance becomes sample-size independent. In the high-dimensional limit, it either saturates to a strictly positive constant proportional to $\sigma_\varepsilon^2 (\alpha-1)$ in the ridgeless setting, or asymptotically collapses to $0$ upon the introduction of an order-one ridge penalty ($\lambda = \Theta_d(1)$). Notably, the two regimes reconcile continuously at the critical boundary $\alpha = 1$: taking either limit ($\alpha \to 1^\pm$) identically yields a vanishing variance ($\mathsf{V} \to 0$).

Regarding the bias, we assumed a specific structure of the target function upon expansion into the kernel eigenbasis (Eq.~\eqref{eq:DecompositionTargetFunction}), where the parameter $\omega$ governs the multi-dimensional power-law decay of the coefficients. A comprehensive visual summary of the resulting bias regimes, formulated as a phase diagram over the $(\alpha, \omega)$ parameter space, is presented in Fig.~\ref{fig:PhaseDiagram} for $\lambda = 0^+$.
When $\alpha < 1$ and $\omega < 1/2$, the kernel effectively resolves the polynomial features of the target function up to order $m < \kappa$. All higher-order features ($m > \kappa$) remain unlearned, while the current $\kappa$-th shell is only partially resolved, with its learned fraction strictly depending on the parameter $\omega$. Holding $\alpha$ fixed while increasing $\omega$ (equivalent to moving upward along a vertical line in the leftmost region of Fig.~\ref{fig:PhaseDiagram}) systematically reduces $\eta$, thereby decreasing the unlearned fraction of the $\kappa$-th shell. Crucially, as $\omega$ surpasses the critical threshold of $1/2$, a phase transition occurs: the kernel not only perfectly resolves the $\kappa$-th shell, but also unlocks higher-order polynomial features up to the extended effective capacity $\kappa_{\text{eff}} = \lfloor \frac{\kappa}{1-\alpha} \rfloor$.

In the strongly anisotropic regime ($\alpha > 1$) when the target energy is diffused ($0 \leq \omega < 1/2$), the kernel fails to learn any polynomial features of degree $m \geq \kappa$, causing the 
bias to plateau at a constant value independent of $\omega$. However, as soon as $\omega > 1/2$, the target function effectively aligns with the kernel's geometric structure, allowing the bias to decay toward zero. Within this learning phase, a further transition occurs at $\omega = \alpha + 1/2$, demarcating two distinct power-law decay rates for the finite-sample convergence.

\section*{Acknowledgements}
We would like to thank Elliot Paquette and Florent Krzakala for useful discussion. This work was supported by the French government, managed by the National Research Agency (ANR), under the France 2030 program with the project references ``ANR-23-IACL-0008'' (PR[AI]RIE-PSAI) and ``ANR-25-CE23-5660'' (MAPLE), as well as the Choose France - CNRS AI Rising Talents program. AW was also funded by the PSL Graduate Program in Computer Science. LR acknowledges the support of Scuola Galileiana di Studi Superiori and École Normale Supérieure during his visiting research period in Paris.

%% file: Appendix/Kernel_operator_properties.tex
\section{Properties of the Kernel Operator}
\label{section:kernel_operator}

In this section, we prove \Cref{proposition:asymptotic_matching_spectrum} from the main manuscript. 

\begin{proposition}[\Cref{proposition:asymptotic_matching_spectrum} in the main manuscript] 
\label{prop:spec:highd}
Assume $\Tr(\Sigma^2) = o_{d}(1)$, and $h_{k} > 0$ for all $k \in [D]$ in \Cref{assumption:h}. Then, the spectrum of the kernel in \Cref{assumption:h} indexed by multi-indices $\beta \in \ZZ^{d}_{\geq 0}$, $|\beta| \leq D $  satisfies: 
\begin{equation}
    \lambda_{\beta} = h_{|\bm\beta|} |\beta|! \, \sigma_1^{\beta_1} \dots \sigma_d^{\beta_d}( 1+ o_{d}(1) ) . 
\end{equation}
\end{proposition}

\begin{proof}
    Denote $z = \Sigma^{-\frac{1}{2}} x \sim \mathcal{N}(0,I_{d})$. We will follow the arguments of Appendix 1 in \citep{wortsman_kernel_2025}. Given $x, x' \in \RR^{d}$, we can write our kernel in the following form by expanding the inner product: 
    \begin{align}
        k(x,x') &= \sum_{k = 0}^{D} h_k \langle x, x'\rangle^{k} \\ 
        & = \sum_{k=0}^{D} h_k \sum_{\beta \in \ZZ^{d}_{\geq 0 }: |\beta| = k} \binom{k}{\beta} x^{\beta} x'^{\beta} \\ 
        & = \sum_{k=0}^{D} h_k \sum_{\beta \in \ZZ^{d}_{\geq 0 }: |\beta| = k} \binom{k}{\beta} \sigma^{\beta} z^{\beta} z'^{\beta},
    \end{align}
    where we denote $\sigma^{\beta} = \prod_{j=1}^{d} \sigma_i^{\beta_i}$. Let 
    \begin{equation}
    \Phi_{\beta}(x) =  \sqrt{h_k\binom{k}{\beta} \sigma^{\beta}} z^{\beta},
    \end{equation}
    and denote $\Phi(x) \in \RR^{\binom{D + d}{D}}$ the vector with this coordinates.  We can expand each coordinate into Hermite polynomials via the transformation precised in \citep{wortsman_kernel_2025}: 
    \begin{equation}
         \Phi_{\alpha}(x) =  \sqrt{h_k\binom{k}{\beta} \sigma^{\beta}} z^{\beta} = \sqrt{h_k\binom{k}{\beta} \sigma^{\beta}}\sum_{\gamma \leq \beta: \gamma \equiv_{2} \beta } \dfrac{\beta !}{2^{\frac{|\beta - \gamma|}{2}} \left( \frac{|\beta - \gamma|}{2}\right )! \sqrt{\gamma !}} He_{\gamma}(z). 
         \label{eq:monomial_to_hermite}
    \end{equation}
    Let $d_{\beta} = h_{|\beta|} |\beta|! \sigma^{\beta}$, and define $D = \mathrm{diag}\left ( (d_{\beta})_{|\beta| \leq D } \right )$. We also define the matrix $B \in \RR^{\binom{d +D }{D}}$  : 
    \begin{equation}
        B_{\beta, \gamma } : = \begin{cases}
            \sqrt{\dfrac{h_{|\beta|}|\beta|! \beta! }{h_{|\gamma|}|\gamma|! \gamma! }} \dfrac{\sigma^{\frac{\beta - \gamma}{2}}}{2^{\frac{|\beta - \gamma|}{2}} \left ( \frac{|\beta - \gamma|}{2} \right )! }  & \text{ if } \gamma \leq \beta \text{  and } \gamma \equiv_{2} \beta \\ 
            0 & \text{ else.}
        \end{cases}
    \end{equation}
    In particular, note that $B_{\beta, \beta} = 1$. Then by \Cref{eq:monomial_to_hermite} we have that: 
    \begin{equation}
        \Phi(x) = B D^{\frac{1}{2}}. 
    \end{equation}
    Now, assume that 
    \begin{equation}
        \| B - I\|_\mathrm{op} = o_{d}(1)
        \label{eq:B_minus_I}
    \end{equation}. Then we would have that: 
    \begin{equation}
        (1 - o_{d}(1))^{2} I  \preceq B^TB \preceq (1 + o_{d}(1) I,
    \end{equation}
    and therefore, multiplying by $D^{\frac{1}{2}}$: 
    \begin{equation}
        (1 - o_{d}(1))^{2} D^{\frac{1}{2}}  \preceq D^{\frac{1}{2}}B^TBD^{\frac{1}{2}} \preceq (1 + o_{d}(1) D^{\frac{1}{2}},
    \end{equation}
    and since the eigenvalues of the kernel are in the same as the eigenvalues of $D^{\frac{1}{2}}B^TBD^{\frac{1}{2}}$ by Eq. A.26 in \citep{wortsman_kernel_2025}, we would conclude the result. Therefore, we now focus on proving \Cref{eq:B_minus_I}. 

    Denote by $\mathcal{I}_{k}:= \{ \beta \in \ZZ^{d}_{\geq 0}: |\beta|=k\}$ for $k \in [D]$. Given $\ell, k \in [d]$, define 
    \[
    B_{k, \ell} = (B_{\beta, \gamma})_{|\beta|=k, \gamma = \ell}.
    \]
    Note that if $k - \ell$ is not even, then $B_{\ell, k} = 0$. Also, note that $B_{k,k} = I$. Let $\ell >0$, and let $u \in \RR^{\binom{d+k-1}{k}}$, so the coordinates of $u$ can be indexed by $(u_{\beta})_{\beta = k}$. Then we have: 
    \begin{align}
        (B_{k + 2\ell, \ell} u)_{\beta}& = \sqrt{\dfrac{h_{k + 2\ell} (k+2\ell)!}{h_{\ell} \ell!}} \sum_{\gamma \in  \ZZ^{d}_{\geq 0}:|\gamma|=\ell, 2\gamma \leq \beta} \dfrac{\sigma^{\gamma}}{2^{\ell} \gamma !} \sqrt{\dfrac{\beta!}{(\beta - 2\gamma)!}} u_{\alpha - 2\gamma}.
    \end{align}
    Then, since we can bound most of the coefficients of the sum by constants that no not depend on the dimension: 
    \begin{equation}
        \|B_{k + 2\ell, \ell} u\|^{2}_2 \leq C_{D} \sum_{|\beta|=k+2\ell}\sum_{\gamma \in  \ZZ^{d}_{\geq 0}:|\gamma|=\ell, 2\gamma \leq \beta} \sigma^{2\gamma} u_{\alpha - 2\gamma}^2.
    \end{equation}
    Applying Cauchy-Schwartz, we get: 
    \begin{equation}
         \|B_{k + 2\ell, \ell} u\|^{2}_2 \leq C \|u\|_2^{2 }  \sum_{|\gamma| = \ell} \sigma^{2\ell} \leq C \| u\|_{2} \left ( \Tr(\Sigma^2)\right )^{\ell }.
    \end{equation}
    Hence, we conclude: 
    \begin{equation}
        \| B_{k + 2\ell, \ell}\|_\mathrm{op} \leq \left ( \Tr(\Sigma^2)\right )^{\frac{\ell}{2}}.
    \end{equation}
    Then, by applying triangular inequality: 
    \begin{equation}
        \| B - I \|_\mathrm{op} \leq C_{D}\sum_{k = 0}^{D} \sum_{\ell = 1, k + 2\ell \leq D }^{k} (\Tr(\Sigma^2))^{\frac{\ell }{2}}. 
    \end{equation}
    Since we have a constant number of sums, we conclude:
    \begin{equation}
        \| B - I \|_\mathrm{op} \leq C_{D}(\Tr(\Sigma^2))^{\frac{1}{2}},
    \end{equation}
    and since $Tr(\Sigma^2) =o_{d}(1)$ by assumption, we conclude. 
\end{proof}

%% file: Appendix/Deterministic_Equivalents.tex
\section{Deterministic Equivalents}
\label{appendix:det_equivalents}

The bias-variance decomposition in \Cref{eq:BiasVariance} consists on two random terms. As discussed in \cref{s:Gen}, under specific assumptions, \cite{cheng_dimension_2025, misiakiewicz2024non, defilippis_dimension-free-2024} proved that this quantities concentrate around deterministic quantities, called \textit{deterministic equivalents}, and  presented in \Cref{eq:DetEquivalent}. In particular, \cite{misiakiewicz2024non} provides a abstract assumptions to derive deterministic equivalents for Kernel Ridge Regression. We will first briefly state these assumptions and then we will provide a proof of their validity in our setting.  

The idea behind the following assumptions is to separate the kernel eigenfunctions into lower are higher order. Given the $m$ largest eigenvalues, we denote by $\Lambda_{\leq m}$ the covariance of the associated eigenfunctions (which we will sometimes call the lower-degree part) and $\Lambda_{>m}$ the covariance associated to the higher eigenvalues (which we will also call the higher-degree part). The exposition of the different notions and assumptions in this section follows \cite{misiakiewicz2024non}. 

\begin{definition}[Ranks] Let $p \in \NN \cup \{ \infty\}$ be the rank of the kernel $k$. We introduce the following notions of rank for our covariance $\Lambda$: 
\begin{enumerate}
    \item \textbf{Regularized tail Rank:} For any integer $m \leq p$ and regularization parameter $\lambda \geq 0$, the regularized tail rank $r_{\lambda}(m)$ is defined by 
    \[
    r_{\lambda}(m) = \dfrac{\lambda + \sum_{j = m+1}^{p}\lambda_{j}}{\lambda_{m+1}}.
    \]
    We denote $\lambda_{>m}:= \lambda + \sum_{j = m+1}^{p}\lambda_{j}$. 
    \item \textbf{Effective Rank:} For any integers $n$ and $m \leq p$, the effective rank $r_{\mathrm{eff}, m}(n)$ of $(n, \Lambda_{\leq m})$ is defined as the smallest scalar such that $r_{\mathrm{eff}, m}\geq n$ and 
    \[
    r_{\mathrm{eff}, m}(n) \geq \dfrac{\sum_{j=k+1}^{m}\xi_j}{\xi_{k+1}} , \text{ for all } 0 \leq k \leq \min(n,m) -1. 
    \]
    \item \textbf{Relative Approximation bound}: For any $n, m \in \NN$ and ridge regularization parameter $\lambda \geq 0$, we define: 
    \begin{equation}
        \nu_{\lambda, m}(n) = 1 + \dfrac{\lambda_{\lfloor \eta n\rfloor, m} r_{\mathrm{eff}, m}(n) \sqrt{\log (r_{\mathrm{eff}, m}(n))}}{\lambda + \Tr(K_{>m})}. 
    \end{equation}
    for $\eta \in (0,\frac{1}{2})$, and $\lambda_{\lfloor \eta n\rfloor, m} = \lambda_{\lfloor \eta n \rfloor}$ if $\lfloor \eta x \rfloor \leq m$, and $0$ otherwise. 
\end{enumerate}
\label{definition:ranks}
\end{definition}

The first assumption we will need concerns the behaviour of the lower and higher order eigenfunctions.

\begin{assumption}[Concentration at $n \in \NN$] There exists constants $C_x,c_x, \beta_x > 0$ and $m \in \NN \cup \{\infty\}$, $m \leq p$, such that the regularized tail rank at $m$ satisfies 
\begin{equation}
r_{\lambda}(m) \geq 2n,
\label{eq:spectral_condition}
\end{equation}
and the following hold: 
\begin{enumerate}
    \item \textbf{Low-degree Features:} There exists $\varphi_1(m) >0$ such that for any deterministic vector $v \in \RR^{m}$, with $\| \Lambda^{\frac{1}{2}} v\|_2 < \infty$ and p.s.d matrix $A \in \RR^{m \times m}$ with $\mathrm{Tr}(\Lambda_{\leq m}A) < \infty$ we have 
    \begin{align*}
        \mathbb{P} \left ( \left | \langle v, \phi_{\leq m }\rangle \right | \geq t v^{t} \Lambda_{\leq m} v\right ) \leq C_{x}\exp \left( -c_{x}t^{\frac{2}{\beta_{x}}}\right ) \\ 
        \mathbb{P} \left ( \left | \phi_{\leq m}^T A \phi_{\leq m} - \mathrm{Tr}(\Lambda_{\leq m}A)\right | \geq t \varphi_{1}(m) \| \Lambda_{\leq m}^{\frac{1}{2}}A\Lambda_{\leq m}^{\frac{1}{2}}\right ) \leq C_{x}\exp \left( -c_{x}t^{\frac{1}{\beta_{x}}}\right )
    \end{align*}
    \item \textbf{High-degree Features:} There exists $p_{2,n}(m) \in (0,1)$ and $\varphi_{2,n}(m) \geq 1$ such that with probability at least $1- p_{2,n}(m)$, we have 
    \[
    \| \Phi_{> m} \Phi_{>m}^T - \mathrm{Tr}(\Lambda_{>m}) I_{n}\|_\mathrm{op} \leq \varphi_{2,n}(m) \sqrt{\dfrac{n}{r_{\lambda}(m)}}(\lambda+ \mathrm{Tr}(\Lambda_{\geq m}))
    \]
    \item \textbf{Target Function: } The high-degree part of the target function satisfies the tail bound
    \[
    \mathbb{P}\left ( |f_{\star, >m}(x) | \geq t \| f_{\star, >m}\|_{L^{2}}\right ) \leq C_x \exp (-c_{x} t^{\frac{2}{\beta}}). 
    \]
\end{enumerate}
\label{assumption:lower_and_higher_order}
\end{assumption}

At last, we need the following assumptions on the target function $f^\star(x)$. 

\begin{assumption}[Label noise]
The label noises $\varepsilon_i, i\in [n]$ are independent, mean-zero, $\sigma_{\varepsilon}^{2}$-sub-Gassian random variables. 
\label{assumption:label_noise}
\end{assumption}

We can now state the main theorem that allows us to rigorously state the quantities in \Cref{eq:DetEquivalent} as deterministic equivalents for the generalization error in our setting. 

\begin{theorem}[Theorem 1 in \citep{misiakiewicz2024non}]
Consider $D,K>0$, an integer $n$, a regularization parameter $\lambda \geq 0$, 
and a target function $f_*\in L^2(\mathcal{X})$  with parameters
\[
\|\theta_*\|_2=\|f_*\|_{L^2}<\infty.
\]
Suppose \Cref{assumption:lower_and_higher_order} is satisfied for some
\[
m:=m(n)\in\mathbb{N}\cup\{\infty\},
\]
and that $\{\varepsilon_i\}_{i\in[n]}$ satisfy \Cref{assumption:label_noise}. There exist
constants
\[
\eta:=\eta_x\in\left(0,\frac12\right),
\qquad
C_{D,K}>0,
\qquad
C_{x,\varepsilon,D,K}>0,
\]
such that, for all $n\geq C_{D,K}$ and $\lambda_{>m}>0$, if
\[
\lambda_{>m}\,\nu_{\lambda,m}(n)\geq n^{-K},
\qquad
\varphi_{2,n}(m)\sqrt{\frac{n}{r_\lambda(m)}}\leq\frac12,
\qquad
\varphi_1(m)\nu_{\lambda,m}(n)^8
\log^{3\beta+\frac12}(n)
\leq K\sqrt{n},
\]
then, with probability at least $1-n^{-D}-p_{2,n}(m)$, 
we have
\[
\left|
\mathcal{R}
\bigl(\theta_*;X,\varepsilon,\lambda\bigr)
-
\mathsf{R}_n(\theta_*,\lambda)
\right|
\leq
C_{x,\varepsilon,D,K}\,
\mathcal{E}_{R,n}(m)\,
R_n(\theta_*,\lambda),
\]
where the relative approximation rate is given by
\[
\mathcal{E}_{R,n}(m)
:=
\frac{
\varphi_1(m)\nu_{\lambda,m}(n)^6
\log^{3\beta+\frac12}(n)
}{
\sqrt{n}
}
+
\nu_{\lambda,m}(n)\varphi_{2,n}(m)
\sqrt{\frac{n}{r_\lambda(m)}}.
\]
\label{theorem:main_theorem_deterministic_equivalents}
\end{theorem}

\Cref{theorem:main_theorem_deterministic_equivalents} roughly tells us that, if we can prove both \Cref{assumption:lower_and_higher_order} and \Cref{assumption:label_noise}, then we can apply the deterministic equivalents from \Cref{eq:DetEquivalent}. Hence, before manipulating the quantities in \Cref{eq:DetEquivalent} by applying \Cref{th:KernelSpectrum}, we need to prove that \Cref{assumption:lower_and_higher_order} and \Cref{assumption:label_noise} are satisfied in our setting. 

In the following sections, we will prove the Assumptions of \Cref{theorem:main_theorem_deterministic_equivalents}. 
We will separate the spectrum into the weakly anisotropic setting $\alpha <1$ and the strongly anisotropic setting $\alpha >1$. 

Before beginning, we note that the rank of our kernel is $p := O(d^{D})$. 

\subsection{The Weakly Anisotropic Setting}

We begin by proving \Cref{assumption:lower_and_higher_order}. We begin by defining our threshold $m = m(n)$. We will only do the case $ 0 < \alpha < 1$. The case $\alpha = 0$ follows by the same arguments. 

\begin{lemma}
Let $\alpha \in (0,1)$, and let $n = \psi d^{\kappa}$, for $\psi >0$ independent of $d$ and $\kappa >0$ also independent of $d$, with $\lfloor \kappa \rfloor \in [1, D]$. Define: 
\begin{equation}
        m := \left |\{ k \in \NN : k\leq p,  \lambda_k > d^{-(\kappa + \rho)}\} \right |, 
\label{eq:def_m}        
\end{equation}
with $\rho >0$ such that $ 0 < \rho < (1-\alpha) (\lfloor \kappa \rfloor + 1-\kappa)$. Then, if $\kappa \not \in \NN$, \[
\exists \delta_0 >0 \text{ s. t } m \leq n^{1- \delta_0}.
\]
If $\kappa \in \mathbb{N}$, then 
\[
m \leq \Theta(n).
\]
\label{lemma:introduction_m}
\end{lemma}

\begin{proof}
   Define: 
    \begin{equation}
        m := \left |\{ k \in \NN : k\leq p,  \lambda_k > d^{-(\kappa + \rho)}\} \right |. 
    \end{equation}
    Note that $\lambda_{m + 1} = O (d ^{-(\kappa + \rho)})$ by construction. By Lemma C.10 from \cite{wortsman_kernel_2025}, we conclude the first equation. For the second case, when $\kappa \in \NN$, the set that defines $m$ necessarily contains the whole degree $\kappa$ shell, and the result follows. 
\end{proof}

\begin{lemma}
Let $m=m(n)$ as defined in \Cref{lemma:introduction_m}. Then, for any ridge regularization parameter $\lambda > 0$, 
\[
r_{\lambda}(m) \geq 2n. 
\]
Moreover, for $n = \Theta(d^{\kappa})$ the relative approximation bound $\nu_{\lambda, m} = 1$ when $\kappa \not \in \NN$, and $\nu_{\lambda, m} = O(\mathrm{poly}\log(n))$ when $\kappa \in \NN$.  
\label{lemma:regularized_eff_dim_condition}
\end{lemma}
\begin{proof}
    By definition, we have: 
    \begin{equation}
        r_\lambda(m) = \dfrac{1}{\lambda_{m+1}} \left ( \lambda +  \sum_{j \geq m +1 } \lambda_j \right ). 
        \label{eq:eff_dim_reg_1}
    \end{equation}
    Assuming the kernel has a shell after degree $\lfloor \kappa \rfloor$, and using the fact that $\lambda_{m+1} \leq C d^{-(\kappa + \rho) }$
    we get: 
    \[
    r_\lambda(m) \geq C d^{\kappa + \rho}.
    \]
    Therefore: 
    \begin{equation}
        \dfrac{r_{\lambda }(m)}{n} \geq Cd^{\rho} =\omega_{d}(1).
    \end{equation}
    Hence, for $d$ large enough, 
    \begin{equation}
        r_{\lambda}(m) \geq 2n .
    \end{equation}
    We now focus on proving the results for $\nu_{\lambda, m}$. Recall: 
    \begin{equation}
         \nu_{\lambda, m}(n) = 1 + \dfrac{\lambda_{\lfloor \eta n\rfloor, m} r_{\mathrm{eff}, m}(n) \sqrt{\log (r_{\mathrm{eff}, m}(n))}}{\lambda + \Tr(K_{>m})}.
    \end{equation}   
    If $\kappa \not \in \NN$, by \Cref{lemma:introduction_m} we have $m \leq n^{1-\delta_0}$. Therefore, $\lambda_{\lfloor \eta n\rfloor, m}=0$ and we conclude $\nu_{\lambda, m}(n) = 1$. 

    Now assume $\kappa \in \NN$. Since by \Cref{lemma:introduction_m} we have $m = \Theta(n)$, for any constant $\eta$, the index $\lfloor \eta n \rfloor$ lies in the $\kappa$-th shell. Then, by Corollary 2 in \citep{wortsman_kernel_2025}: 
    \begin{equation}
        \lambda_{\lfloor \eta n\rfloor } = O (d^{(1-\alpha)\kappa} \lfloor \eta n \rfloor ^{-\alpha }) = O \left ( d^{-(1-\alpha)\kappa } d^{-\kappa \alpha}\right ) = O \left ( d^{-\kappa}\right ). 
        \label{eq:lambda_lfloor}
    \end{equation}
    On the other hand, to compute $r_{\mathrm{eff}, m}(n)$ we have: 
    \begin{align}
        r_{\mathrm{eff}, m}(n)  & = \max_{0 \leq k \leq \min (m,n ) - 1} \left \{\dfrac{\sum_{j=k+1}^{m} \lambda_j}{\lambda_{k+1}} \right \} \\
        & \leq  \max_{0 \leq k \leq \min (m,n ) - 1} m \dfrac{\lambda_{k+1} }{\lambda_{k+1}} \leq m = \Theta(n). 
        \label{eq:bound_reff_m}
    \end{align}
    Finally, since $\lambda >0$, we can replace \Cref{eq:lambda_lfloor} and \Cref{eq:bound_reff_m} in the definition of $\nu_{\lambda, m}(n)$ to obtain: 
    \begin{equation}
        \nu_{\lambda, m}(n)  = O(\mathrm{poly}\log(n)). 
    \end{equation}
\end{proof}

So far, we have only proved the regularized effective dimension condition from \Cref{assumption:lower_and_higher_order}. We now proceed to prove the rest of the conditions. We begin by Lower-degree conditions. 

\begin{lemma}[Lower-Order Conditions in the Weakly Anisotropic Setting] Let $\alpha \in [0,1)$, and assume $n = \psi d^{\kappa}$ for $\psi, \kappa >0$ constants independent of the dimension. Then the lower-degree conditions of \Cref{assumption:lower_and_higher_order} are satisfied with 
\[
 \varphi_{1}(m) = O\left ( \sqrt{m}\mathrm{poly}\log(n)\right )
\]
\label{lemma:lower_order_weak}
\end{lemma}

\begin{proof}
    The condition on the lower-order eigenfunctions says that we need 
 \begin{align}
        \mathbb{P} \left ( \left | \langle v, \phi_{\leq m }\rangle \right | \geq t v^{t} \Lambda_{\leq m} v\right ) \leq C_{x}\exp \left( -c_{x}t^{\frac{2}{\beta_{x}}}\right ) \label{eq:linear_test}\\ 
        \mathbb{P} \left ( \left | \phi_{\leq m}^T A \phi_{\leq m} - \mathrm{Tr}(\Lambda_{\leq m}A)\right | \geq t \varphi_{1}(m) \| \Lambda_{\leq m}^{\frac{1}{2}}A\Lambda_{\leq m}^{\frac{1}{2}}\right ) \leq C_{x}\exp \left( -c_{x}t^{\frac{1}{\beta_{x}}}\right ). \label{eq:quadratic_test}
\end{align}

\citep{misiakiewicz2024non} notes that, in order to proof this condition it suffices to prove that the top eigenspaces are hypercontractive, that is, there exists constants $C,c$  such that for all $g \in \mathrm{span} \{\phi_j : j \in [m] \}$, and for any $q \geq 2$: 
\begin{equation}
    \| g\|^{2}_{L^{q}} \leq (C)^{c}\| g\|^2_{L^{2}}. \label{eq:hypercontractivity_low_order}
\end{equation}
For our cases, we are working with polynomial inner product kernels. Its not hard to prove that, under this assumption, the eigenfunctions of the kernel are polynomials (although not Hermite polynomials), and hence, satisfy \Cref{eq:hypercontractivity_low_order} with $C = 1$ and $c = D$. This implies \Cref{eq:linear_test} is satisfied by Markov's Inequality. 

On the other hand, as noted after Equation 25 in \citep{misiakiewicz2024non}, \Cref{eq:hypercontractivity_low_order} and the fact that we have Gaussian polynomials also implies \Cref{eq:quadratic_test}
is satisfied with 
\begin{equation}
    \varphi_{1}(m) = O\left ( \sqrt{m}\log(n)\right ).
    \label{eq:varphi_1_n}
\end{equation}
Therefore, Lower-order Conditions are satisfied in our setting. 
\end{proof}

We can actually improve the bound in \Cref{eq:varphi_1_n}. 
\begin{lemma}
    Let $e(x) = (e_{j}(x))_{j \in [m]}$ be the vector with the top $m$ eigenfunctions of the kernel. Then, there exists $\delta >0$ such that, for any matrix $A \in \RR^{m \times m}$
    \begin{equation}
        \mathbb{P}\left (\left | e(x) ^T A e(x) - \mathrm{Tr}(A)\right |\geq t \mathrm{poly}\log(n) \sqrt{m\mathrm{Tr}(\Sigma^2)}\| A \|_{F}\right ) \leq C \exp \left ( -c mt^{m}\right ). 
    \end{equation}
    Therefore, we can take
    \begin{equation}
        \varphi_{1}(m) = \tilde{O}\left ( \sqrt{m\mathrm{Tr}(\Sigma^2)}\log(n)\right )
    \end{equation}
    \label{lemma:varphi_1_improvement}
\end{lemma}

\begin{proof}
    Indeed, note that: 
    \begin{align}
        e ^T A \varphi(x) - \mathrm{Tr}(A) & = \phi(x)^T B_{m}^T A B_{m} \phi(x) - \mathrm{Tr}(A),
    \end{align}
    where $\phi(x) =  (phi_{j}(x))_{j \in [m]}$ is a vector with $m$ Hermite polynomials with matching multi-indices, and $B_{m}$ is a change of basis matrix. By the arguments in \Cref{section:kernel_operator}, we have: 
    \begin{equation}
    \|B_mB_m^T - I \|_{\mathrm{op}} \leq C \sqrt{\mathrm{Tr}(\Sigma^{2})}. 
    \label{eq:approximation_B_m}
    \end{equation}
    for $0 \leq \alpha \leq 1$. Therefore: 
    \begin{align}
        e(x) ^T A e(x) - \mathrm{Tr}(A) & = \phi(x)^T B_{m}^T A B_{m} \phi(x) - \mathrm{Tr}(A) \\
        & = \phi(x)^T B_{m}^T A B_{m} \phi(x) - \mathrm{Tr}(B_mAB_m) + \mathrm{Tr}((B_mB_m^T - I_m)A).
    \end{align}
    As discussed after Eq. 25 in \citep{misiakiewicz2024non}, for Hermite polynomials, for any p.s.d matrix $C \in \RR^{m \times m}$
    \begin{equation}
        \mathbb{P}\left (\left | \phi(x) ^T C \phi(x) - \mathrm{Tr}(C)\right |\geq t \mathrm{poly}\log(n) \sqrt{\dfrac{m}{d}}\| C \|_{F}\right ) \leq C \exp \left ( -c mt^{m}\right )
        \label{eq:bound_hermite_p}
    \end{equation}
    Since $\|B_{m}^T A B_{m}\|_\mathrm{F} \leq \| B_{m}^T B\|_{\mathrm{op}}\| A \|_{F} \leq C(D) A$, we already get the desired bound for this quantity.   On the other hand, by \Cref{eq:approximation_B_m}: 
    \begin{equation}
        \mathrm{Tr}((B_mB_m^T - I_m)A) \leq C\sqrt{m\mathrm{Tr}(\Sigma^2)} \| A\|_{F}.
        \label{eq:bound_trace_BA}
    \end{equation}    
    Putting \Cref{eq:bound_hermite_p} and \Cref{eq:bound_trace_BA} together, we conclude. 
\end{proof}

For higher order conditions the fact that eigenfunctions are polynomials also makes things easier, as hypercontractivity is a key property allowing to prove the higher order condition in \Cref{assumption:lower_and_higher_order}. More precisely, following \citep{misiakiewicz2024non}, we need to prove that \textbf{a)} The higher-order eigenspaces are hypercontractive, and \textbf{b)} The higher-order diagonal of the kernel matrix concentrates. The latter can be written as: 
\begin{equation}
    \exists \alpha_{1} >0, \text{ such that } \mathbb{E} \left [ \max_{i \in [n]} \left | K_{>m}(x_i, x_i ) - \mathbb{E} \left[ K_{>m}(x_i, x_i )\right ]\right |\right] \leq \alpha_{1}\sqrt{\dfrac{n \lambda_{m+1}}{\mathrm{Tr}(\Lambda_{>m})}} \mathbb{E}\left [K_{>m}(x_i, x_i) \right]. 
    \label{eq:diagonal_concentration}
\end{equation}

Since we already have hypercontractivity, we are left with proving \Cref{eq:diagonal_concentration}. 

\begin{lemma}[Concentration of Diagonal Features]
    Let $\alpha \in [0,1)$. Let 
    \[
    F_m (x) : = K_{>m}(x_i, x_i ),
    \]
    Then
    \[
    \mathbb{E} \left [(F_m - \Tr(K_{>m}))^{2} \right ] \leq C \| \Sigma \|_\mathrm{op}\Tr(K_{>m}).
    \]
    \label{lemma:variance_concentration}
\end{lemma}
\begin{proof}
    We first note that 
    \begin{equation}
        \mathbb{E} \left [(F_m - \Tr(K_{>m}))^{2} \right ] = \mathrm{Var}(F_m). 
    \end{equation}
    Therefore, we want to bound a variance. The Gaussian Poincaré Inequality \Cref{eq:Gaussian_Poincare} gives: 
    \begin{equation}
        \mathrm{Var}(F_m) \leq \| \Sigma\|_\mathrm{op } \mathbb{E}\left [\| \nabla F_m (x)\|^2 \right ]. 
        \label{eq:gaussian_poincare_app}
    \end{equation}
    Now, let $P_{>m}$ be the orthogonal projector into the higher degree features. We will now compute the gradient on the RHS. Let $j \in  [d]$. We have: 
    \begin{align}
        \partial_{j} F_m(x) = \partial_j \| P_{>m}\Phi \|^{2} = 2 \langle P \Phi(x), \partial_j \Phi(x) \rangle,
        \label{eq:inner_prod_derivative}
    \end{align}
    where $\Phi(x)$ is the vector of polynomial features, and the inner product is taken in $\RR^{p}$, for $p$ the range of the kernel. Let $\beta \in \ZZ^{d}_{\geq 0}$, with $|\beta| \leq D$. Then: 
    \begin{equation}
        \partial_j \Phi_{\beta}(x) = \partial_j \sqrt{h_{|\beta|} \binom{|\beta|}{\beta}} x^{\beta} = \mathbf{1}_{\beta_j >0}\sqrt{h_{|\beta|} \binom{|\beta|}{\beta}} \beta_j x^{\beta-e_j}. 
    \end{equation}
    Replacing in  \Cref{eq:inner_prod_derivative}: 
    \begin{equation}
        \partial_{j} F_m(x) = 2 \sum_{\ell = 1}^{D} \sum_{\beta \in  \ZZ^{d}_{\geq 0 }: |\beta| = \ell, \beta_j >0 } ( P \Phi(x))_{\beta} \sqrt{h_{|\beta|}} \beta_j x^{\beta-e_j}.
    \end{equation}
    By taking the norm and repeatidly using Cauchy-Schwarz, we have: 
    \begin{align}
        \|\nabla F_m(x)\|^2 & \leq  \sum_{j= 1}^{d} \left ( 2 \sum_{\ell = 1}^{D} \sum_{\beta \in  \ZZ^{d}_{\geq 0 }: |\beta| = \ell, \beta_j >0 } ( P \Phi(x))_{\beta} \sqrt{h_{|\beta|}} \beta_j x^{\beta-e_j}\right ) ^{2} \\
        & \leq 4 \sum_{j=1}^{d} \sum_{\ell = 1}^{D} h_\ell \left( \sum_{\beta \in  \ZZ^{d}_{\geq 0 }: |\beta| = \ell } \beta_j (P\Phi)_{\beta }^{2}\right ) \left( \sum_{\beta \in  \ZZ^{d}_{\geq 0 }: |\beta| = \ell, \beta_j >0 } \binom{\ell}{\beta} \beta_j x^{2\beta - 2e_j}\right ) \\
        & = 4 \sum_{\ell = 1}^{D} h_\ell \sum_{j=1}^{d} \left( \sum_{\beta \in  \ZZ^{d}_{\geq 0 }: |\beta| = \ell } \beta_j (P\Phi)_{\beta }^{2}\right ) \left( \ell \sum_{\beta \in  \ZZ^{d}_{\geq 0 }: |\beta| = \ell-1} \binom{\ell-1}{\beta} x^{2\beta }\right )  \\
        & = 4 \sum_{\ell = 1}^{D} \ell \| x\|^{2(\ell -1)}\sum_{j=1}^{d} \left( \sum_{\beta \in  \ZZ^{d}_{\geq 0 }: |\beta| = \ell } \beta_j (P\Phi)_{\beta }^{2}\right )  \\
        & \leq  4 \sum_{\ell = 1}^{D} h_\ell \ell^2 \| x\|^{2(\ell -1)} \| P\Phi_\ell\|_{2}^2 \\
        & \leq  \left ( \sum_{\ell = 1}^{D} h_\ell \ell^2 \| x\|^{2(\ell -1)}\right ) \left ( \sum_{\ell =1}^{D} h_\ell \| P\Phi_\ell \|_{2}^2 \right ) \leq  \left ( \sum_{\ell = 1}^{D} h_\ell \ell^2 \| x\|^{2(\ell -1)}\right ) F_m(x) .
    \end{align}
    The, taking expectation: 
    \begin{equation}
        \mathbb{E}\left [\|\nabla F_m(x)\|^2 \right ] \leq \mathbb{E} \left [ \left ( \sum_{\ell = 1}^{D} h_\ell \ell^2 \| x\|^{2(\ell -1)}\right ) F_m(x)\right ].
    \end{equation}
    Applying \Cref{lemma:moments_norm} and Cauchy-Schwarz, we conclude: 
    \begin{equation}
         \mathbb{E}\left [\|\nabla F_m(x)\|^2 \right ] \leq C_{D, h} \mathbb{E} \left[ F_m^{2}\right ]^{\frac{1}{2}} = C_{D,h} \Tr(K_{>m}). 
    \end{equation}
    Replacing this in \Cref{eq:gaussian_poincare_app}:
    \begin{equation}
         \mathrm{Var}(F_m) \leq C_{D,h}\| \Sigma\|_\mathrm{op} \Tr(K_{>m}), 
    \end{equation}
    and we conclude. 
\end{proof}

We can now prove \Cref{eq:diagonal_concentration}. 

\begin{lemma}[Diagonal Concentration] For $m$ as chosen in \Cref{lemma:introduction_m}, \Cref{eq:diagonal_concentration} holds with 
\[
\alpha_1 = \frac{ \mathbb{E} \left [ \max_{i \in [n]} \left |  F_m(x_i) - \Tr(K_{>m}) \right |\right ] }{\sqrt{n\lambda_{m+1}\Tr(K_{>m})}}.
\]
\label{lemma:diagonal_concentration_1}
\end{lemma}

\begin{proof}
    By \Cref{lemma:variance_concentration} and \Cref{lemma:L2_bounds_diagonal}, we have that: 
    \begin{equation}
        \mathbb{E} \left[ (F_m - \Tr(K_{>m})^{2}\right ]^{\frac{1}{2}} \leq C_{D,h}\min \{ \Tr(K_{>m}), \sqrt{\| \Sigma\|_\mathrm{op} \Tr(K_{>m})} \}. 
    \end{equation}
    Then, by hypercontractivity: 
    \begin{equation}
         \mathbb{E} \left[ (F_m - \Tr(K_{>m})^{q}\right ]^{\frac{1}{q}} \leq C_{D,q, h}\min \{ \Tr(K_{>m}), \sqrt{\| \Sigma\|_\mathrm{op} \Tr(K_{>m})} \}.
    \end{equation}
    Then we have:
    \begin{align}
        \mathbb{E} \left [ \max_{i \in [n]} \left |  F_m(x_i) - \Tr(K_{>m}) \right |\right ] & \leq \mathbb{E} \left [ \max_{i \in [n]} \left |  F_m(x_i) - \Tr(K_{>m}) \right |^{q}\right ] ^{\frac{1}{q}} \\ 
        & \leq \mathbb{E} \left [ \sum_{i=1}^{n} \left |  F_m(x_i) - \Tr(K_{>m}) \right |^{q}\right ] ^{\frac{1}{q}} \\
        & \leq n^{\frac{1}{q}} \mathbb{E} \left[ \left |  F_m(x_i) - \Tr(K_{>m}) \right |^{q}\right ] ^{\frac{1}{q}} \\
        & \leq C_{D, q, h}n^{\frac{1}{q}} \min \{ \Tr(K_{>m}), \sqrt{\| \Sigma\|_\mathrm{op} \Tr(K_{>m})} \}.
    \end{align}
    Taking $q$ large enough, we conclude: 
    \begin{equation}
         \mathbb{E} \left [ \max_{i \in [n]} \left |  F_m(x_i) - \Tr(K_{>m}) \right |\right ] \leq C_{D,h}\mathrm{poly}\log(n) \min \{ \Tr(K_{>m}), \sqrt{\| \Sigma\|_\mathrm{op} \Tr(K_{>m})} \}. 
    \end{equation}
    Then, defining $\alpha_1 : = \frac{ \mathbb{E} \left [ \max_{i \in [n]} \left |  F_m(x_i) - \Tr(K_{>m}) \right |\right ] }{\sqrt{n \lambda_{m+1}\Tr(K_{>m})}}$, we have: 
    \begin{align}
        \alpha & = O \left( \sqrt{\| \Sigma \|_\mathrm{op}\dfrac{1}{n \lambda_{m+1}}} \right ).
    \end{align}
    Note that if $\kappa \in \NN$, then we get $\alpha  = o_{d}(1)$ since $\| \Sigma\|_\mathrm{op} = o_{d}(1)$. On the other hand, if $\kappa \not \in \NN$, since $\rho \leq 1- \alpha$, we can also conclude that $\alpha = O_{n}(1)$. Therefore, we conclude the Lemma. 
\end{proof}

By Proposition 1 in \cite{misiakiewicz2024non}  \Cref{lemma:diagonal_concentration_1} implies the concentration for the higher order part of the diagonal of the kernel matrix with a function 
\begin{equation}
\varphi_{2,n}(m) \leq \dfrac{C_{D}}{p_{2,n}} (\alpha_{1} + C_D \mathrm{poly} \log (n)), 
\label{eq:bound_phi_2_n}
\end{equation}
for any chosen $p_{2,n}$. 

Note that \Cref{assumption:label_noise} already assumes Part 3 of \Cref{assumption:lower_and_higher_order}. Therefore, we are only left with proving the spectral assumptions in \Cref{theorem:main_theorem_deterministic_equivalents}.

\begin{lemma}Let $\alpha \in (0,1)$, and assume $n = \psi  d^{\kappa}$ for $\psi, \kappa>0$. Let $m=m(n)$ as defined in \Cref{lemma:introduction_m}. Then: 
\begin{enumerate}
    \item $\lambda + \Tr (K_{>m}) \nu_{\lambda,m}(n) \geq n^{-K}$. 
    \item $\varphi_{2,m}(m) \sqrt{\dfrac{n}{r_{\lambda}(m)}} \leq \frac{1}{2}$. 
    \item $\varphi_1(m) \nu_{\lambda, m}^8(n) \log^{3 \beta + \frac{1}{2}}(n) \leq K \sqrt{n}$. 
\end{enumerate}    
\label{lemma:spectral_conditions_weak}
\end{lemma}

\begin{proof}
    By \Cref{lemma:regularized_eff_dim_condition}, $\nu_{\lambda, m}(n) = \tilde{O}(1)$ regardless of $\kappa \in \NN$ or not. Therefore,  the first condition gives: 
    \begin{equation}
        \lambda + \Tr (K_{>m}) \nu_{\lambda,m}(n) =  \Theta(1). 
    \end{equation}
    Therefore, for any $n$ large enough the first condition holds. 

    By \Cref{eq:bound_phi_2_n}, the second condition holds for $d$ large enough, as we have: 
    \begin{equation}
        \varphi_{2mn}(m) \sqrt{\dfrac{n}{r_{\lambda}(m)}} = o_{d}(1). 
    \end{equation}
    For the third condition, using again the fact that $\nu_{\lambda, m}(n) =\tilde{O}(1)$: 
    \begin{equation}
        \varphi_1(m) \nu_{\lambda, m}(n)^8 \log^{3 \beta + \frac{1}{2}}(n) = O(\varphi_1(m) \log^{3 \beta + \frac{1}{2}}(n)). 
    \end{equation}
    By \Cref{eq:varphi_1_n}, we conclude: 
    \begin{equation}
        \varphi_1(m) \nu_{\lambda, m}(n)^{8} \log^{3 \beta + \frac{1}{2}}(n) = O \left ( \sqrt{m} \mathrm{poly}\log (n) \right ).
    \end{equation}
    Then, for $\kappa \not \in \NN$, \Cref{lemma:introduction_m} tells us that $m \leq n^{1-\delta_0}$, so we can conclude. For $\kappa \in \NN$, we use the estimate from \Cref{lemma:varphi_1_improvement} to conclude for $d$ large enough. 
\end{proof}

Putting \Cref{lemma:introduction_m}, \Cref{lemma:regularized_eff_dim_condition}, \Cref{lemma:lower_order_weak}, \Cref{lemma:diagonal_concentration_1} and  \Cref{lemma:spectral_conditions_weak}, all the assumptions of \Cref{theorem:main_theorem_deterministic_equivalents} hold, and therefore conclude \Cref{prop:det}, that is, we can apply the deterministic equivalents in the weakly anisotropic setting. 

\subsection{The Strongly Anisotropic Setting}

The strongly anisotropic setting is different than the weakly anisotropic setting only by some subtle technicalities. We sketch the validity of the deterministic equivalents in this setting below. 

First, we can choose $m(n)$ as in \Cref{lemma:introduction_m}, however, we have to restrict $\rho \in (0, (\alpha -1) \kappa)$. We get $m \leq  n^{1-\delta_0}$ for some $\delta_0 >0$ regardless of $\kappa \in \NN$ or not. 

The conclusions for the effective dimensions in \Cref{lemma:regularized_eff_dim_condition} still hold. Indeed, by Corollary 2 in  \citep{wortsman_kernel_2025}, in this case, each shell almost follows a capacity condition: 
\begin{equation}
    \lambda_k = \Theta( \dfrac{1}{r_0(\Sigma)}k^{-\alpha} \mathrm{poly}\log(d)) = \Theta(k^{-\alpha}\mathrm{poly}\log(d)),
\end{equation}
because the normalizing constant $r_0(\Sigma) = \Theta_{d}(1)$ in this setting. Therefore, in this case: 
\begin{equation}
    \Tr(K_{>m}) = \sum_{k \geq m} \lambda_m \leq C \mathrm{poly}\log(m) \int_m^{\infty} x^{-\alpha} dx = C  \mathrm{poly}\log(m)m^{1-\alpha},
\end{equation}
since $\alpha >1$. Since $m$ grows with $n$, we have that $\Tr(K_{>m}) \to 0$ with $n$, and therefore we need the ridge regularization parameter $\lambda >0$ in order to have $\frac{r_m(\lambda)}{n} \not \to 0$ in \Cref{lemma:regularized_eff_dim_condition}.

The lower and higher order conditions, and the conditions in \Cref{lemma:spectral_conditions_weak} still hold in this setting. Therefore, we the deterministic equivalents are still valid as long as we have the ridge regularization $\lambda >c$ for some constant $c$ independent of $n,d$.

%% file: Appendix/Reparametrization_kernel_eigenvalues.tex
\section{Reparametrization of the kernel eigenvalues}\label{appendix:Reparametrization}
By Proposition~\ref{th:KernelSpectrum} and Eq.~\eqref{eq:SpectrumSpec}, we know that in the high-dimensional limit, the kernel eigenvalues can be indexed by a multi-index $\beta\in\mathbb{Z}_{\geq 0}^d$:
\[
\lambda_{\bm\beta} = h_{|\beta|} |\beta|! \, r_\alpha(d)^{-|\beta|} \prod_{j=1}^d j^{-\alpha\beta_j} 
\]
assuming $\sigma_j = r_\alpha(d)^{-1} j^{-\alpha}$. To evaluate the deterministic equivalents, it is convenient to reparameterize these eigenvalues using a new set of variables over which we can sum more naturally. Specifically, we can group the eigenvalues using $m = |\bm\beta|$ that can be seen as a single \textit{shell}. Within each shell of constant $m$, we can map $\bm\beta$ to a set of $m$ indices $i_1, \dots, i_m \in \{1, \dots, d\}$ with $1 \leq i_1 \leq \dots \leq i_m \leq d$ defined such that:
\begin{equation}\label{eq:DefinitionIndices}
    \beta_j = \sum_{k=1}^m \delta_{i_k, j} 
\end{equation}
Equivalently, the index $i_k$ simply tracks the active dimensions composing the multi-index $\bm\beta$. Using this representation, the eigenvalues can be rewritten as:
\[
\lambda_{i_1, \dots, i_m}^{(m)} = h_m m! \, r_\alpha(d)^{-m} (i_1 \cdots i_m)^{-\alpha} 
\]
where $m \in \mathbb{N}$ and the indices are ordered as $1 \leq i_1 \leq \dots \leq i_m \leq d$. Consequently, summations over the entire spectrum can be recast as:
\[
\sum_{\bm\beta} F(\lambda_{\bm\beta}) = \sum_{m=0}^\infty \> \sum_{1\leq i_1 \leq \dots \leq i_m \leq d} F\big(\lambda^{(m)}_{i_1, \dots, i_m}\big)
\]
for a generic function $F$. While conceptually useful, this formulation still presents analytical difficulties due to the constrained ordered summation. To circumvent this issue, we can relax the strict ordering and extend the sum to a free summation over all possible index permutations:
\begin{equation}\label{eq:BetaCorrection}
\sum_{1\leq i_1 \leq \dots \leq i_m \leq d} F\big(\lambda^{(m)}_{i_1, \dots, i_m}\big) = \frac{1}{m!} \sum_{i_1, \dots, i_m = 1}^d F\big(\lambda^{(m)}_{i_1, \dots, i_m}\big) \prod_{j=1}^d \beta_j!
\end{equation}
The combinatorial factor $1/m!$ accounts for the permutations of the $m$ free indices, while the correction term $\prod_{j=1}^d \beta_j!$ (with $\beta_j$ determined by Eq.~\eqref{eq:DefinitionIndices}) accurately compensates for the multinomial overcounting of "diagonal" elements, where two or more indices coincide. However, in the high-dimensional limit ($d \to \infty$), the number of tuples composed of at least one repeated index becomes strictly negligible. The sum is overwhelmingly dominated by bulk configurations where $\beta_j \in \{0, 1\}$, for which the correction factor reduces identically to $1$. Since the contribution of the diagonal elements is asymptotically subleading, we can drop the factorial terms and approximate the trace with a pure free summation:
\[
\sum_{1\leq i_1 \leq \dots \leq i_m \leq d} F\big(\lambda^{(m)}_{i_1, \dots, i_m}\big) \simeq \frac{1}{m!} \sum_{i_1, \dots, i_m = 1}^d F\big(\lambda^{(m)}_{i_1, \dots, i_m}\big) 
\]

A similar approach can be employed to represent the target function $f_\star$ in the kernel eigenbasis. As in Eq.~\eqref{eq:decompo1}, we can re-write $f_\star$ in the basis that diagonalises the kernel operator $e_{\beta}(x)$:
\[
f_\star(\bm x) = \sum_{\bm\beta \in \mathbb{Z}^d_{\geq 0}} \theta_{\beta}e_{\beta}(\bm x)
\]
Grouping these functions by their polynomial shell $m = |\bm\beta|$, we can rewrite the target function as:
\begin{equation}\label{eq:appendixBeta}
    f_\star(x) = \sum_{m=0}^\infty \sum_{|\beta| = m} \theta_{\beta} e_{\beta}(\bm x) = \sum_{m=0}^\infty \sum_{1 \leq i_1 \leq \dots \leq i_m \leq d} \theta_{i_1, \dots, i_m} e_{i_1, \dots, i_m} (\bm x) 
\end{equation}
To convert this strictly ordered sum into a free summation over $\{1, \dots, d\}^m$, the sequence of scalar weights $\theta_{i_1, \dots, i_m}$, originally defined only for ordered index tuples, must be symmetrically extended. By defining the extended weights to be invariant under any permutation of the indices $(i_1, \dots, i_m)$, we can safely drop the ordering constraint. Neglecting the multinomial corrections from the diagonal elements we obtain the free summation form:
\begin{equation}\label{eq:DecompositionTargetFunctionProved}
    f_\star(x) \simeq \sum_{m=0}^\infty \frac{1}{m!} \sum_{i_1, \dots, i_m = 1}^d \theta_{i_1, \dots, i_m} e_{i_1, \dots, i_m}(\bm x) 
\end{equation}

\begin{remark}
    The free summation in Eq.~\eqref{eq:DecompositionTargetFunctionProved} relies on the assumption that the bulk terms ($\beta_j \in \{0,1\}$) overwhelmingly dominate the trace. This holds strictly true for ``delocalized'' or ``dense'' target functions, where the spectral energy is distributed across a vast number of tuples. However, if $f_\star$ is highly concentrated on a few specific diagonal features (e.g., purely dependent on a single coordinate raised to the $m$-th power, implying $\beta_j = m$), the diagonal correction factor $\prod_{j=1}^d \beta_j!$ from Eq.~\eqref{eq:BetaCorrection} can no longer be ignored and we must take care of this correction.
\end{remark}

%% file: Appendix/Weak_Anisotropy.tex
\section{Weakly anisotropic regime}
In this section, our goal is to prove mathematically the formulaes shown in Sec.\ref{s:WA}
\subsection{Kernel State Equation}
\label{appendix:KernelStateEquation}
Consider Eq.~\eqref{eq:SelfConsistencyEquation} that reads:
\[
n - \frac{\lambda}{\nu} = \Tr(\Lambda(\Lambda + \nu I)^{-1})
\]
Under the polynomial scaling $n = \psi d^{\kappa}$ and using the results of Appendix \ref{appendix:Reparametrization}, we can unroll the $\Tr$ operator and rewrite the self-consistency equation for $\nu$ as
\begin{equation}\label{eq:DecompPsiPsimReg}
    \begin{aligned}
        \psi &= \frac{\lambda}{\nu}d^{-\kappa} + d^{-\kappa} \sum_{\bm\beta} \frac{\lambda_{\beta}}{\lambda_{\beta}+\nu_\star} = 
        \\ &=  \frac{\lambda}{\nu}d^{-\kappa}  +\sum_{m=0}^\infty \frac{d^{-\kappa}}{m!} \sum_{i_1, \dots, i_m = 1}^d \frac{1}{1+\nu(i_1 \dots i_m)^\alpha (h_m m!)^{-1} r_\alpha(d)^{m}} = \\
        &= \frac{\lambda}{\nu}d^{-\kappa}  +\sum_{m=0}^\infty \psi^{(m)}
    \end{aligned}
\end{equation}
where we define:
\[
\psi^{(m)} = \frac{1}{m!} \sum_{i_1, \dots, i_m = 1}^d \frac{d^{-\kappa}}{1+\nu(i_1 \dots i_m)^\alpha (h_m m!)^{-1} r_\alpha(d)^{m}}
\]
In the regime where $\alpha < 1$, then $r_\alpha(d) = (1-\alpha)^{-1}d^{1-\alpha}$ at leading order (Eq.~\eqref{eq:EffectiveDimension}). Let us denote by $\delta_m = h_m m!$, so that:
\begin{equation}\label{eq:RefPsimNu}
    \psi^{(m)} =  \frac{1}{m!} \sum_{i_1, \dots, i_m = 1}^d \frac{d^{-\kappa}}{1+\nu (i_1 \dots i_m)^\alpha d^{m(1-\alpha)} \delta_m^{-1} (1-\alpha)^{-m}}
\end{equation}
To solve the self-consistency equation $\psi = \frac{\lambda}{\nu}d^{-\kappa} + \sum_m \psi^{(m)}$, we need to formulate an ansatz on the form of $\nu$. In particolar, we will assume:
\[
\nu = \xi d^{-\kappa}
\]
where $\xi$ is of order $1$ with respect to $d$, $\xi \asymp 1$, and should only depend on $\alpha, \psi, \lambda, \{\delta_m\}$. Under this ansatz, the generic term $\psi^{(m)}$ becomes:
\begin{equation}\label{eq:RefPsim}
    \psi^{(m)} =  \frac{1}{m!} \sum_{i_1, \dots, i_m = 1}^d \frac{d^{-\kappa}}{1+ (i_1 \dots i_m)^\alpha d^{m(1-\alpha)-\kappa} \xi \delta_m^{-1} (1-\alpha)^{-m}}
\end{equation}
The behavior of $\psi^{(m)}$ depends on $m$. In particular, we can distinguish different regimes depending on the relative value of $m$ with respect to $\kappa\in\mathbb{N}$:

\paragraph*{Case a: $\bm{m < \kappa}$}\label{appendix:m<kappa} 
Since $i_k \geq 1$, then $(i_1 \cdots i_m)^\alpha \geq 1$. This means that we can bound Eq.~\eqref{eq:RefPsim}:
\begin{equation}
	\begin{aligned}
		\psi^{(m)} &\leq \frac{1}{m!} \sum_{i_1, \dots, i_m = 1}^d \frac{d^{-\kappa}}{1+ d^{m(1-\alpha)-\kappa} \xi \delta_m^{-1} (1-\alpha)^{-m}} = \\
		&=\frac{1}{m!} \frac{d^{m-\kappa}}{1+ d^{m(1-\alpha)-\kappa} \xi \delta_m^{-1} (1-\alpha)^{-m}}
	\end{aligned}
\end{equation}
Since $m < \kappa$, then $m(1-\alpha)-\kappa < 0$. This means that when $d\to\infty$, this upper bound simplifies to:
\[
\psi^{(m)} \leq \frac{1}{m!} \frac{d^{m-\kappa}}{1+ d^{m(1-\alpha)-\kappa} \xi \delta_m^{-1} (1-\alpha)^{-m}} \sim \frac{1}{m!}d^{m-\kappa} \Longrightarrow \psi^{(m)} = O(d^{m-\kappa})
\]
Similarly, we have that $(i_1 \dots i_m)^\alpha \leq d^{m\alpha}$ as $i_k \leq d$ so that:
\begin{equation*}
	\begin{aligned}
		\psi^{(m)} &\geq \frac{1}{m!} \sum_{i_1, \dots, i_m = 1}^d \frac{d^{-\kappa}}{1+ d^{m-\kappa} \xi \delta_m^{-1} (1-\alpha)^{-m}} = \\
		&=\frac{1}{m!} \frac{d^{m-\kappa}}{1+ d^{m-\kappa} \xi \delta_m^{-1} (1-\alpha)^{-m}} \, .
	\end{aligned}
\end{equation*}
Again, since $m < \kappa$, the term $d^{m-\kappa} \to 0$ as $d\to\infty$. Thus, the denominator converges to $1$:
\[
\psi^{(m)} \geq \frac{1}{m!} d^{m-\kappa} \implies \psi^{(m)} = \Omega(d^{m-\kappa}) \, .
\]
As a consequence,  $\psi^{(m)} = \Theta(d^{m-\kappa})$ and vanishes when $d\to\infty$ for $m < \kappa$. When summing all these terms up, we will obtain:
\[
\sum_{m < \kappa} \psi^{(m)} \asymp d^{(\kappa-1)-\kappa} = d^{-1} \to 0 \, .
\]
Hence the contribution on the full summation of the under-resonating $m < \kappa$ terms, as long as it amounts to a finite value, is negligible.
\paragraph*{The integral regime and the Irwin-Hall trick}
For values of $m \geq \kappa$, the direct bounding argument employed above is no longer sufficiently tight. However, as $d \to \infty$, the discrete indices $i_q$ can be treated as continuous variables $x_q \in [1, d]$. Since the summand varies slowly over the integer grid, we can approximate the discrete summation with an $m$-dimensional Riemann integral:
\begin{equation}
	\psi^{(m)} \simeq \frac{1}{m!} \int_1^d \cdots \int_1^d \mathrm{d}x_1 \cdots \mathrm{d}x_m \, \frac{d^{-\kappa}}{1+ (x_1 \cdots x_m)^\alpha d^{m(1-\alpha)} \nu C_m} \, .
\end{equation}
Upon introducing the change of variables $y_i = x_i \, C_m^{\frac{1}{m\alpha}}d^{\frac{1-\alpha}{\alpha}} \nu^{\frac{1}{m\alpha}}$ and recalling the ansatz $\nu = \xi d^{-\kappa}$, we obtain:
\begin{equation}\label{eq:PsimIntegral}
	\psi^{(m)} = \frac{1}{m!} C_m^{-1/\alpha}\xi^{-1/\alpha} d^{(m - \kappa) (1-1/\alpha)} \int_{y_{\min}}^{y_{\max}}\cdots\int_{y_{\min}}^{y_{\max}} \frac{\mathrm{d}y_1 \cdots \mathrm{d}y_m}{1+ (y_1 \cdots y_m)^\alpha} \, ,
\end{equation}
where the integration boundaries explicitly read:  
\begin{subequations}\label{eq:y_limits}
	\begin{align}
		y_{\min} &= C_m^{\frac{1}{m\alpha}} d^{\frac{1-\alpha}{\alpha}} \nu^{\frac{1}{m\alpha}} = C_m^{\frac{1}{m\alpha}} \xi^{\frac{1}{m\alpha}} d^{\frac{1}{\alpha}\left(1-\frac{\kappa}{m} - \alpha\right)} \, , \label{eq:min} \\
		y_{\max} &= C_m^{\frac{1}{m\alpha}} d^{\frac{1-\alpha}{\alpha}} \nu^{\frac{1}{m\alpha}} d = C_m^{\frac{1}{m\alpha}} \xi^{\frac{1}{m\alpha}} d^{\frac{1}{\alpha}\left(1-\frac{\kappa}{m}\right)} \, . \label{eq:max}
	\end{align}
\end{subequations}
The asymptotic behavior of $y_{\min}$ and $y_{\max}$ as $d\to\infty$ is illustrated in Fig.~\ref{fig:MinMaxWA}.

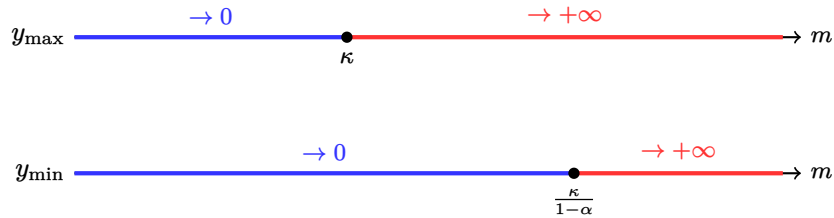
\begin{figure}[htbp]
	\centering
	\begin{tikzpicture}[
		scale=1.2,
		every node/.style={font=\small},
		dot/.style={circle, fill=black, inner sep=1.5pt}
		]
		\node[left, font=\bfseries] at (0, 1.5) {$y_{\max}$};
		\draw[thick, ->] (0, 1.5) -- (8, 1.5) node[right] {$m$};
		\draw[ultra thick, blue!80] (0, 1.5) -- (3, 1.5) node[midway, above=1pt] {$\to 0$};
		\draw[ultra thick, red!80] (3, 1.5) -- (7.8, 1.5) node[midway, above=1pt] {$\to +\infty$};
		\node[dot, label=below:{$\kappa$}] at (3, 1.5) {};
		
		\node[left, font=\bfseries] at (0, 0) {$y_{\min}$};
		\draw[thick, ->] (0, 0) -- (8, 0) node[right] {$m$};
		\draw[ultra thick, blue!80] (0, 0) -- (5.5, 0) node[midway, above=1pt] {$\to 0$};
		\draw[ultra thick, red!80] (5.5, 0) -- (7.8, 0) node[midway, above=1pt] {$\to +\infty$};
		\node[dot, label=below:{$\frac{\kappa}{1-\alpha}$}] at (5.5, 0) {};
	\end{tikzpicture}
	\caption{\textit{Asymptotic behavior of the integration boundaries $y_{\min}$ and $y_{\max}$ derived in Eqs.~\eqref{eq:y_limits}. Notably, when $m = \kappa$ or $m = \frac{\kappa}{1-\alpha} \in \mathbb{N}$, the respective boundary converges to a strictly positive, macroscopic constant $\Theta_d(1)$.}}
	\label{fig:MinMaxWA}
\end{figure}
Having established that the sub-critical terms vanish, we restrict our focus to the remaining shells. Before evaluating these specific branches, we establish a key mathematical identity that will repeatedly emerge:
\begin{lemma}[\textit{Irwin-Hall Trick}]\label{lemma:IrwinHall}
	Let $y_{\min}, y_{\max} \in \mathbb{R}^{+}$ such that $y_{\max} = y_{\min} d$. Consider the $m$-dimensional integral defined as:
	\begin{equation}\label{eq:IHBasis}
			\mathcal{I} = \int_{y_{\min}}^{y_{\max}}\cdots\int_{y_{\min}}^{y_{\max}} f(y_1 \cdots y_m) \, \mathrm{d}y_1 \cdots \mathrm{d}y_m \, .
	\end{equation}
	for some regular enough $f : \mathbb{R}\to \RR$. Then, denoting $l = \log(y_{\min})$ and $L = \log(y_{\max})$, this integral can be exactly rewritten through a 1-dimensional integral as:
	\[
	\mathcal{I} = \log(d)^{m-1} \int_{ml}^{mL} \mathrm{d}S \> e^{S} f(e^{S}) f_{IH}\left(\frac{S-ml}{\log(d)}; m\right) \, ,
	\]
	or, equivalently via variable substitution:
	\[
	\mathcal{I} = \log(d)^{m} \int_{0}^{m} \mathrm{d}x \> e^{x\log(d) + ml} f(e^{x\log(d)+ml}) f_{IH}(x; m) \, ,
	\]
	where $f_{IH}$ is the Irwin-Hall probability density function:
	\[
	f_{IH}(x; m) = \frac{1}{(m-1)!} \sum_{k=0}^{\lfloor x \rfloor} (-1)^k \binom{m}{k} (x-k)^{m-1} \, .
	\]
\end{lemma}

\begin{proof}
	Let us apply the change of variables $y_i = e^{u_i}$ to Eq.~\eqref{eq:IHBasis}:
	\begin{equation}\label{eq:IH1}
		\mathcal{I} = \int_l^L \cdots \int_l^L e^{u_1 + \dots + u_m} f(e^{u_1 + \dots + u_m}) \, \mathrm{d}u_1 \cdots \mathrm{d}u_m \, ,
	\end{equation}
	where $l = \log(y_{\min})$ and $L = \log(y_{\max})$. The core insight is to treat $u_1, \dots, u_m$ as independent random variables uniformly distributed in the interval $[l, L]$, allowing us to interpret Eq.~\eqref{eq:IH1} as an expectation value:
	\[
	\mathcal{I} = (L-l)^m \, \mathbb{E}_{u_i \sim \mathcal{U}[l,L]} \big[ e^{u_1 + \dots + u_m} f(e^{u_1 + \dots + u_m}) \big] \, .
	\]
	By defining the sum variable $S = \sum_i u_i$, which represents the sum of $m$ independent and uniformly distributed random variables, and noticing that $L-l = \log(y_{\max}/y_{\min}) = \log(d)$, we can rewrite this expectation as:
	\[
	\mathcal{I} = \log(d)^m \, \mathbb{E}_{S} \big[ e^{S} f(e^{S}) \big] \, .
	\]
	The probability distribution for the sum of $m$ independent standard uniform variables $z_i \sim \mathcal{U}[0,1]$ is a well-known result and corresponds to the Irwin-Hall distribution. From this base distribution, it is straightforward to deduce the probability density function for our shifted and scaled sum $S$:
	\[
	p_S(S) = \log(d)^{-1} f_{IH}\left(\frac{S-ml}{\log(d)}; m\right) \, ,
	\]
	where the standard Irwin-Hall density is defined as:
	\[
	f_{IH}(x; m) = \frac{1}{(m-1)!} \sum_{k=0}^{\lfloor x \rfloor} (-1)^k \binom{m}{k} (x-k)^{m-1} \, .
	\]
	Integrating over the support of $S$ yields the first form of the integral:
	\[
	\mathcal{I} = \log(d)^{m-1} \int_{ml}^{mL} dS \> e^{S} f(e^{S}) f_{IH}\left(\frac{S-ml}{\log(d)}; m\right) \, .
	\]
	Alternatively, introducing the substitution $x = \frac{S-ml}{\log(d)}$, we directly obtain the equivalent formulation:
	\[
	\mathcal{I} = \log(d)^{m} \int_{0}^{m} dx \> e^{x\log(d) + ml} f(e^{x\log(d)+ml}) f_{IH}(x; m) \, .
	\]
\end{proof}Equipped with Lemma~\ref{lemma:IrwinHall}, we can now systematically analyze the specific branches for $m \geq \kappa$ by turning $m$-dimensional integrals into simpler one-dimensional integrals. 
\paragraph*{Case b: $\bm{m = \kappa}$} 
When $m = \kappa$, we have from Eqs.~\eqref{eq:y_limits} that $y_{\min} \to 0$ and $y_{\max} = (C_\kappa \xi)^{1/(\kappa\alpha)} = \Theta_d(1)$. From Eq.~\eqref{eq:PsimIntegral}, the resonant term reads:
\begin{equation*}
	\psi^{(\kappa)} = \frac{1}{\kappa!} (C_\kappa \xi)^{-1/\alpha} \int_{y_{\min}}^{y_{\max}}\cdots \int_{y_{\min}}^{y_{\max}} \frac{\mathrm{d}y_1\cdots \mathrm{d}y_m}{1+(y_1\cdots y_m)^\alpha} = \frac{1}{\kappa!} (C_\kappa \xi)^{-1/\alpha} \mathcal{I} \, .
\end{equation*}
Applying Lemma~\ref{lemma:IrwinHall}, we can exactly recast the integral $\mathcal{I}$ as:
\begin{equation}\label{eq:gancio20}
	\mathcal{I} = \log(d)^\kappa \int_0^\kappa dx \, \frac{e^{x\log(d) + \kappa \log(y_{\min})}}{1+e^{\alpha x\log(d) + \alpha \kappa \log(y_{\min})}} f_{IH}(x; \kappa) \, .
\end{equation}
Exploiting the relation $y_{\max} = y_{\min} d$, which implies $\log(y_{\min}) = \log(y_{\max}) - \log(d)$, we obtain:
\begin{equation*}
	\begin{aligned}
		\mathcal{I} &= \log(d)^\kappa y_{\max}^{\kappa} \int_0^\kappa \mathrm{d}x \, \frac{e^{\log(d)(x-\kappa)}}{1+ y_{\max}^{\alpha\kappa}e^{\log(d) \alpha (x-\kappa)}} f_{IH}(x; \kappa) \\
		&= \log(d)^\kappa (C_\kappa \xi)^{1/\alpha} \int_0^\kappa \mathrm{d}x \, \frac{e^{\log(d)(x-\kappa)}}{1+ C_\kappa \xi e^{\log(d) \alpha (x-\kappa)}} f_{IH}(x; \kappa) \, .
	\end{aligned}
\end{equation*}
Finally, introducing the change of variables $t = \log(d) (\kappa - x)$ yields:
\begin{equation}\label{eq:gancio21}
	\mathcal{I} = \log(d)^{\kappa-1} (C_\kappa \xi)^{1/\alpha} \int_0^{\kappa \log(d)} \mathrm{d}t \, \frac{e^{-t}}{1+ C_\kappa \xi e^{-\alpha t}} f_{IH}\left(\kappa - \frac{t}{\log(d)}; \kappa\right) \, .
\end{equation}
We now compute the asymptotic limit of this integral as $d \to \infty$ using the Dominated Convergence Theorem (DCT). To this end, we recast $\mathcal{I}$ as:
\begin{equation*}
	\mathcal{I} = (C_\kappa \xi)^{1/\alpha} \int_{\mathbb{R}^+} g_d(t) \,\mathrm{d}t \, ,
\end{equation*}
where:
\begin{equation*}
	g_d(t) = \mathbf{1}_{[0,\kappa\log(d)]}(t) \, \log(d)^{\kappa-1} \frac{e^{-t}}{1+ C_\kappa \xi e^{-\alpha t}} f_{IH}\left(\kappa - \frac{t}{\log(d)}; \kappa\right) \, ,
\end{equation*}
and $\mathbf{1}_{[0,\kappa\log(d)]}(t)$ denotes the indicator function of the domain $[0,\kappa\log(d)]$. 

To evaluate the pointwise limit of $g_d(t)$, we exploit the asymptotic expansion of the Irwin-Hall distribution near its rightmost boundary:
\begin{equation*}
	f_{IH}(\kappa - \epsilon; \kappa) = \frac{\epsilon^{\kappa-1}}{(\kappa-1)!} \quad \text{as } \epsilon \to 0 \, .
\end{equation*}
This can be easily seen from its definition and from the fact that $f_{IH}(\epsilon;\kappa) = f_{IH}(\kappa-\epsilon;\kappa)$. Setting $\epsilon = \frac{t}{\log(d)}$, the diverging logarithmic factor cancels out as $d \to \infty$, yielding:
\begin{equation*}
	\lim_{d\to\infty} g_d(t) = \lim_{d\to\infty} \log(d)^{\kappa-1} \frac{e^{-t}}{1+ C_\kappa \xi e^{-\alpha t}} \left( \frac{t}{\log(d)} \right)^{\kappa-1} \frac{1}{(\kappa-1)!} = \frac{1}{(\kappa-1)!} \frac{e^{-t} t^{\kappa-1}}{1+ C_\kappa \xi e^{-\alpha t}} \, .
\end{equation*}
To apply the DCT, we must find a dominating function $G(t) \ge |g_d(t)|$ independent of $d$. Leveraging the global inequality for the Irwin-Hall distribution $f_{IH}(\kappa-x; \kappa) \le x^{\kappa-1}$ (valid for all $x \ge 0$), we bound $g_d(t)$ as:
\begin{equation*}
	|g_d(t)| \le \log(d)^{\kappa-1} \frac{e^{-t}}{1+ C_\kappa \xi e^{-\alpha t}} \left( \frac{t}{\log(d)} \right)^{\kappa-1} = \frac{e^{-t} t^{\kappa-1}}{1+ C_\kappa \xi e^{-\alpha t}} \equiv G(t) \, ,
\end{equation*}
which is strictly integrable over $\mathbb{R}^+$ since $\alpha < 1$. Since the hypotheses of the DCT are satisfied, we can exchange the limit and the integral:
\begin{equation*}
	\lim_{d\to\infty} \mathcal{I} = (C_\kappa \xi)^{1/\alpha} \int_{0}^{\infty} \left(\lim_{d\to\infty} g_d(t)\right) dt = (C_\kappa \xi)^{1/\alpha} \frac{1}{(\kappa-1)!} \int_{0}^{\infty} \frac{e^{-t} t^{\kappa-1}}{1+ C_\kappa \xi e^{-\alpha t}} \, dt \, .
\end{equation*}
Substituting this back into the expression for $\psi^{(\kappa)}$, we obtain:
\begin{equation*}
	\psi^{(\kappa)} = \frac{1}{\kappa!} (C_\kappa \xi)^{-1/\alpha} \mathcal{I} = \frac{1}{\kappa! (\kappa-1)!} \int_{0}^{\infty} \frac{e^{-t} t^{\kappa-1}}{1+ C_\kappa \xi e^{-\alpha t}} \, dt \, .
\end{equation*}
Finally, expanding the constant $C_\kappa = \delta_\kappa^{-1}(1-\alpha)^{-\kappa} = \frac{(1-\alpha)^{-\kappa}}{h_\kappa \kappa!}$, we arrive at the final form:
\begin{equation*}
	\psi^{(\kappa)} = \frac{1}{\kappa! (\kappa-1)!} \int_{0}^{\infty} dt \, e^{-t} t^{\kappa-1} \frac{1}{1 + \frac{(1-\alpha)^{-\kappa}}{h_\kappa \kappa!} \xi e^{-\alpha t}} \, .
\end{equation*}
This quantity is rigorously well-defined, and the integral securely converges to a finite macroscopic value. This aligns with the physical constraint dictated by the state equation, which requires $\psi^{(m)} \le \psi = \Theta_d(1)$.
\paragraph*{Case c: $\bm{m > \frac{\kappa}{1-\alpha}}$}
In this regime, both integration boundaries diverge to infinity ($y_{\min} \to \infty$ and $y_{\max} \to \infty$). However, since $y_{\max} = y_{\min} d$, the upper limit scales significantly faster than the lower one. As $d \to \infty$, the unity term in the denominator becomes negligible compared to the polynomial growth of $(y_1 \cdots y_m)^\alpha$. We can therefore asymptotically approximate the integral in Eq.~\eqref{eq:PsimIntegral} as follows:
\begin{equation}\label{eq:UV}
	\begin{aligned}
		\int_{y_{\min}}^{y_{\max}}\cdots\int_{y_{\min}}^{y_{\max}} \frac{\mathrm{d}y_1 \cdots \mathrm{d}y_m}{1+ (y_1 \cdots y_m)^\alpha} &\simeq \int_{y_{\min}}^{y_{\max}}\cdots\int_{y_{\min}}^{y_{\max}} (y_1 \cdots y_m)^{-\alpha} \, \mathrm{d}y_1 \cdots \mathrm{d}y_m \\
		&= (1-\alpha)^{-m} \big(y_{\max}^{1-\alpha} - y_{\min}^{1-\alpha}\big)^m \\
		&\simeq (1-\alpha)^{-m} y_{\max}^{m(1-\alpha)} \\
		&= (1-\alpha)^{-m} C_m^{\frac{1-\alpha}{\alpha}} \xi^{\frac{1-\alpha}{\alpha}} d^{\frac{1-\alpha}{\alpha}(m-\kappa)} \, .
	\end{aligned}
\end{equation}
Substituting this approximation back into Eq.~\eqref{eq:PsimIntegral}, the fractional scalings in $d$ perfectly cancel out, yielding a dimension-free expression:
\[
\psi^{(m)} = \frac{1}{m!} C_m^{-1}\xi^{-1} (1-\alpha)^{-m} \, .
\]
Recalling the definition of the constant $C_m = \delta_m^{-1} (1-\alpha)^{-m} = (h_m m!)^{-1} (1-\alpha)^{-m}$, we finally isolate the self-regularizing contribution of these higher-degree modes:
\[
\psi^{(m)} = \frac{h_m}{\xi} \quad \text{ for } \quad m > \frac{\kappa}{1-\alpha} \, .
\]
\paragraph*{Case d: $\bm{\kappa < m \leq \frac{\kappa}{1-\alpha}}$}
In this intermediate band, the integration limits exhibit a mixed asymptotic behavior: the upper limit diverges ($y_{\max} \to \infty$), while the lower limit $y_{\min}$ either vanishes to $0$ or converges to a strictly positive constant (when $\frac{\kappa}{1-\alpha} \in \mathbb{N}$). Nevertheless, the integral is heavily ultraviolet-dominated (see Eq.~\eqref{eq:UV}, since $y_{\min}$ is raised to the power of $1-\alpha > 0$). The leading-order asymptotic contribution is overwhelmingly driven by the diverging upper boundary, so that the precise behavior near the origin is asymptotically negligible. Therefore, the evaluation rigorously follows the exact same asymptotic expansion as the super-critical Case c, yielding:
\[
\psi^{(m)} = \frac{h_m}{\xi} \, .
\]

\paragraph*{Synthesis: The Macroscopic Kernel State Equation}
We can finally assemble the self-consistency equation. By substituting the results from the various regimes analyzed above into Eq.~\eqref{eq:DecompPsiPsimReg}, we obtain
\begin{equation}
    \begin{aligned}
        \psi &= \frac{\lambda}{\xi} + 0 + \psi^{(\kappa)} + \sum_{m > \kappa} \frac{h_m}{\xi} \\
             &= \frac{1}{\xi}\Big(\lambda + \sum_{m > \kappa} h_m\Big) + \psi^{(\kappa)} \, .
    \end{aligned}
\end{equation}
We identify the term in parentheses as the \textit{effective ridge regularization} generated by the higher-order tails, $\lambda_{\text{eff}} = \lambda + \sum_{m > \kappa} h_m$. Substituting the explicit integral evaluation for $\psi^{(\kappa)}$, we obtain the kernel state equation that implicitly determines the macroscopic parameter $\xi$:
\begin{equation}\label{eq:AppendixWAKE}
    \psi = \frac{\lambda_{\text{eff}}}{\xi} + \frac{1}{\kappa! (\kappa-1)!} \int_{0}^{\infty} dt \> e^{-t} t^{\kappa-1} \frac{1}{1 + \frac{(1-\alpha)^{-\kappa}}{h_\kappa \kappa!} \xi e^{-\alpha t}} \, .
\end{equation} \qed
\subsection{Variance}\label{appendix:Variance}
The variance $\mathsf{V}$ is given by Eq.~\eqref{eq:BiasVariance} and can be rewritten as $\mathsf{V} = \sigma_\varepsilon^2 \frac{\tau}{1-\tau}$ where:
\[
\tau = \frac{1}{n}\Tr(\Lambda^2(\Lambda+\nu_\star I)^{-2})
\]
with $\nu_\star = \xi_\star d^{-\kappa}$ and $\xi_\star$ is the solution of Eq.~\eqref{eq:AppendixWAKE}.
Under the polynomial scaling $n = \psi d^{\kappa}$ and using the findings in Appendix \ref{appendix:Reparametrization}, we can write $\tau$ as a sum over the shells:
\begin{equation*}
    \tau = \sum_{m=0}^\infty \tau^{(m)}
\end{equation*}
where, using the same notation as in Appendix~\ref{appendix:KernelStateEquation}:
\begin{equation}\label{eq:tauMDefinition}
    \tau^{(m)} = \frac{d^{-\kappa}}{\psi m!} \sum_{i_1, \dots, i_m = 1}^d \frac{1}{(1+C_m d^{m(1-\alpha)}(i_1 \dots i_m)^\alpha \nu)^2}
\end{equation}
\paragraph*{Case a: $\bm{m < \kappa}$}\label{appendix:sssm<kappa} 
For the lower-degree shells, we can employ the exact same bounding argument used in Appendix~\ref{appendix:m<kappa} to conclude that $\tau^{(m)} = \Theta(d^{m-\kappa})$. Consequently, since $m < \kappa$, these terms strictly vanish in the thermodynamic limit: $\tau^{(m)} \to 0$.

\paragraph*{The integral regime} 
As established previously, for $m \geq \kappa$ we can approximate the discrete summation in Eq.~\eqref{eq:tauMDefinition} with a Riemann integral. Introducing the coordinate transformation $y_i = x_i \, C_m^{\frac{1}{m\alpha}}d^{\frac{1-\alpha}{\alpha}} \nu^{\frac{1}{m\alpha}}$ and substituting the ansatz $\nu = \xi d^{-\kappa}$, we obtain:
\begin{equation}\label{eq:taumIntegral}
	\begin{aligned}
		\tau^{(m)} &\simeq \frac{d^{-\kappa}}{\psi m!} \int_1^d \cdots \int_1^d \mathrm{d}x_1 \cdots \mathrm{d}x_m \frac{1}{\big(1+C_m d^{m(1-\alpha)}(x_1 \cdots x_m)^\alpha \nu_\star\big)^2} \\ 
		&=\frac{1}{\psi m!}C_m^{-1/\alpha}\xi^{-1/\alpha} d^{(m-\kappa)(1-1/\alpha)} \int_{y_{\min}}^{y_{\max}} \cdots \int_{y_{\min}}^{y_{\max}} \frac{\mathrm{d}y_1 \cdots \mathrm{d}y_m}{\big(1+(y_1 \cdots y_m)^\alpha\big)^2} \, .
	\end{aligned}
\end{equation}
The integration boundaries $y_{\min}$ and $y_{\max}$ are identical to those defined in Eqs.~\eqref{eq:y_limits}. According to the asymptotic behavior sketched in Fig.~\ref{fig:MinMaxWA}, we will again distinguish the following regimes:
\paragraph*{Case b: $\bm{m > \frac{\kappa}{1-\alpha}}$} 
In this regime, both boundaries diverge ($y_{\min}, y_{\max} \to \infty$). Thus, the unity term in the denominator becomes negligible, allowing us to asymptotically approximate the integral:
\begin{equation}\label{eq:IntegralBif}
	\begin{aligned}
		\int_{y_{\min}}^{y_{\max}}\cdots\int_{y_{\min}}^{y_{\max}} \frac{\mathrm{d}y_1 \cdots \mathrm{d}y_m}{\big(1+ (y_1 \cdots y_m)^\alpha\big)^2} &\simeq \int_{y_{\min}}^{y_{\max}}\cdots\int_{y_{\min}}^{y_{\max}} (y_1 \cdots y_m)^{-2\alpha} \, \mathrm{d}y_1 \cdots \mathrm{d}y_m \\
		&= (1-2\alpha)^{-m} \big(y_{\max}^{1-2\alpha}-y_{\min}^{1-2\alpha}\big)^m \, .
	\end{aligned}
\end{equation}
The dominant contribution to this expression depends crucially on the sign of $1-2\alpha$:
\begin{itemize}
	\item When $\alpha < \frac{1}{2}$, the integral is asymptotically dominated by the upper limit $y_{\max}$. Substituting this back into Eq.~\eqref{eq:taumIntegral} yields:
	\[
	\tau^{(m)} \asymp d^{(m-\kappa)(1-1/\alpha)}y_{\max}^{m(1-2\alpha)} \asymp d^{(m-\kappa)(1-1/\alpha)} d^{\frac{1-2\alpha}{\alpha}(m-\kappa)} = d^{\kappa-m} \to 0 \, .
	\]
	Even when evaluating the infinite sum $\sum d^{\kappa-m}$, the series is bounded by its largest geometric term, behaving as $\Theta(d^{\kappa-\lceil\frac{\kappa}{1-\alpha}\rceil})$, which is asymptotically negligible.
	
	\item When $\alpha = \frac{1}{2}$, the integral in Eq.~\eqref{eq:IntegralBif} scales logarithmically as $\log(y_{\max}/y_{\min})^m \asymp \log(d)^m$. Overall, we obtain $\tau^{(m)} \asymp d^{\kappa - m} \log(d)^m \to 0$, as the polynomial decay suppresses the poly-logarithmic factor.
	
	\item Finally, when $\alpha > \frac{1}{2}$, the dominating term is the lower boundary $y_{\min}$, causing the integral to scale as $y_{\min}^{m(1-2\alpha)} \asymp d^{\frac{1-2\alpha}{\alpha}(m-\kappa-m\alpha)}$. Putting it all together:
	\[
	\tau^{(m)} \asymp d^{\frac{1-2\alpha}{\alpha}(m-\kappa-m\alpha)} d^{(m-\kappa)(1-1/\alpha)} = d^{\kappa - 2m(1-\alpha)} \, .
	\]
	By definition of this regime, $m > \frac{\kappa}{1-\alpha}$, meaning the exponent $\kappa - 2m(1-\alpha)$ is strictly negative, leading to $\tau^{(m)}\to 0$.
\end{itemize}
We can then conclude that the entire super-critical tail is negligible: $\sum_{m > \frac{\kappa}{1-\alpha}} \tau^{(m)} \to 0$.

\paragraph*{Case c: $\bm{\kappa < m \leq \frac{\kappa}{1-\alpha}}$} 
In this intermediate band, $y_{\min} \to 0$ while $y_{\max} \to \infty$. 

For $\alpha < \frac{1}{2}$, as shown in the previous case, the integral is strictly ultraviolet-dominated by the diverging upper limit $y_{\max}$. Thus, its asymptotic behavior is independent of the origin, directly yielding $\tau^{(m)} \asymp d^{\kappa - m} \to 0$. 

Conversely, for $\alpha \geq \frac{1}{2}$, it becomes crucial to account for the integrand's behavior near the origin. We rely once again on the Irwin-Hall representation (Lemma~\ref{lemma:IrwinHall}). Rewriting the integral over the sum variable $S$, we obtain:
\begin{equation}\label{eq:TauIntermediate}
	\int_{y_{\min}}^{y_{\max}} \cdots \int_{y_{\min}}^{y_{\max}} \frac{\mathrm{d}y_1 \cdots \mathrm{d}y_m}{\big(1+(y_1 \cdots y_m)^\alpha\big)^2} = \log(d)^{m-1} \int_{ml}^{mL} dS \, \frac{e^S}{(1+e^{\alpha S})^2} f_{IH}\left(\frac{S-ml}{\log d}; m\right) \, .
\end{equation}
For $\alpha > \frac{1}{2}$, the function $g(S) = \frac{e^S}{(1+e^{\alpha S})^2}$ is integrable over the entire real line. Its exponential tails rapidly suppress the integral away from the origin. Following the same procedure leveraging the DCT as we did in the state equation, we obtain that the overall term scales as:
\begin{equation}\label{eq:tauUpResonnantAlpha>}
	\tau^{(m)} \asymp d^{(m-\kappa)(1-1/\alpha)}\log(d)^{m-1} = \widetilde{\Theta}\big(d^{(m-\kappa)(1-1/\alpha)}\big) \, .
\end{equation}
Since the exponent $(m-\kappa)(1-1/\alpha)$ is strictly negative, this power-law decay collapses the term to zero. 

When $\alpha = \frac{1}{2}$, the integrand exhibits a constant asymptote $g(S) \sim 1$ as $S \to \infty$. Consequently, integrating up to $mL$ yields an additional dimension-dependent factor $\mathcal{O}(\log d)$. Plugging this into Eq.~\eqref{eq:TauIntermediate}, the total term scales as $\tau^{(m)} \asymp d^{-(m-\kappa)} \log(d)^m$. Since $m > \kappa$, the polynomial suppression effortlessly overcomes the logarithmic growth, ensuring $\tau^{(m)} \to 0$.

Hence, throughout the entire intermediate regime, $\tau^{(m)}\to 0$. Depending on the value of $\alpha$, the decay rate is bounded by either $\mathcal{O}(d^{\kappa - m})$ or $\widetilde{\mathcal{O}}(d^{(\kappa-m)(1/\alpha - 1)})$.
\paragraph*{Case d: $\bm{m = \kappa}$} 
Here $y_{\min}\to 0$ and $y_{\max} = (C_\kappa \xi)^{1/\alpha} = \Theta_d(1)$. The integral reads:
\begin{equation*}
	\tau^{(\kappa)} = \frac{1}{\kappa! \psi} (C_\kappa\xi)^{-1/\alpha} \int_{y_{\min}}^{y_{\max}}\cdots\int_{y_{\min}}^{y_{\max}} \frac{\mathrm{d}y_1 \cdots \mathrm{d}y_m}{\big(1+(y_1\cdots y_m)^\alpha\big)^2} \, .
\end{equation*}
Applying Lemma~\ref{lemma:IrwinHall}, we transform this $m$-dimensional integral into a 1-dimensional one:
\begin{equation*}
	\int_{y_{\min}}^{y_{\max}}\cdots\int_{y_{\min}}^{y_{\max}} \frac{\mathrm{d}y_1 \cdots \mathrm{d}y_m}{\big(1+(y_1\cdots y_m)^\alpha\big)^2} = \log(d)^\kappa \int_0^\kappa dx \, \frac{e^{x\log(d) + \kappa \log(y_{\min})}}{\big(1+e^{\alpha x\log(d) + \alpha \kappa \log(y_{\min})}\big)^2} f_{IH}(x; \kappa) \, .
\end{equation*}
The asymptotic evaluation from this point precisely mirrors the derivations in Appendix~\ref{appendix:KernelStateEquation} for the $m = \kappa$ case (see Eq.~\eqref{eq:gancio20}). Executing the high-dimensional limit $d \to \infty$, the logarithmic prefactors exactly cancel out, yielding:
\begin{equation}
	\tau^{(\kappa)} = \frac{1}{\psi} \frac{1}{\kappa! (\kappa-1)!} \int_{0}^{\infty} dt \, e^{-t} t^{\kappa-1} \frac{1}{\big(1+ \frac{(1-\alpha)^{-\kappa}}{h_\kappa \kappa!} \xi e^{-\alpha t}\big)^2} \, .
\end{equation}

\paragraph*{Synthesis: Final formula for the variance}
Combining the results together, we have that the only contribution to Eq.~\eqref{eq:tauMDefinition} that survives in the high-dimensional limit $d\to\infty$ is the $m = \kappa$, so that:
\begin{equation}
    \tau = \frac{1}{\psi} \frac{1}{\kappa! (\kappa-1)!} \int_{0}^{\infty} dt \> e^{-t} t^{\kappa-1} \frac{1}{(1+ \frac{(1-\alpha)^{-\kappa}}{h_\kappa \kappa!} \xi e^{-t\alpha})^2}
\end{equation}
\qed

\subsubsection{Fractional sample complexity \texorpdfstring{$\kappa$}{kappa}}
If $\kappa \not\in \mathbb{N}$, the asymptotic formulas simplify significantly. In this case, The kernel state equation lacks the exact "resonating" shell $m = \kappa$. Consequently, the threshold integral term $\psi^{(\kappa)}$ vanishes, and the self-consistency equation straightforwardly reduces to:
\[
\psi = \frac{\lambda_{\text{eff}}}{\xi}
\]
The same reasoning rigorously applies to the variance. As demonstrated previously, the only term yielding a macroscopic non-zero contribution $\tau^{(m)} = \Theta_d(1)$ occurs when $m = \kappa$. If $\kappa \not\in \mathbb{N}$, then $\tau$ will necessarily go to $0$. In this regime, the overall asymptotic decay of the variance as $d \to \infty$ is bounded by the slowest decaying adjacent integer shells, which is $\tau^{(\lfloor \kappa \rfloor)}$ or $\tau^{(\lceil \kappa \rceil)}$. 

In particular, as analyzed for the sub-critical terms ($m < \kappa$):
\[
\tau^{(\lfloor\kappa\rfloor)} = \Theta(d^{\lfloor\kappa\rfloor-\kappa})
\]
For $\tau^{(\lceil \kappa \rceil)}$ we have to distinguish between $\alpha < \frac{1}{2}$ and $\alpha > \frac{1}{2}$. When $\alpha < \frac{1}{2}$, the scaling for $\kappa < m < \frac{\kappa}{1-\alpha}$ follows the same regime as $m > \frac{\kappa}{1-\alpha}$, yielding for $\tau^{(\lceil \kappa \rceil)} = \Theta(d^{\kappa - \lceil\kappa\rceil})$. When $\alpha > \frac{1}{2}$, the integral localizes and the variance term scales as $\Theta(d^{(\lceil\kappa\rceil-\kappa)(1-1/\alpha)})$ as we proved in Eq.~\ref{eq:tauUpResonnantAlpha>}.

Defining the variance decay exponent $\gamma_V > 0$ such that $\mathsf{V} = \Theta(d^{-\gamma_V})$, the effective exponent is given by the minimum absolute decay rate among the two bounding shells:
\begin{equation}
        \gamma_V = 
        \begin{cases}
            \min\big(\kappa - \lfloor \kappa \rfloor, \; \lceil \kappa \rceil - \kappa \big) & \text{if } 0 \leq \alpha \leq \frac{1}{2} \\[6pt]
            \min\big(\kappa - \lfloor \kappa \rfloor, \; \frac{1-\alpha}{\alpha}(\lceil \kappa \rceil - \kappa) \big) & \text{if } \frac{1}{2} < \alpha < 1
        \end{cases}
\end{equation}
\Cref{fig:placeholder} illustrates how the deterministic equivalent for the variance improves at increasing dimension.
\begin{figure}
    \centering
    \includegraphics[width=0.75\linewidth]{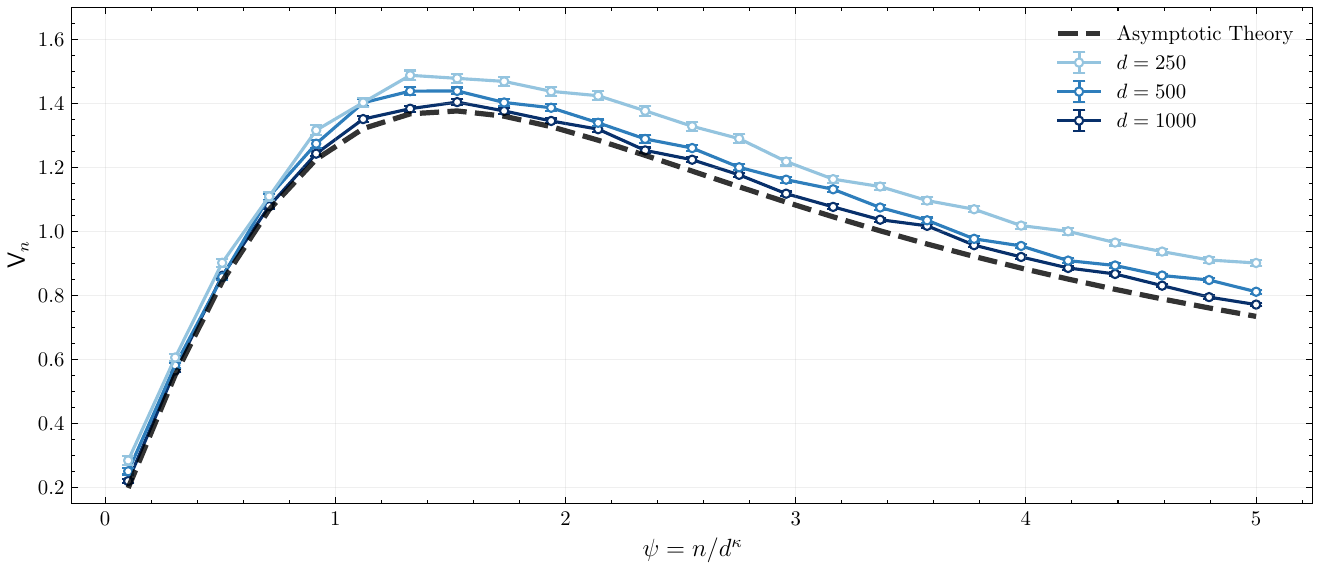}
\caption{Finite-size scaling of the variance learning curve. The empirical variance $\mathsf{V}_n$ is plotted as a function of the load parameter $\psi = n/d^\kappa$ for varying system dimensions $d$ (blue gradients). As the dimension increases, the empirical realizations progressively collapse onto the asymptotic theoretical prediction (black dashed line) derived from our state equations in Theorem~\ref{def:VarianceWeak}. Fixed parameters: $\kappa = 1$, $\alpha = 0.3$, and $\sigma_\varepsilon = 2$, $h(t) = 1 + 8t + 3t^2 + 0.1 t^3 + t^4$, $\lambda = 0$, $\psi = 0.8$ }
    \label{fig:placeholder}
\end{figure}

\subsection{Bias}
The deterministic equivalent for the bias $\mathsf{B}$ in Eq.~\eqref{eq:DetEquivalent} can be recast as:
\begin{equation}\label{eq:BiasIntro}
    \mathsf{B} = \frac{\nu^2\langle\theta, (\Lambda + \nu I)^{-2}\theta \rangle}{1-\frac{1}{n}\Tr(\Lambda^2(\Lambda + \nu I)^{-2})} = \frac{U}{1-\tau}
\end{equation}
where $U \equiv \nu^2\langle\theta, (\Lambda + \nu I)^{-2}\theta \rangle$ captures the explicit dependence on the target function and can be interpreted as the raw unlearned energy of the target function, and $\tau$ is the term previously evaluated in Appendix~\ref{appendix:Variance}  which doesn't depend on the specific target function. The vector $\theta$ contains the projection coefficients of the target function $f_\star$ onto the kernel eigenbasis. Following the reparameterization introduced in Appendix~\ref{appendix:Reparametrization}, we can expand this target contribution across the polynomial shells as
\[
f_\star(\bm x) = \sum_{m=0}^\infty \frac{1}{m!}\sum_{1 \leq i_1 \leq \dots \leq i_m \leq d} \theta_{i_1, \dots, i_m} e_{i_1, \dots, i_m}(\bm\Sigma^{-1/2}\bm x) 
\]
And the numerator of Eq.~\eqref{eq:BiasIntro} becomes:
\begin{equation}
    \begin{aligned}
            U \equiv \nu^2 \langle \theta, (\Sigma + \nu I)^{-2} \theta\rangle &= \sum_{m=0}\>\sum_{1 \leq i_1 \leq \dots \leq i_m \leq d} \frac{\nu^2}{(\lambda_{\beta} + \nu)^2}(\theta_{i_1, \dots, i_m})^2 \\ &= \sum_{m=0}\>\sum_{1 \leq i_1 \leq \dots \leq i_m \leq d} \frac{1}{(1 + \lambda_{ \beta}/\nu)^2}(\theta_{i_1, \dots, i_m})^2= \\
            &= \sum_{m=0}^{\infty} U^{(m)}
    \end{aligned}
\end{equation}
Where:
\begin{equation}
    U^{(m)} = \sum_{1 \leq i_1 \leq \dots \leq i_m \leq d} \frac{1}{(1 + \lambda_{i_1, \dots i_m}^{(m)}/\nu)^2}\left( \theta_{ i_1, \dots, i_m} \right)^2 
\end{equation}
In the limit $d\to\infty$, assuming a ``dense'' target function and extending $\theta_{ i_1, \dots, i_m}$ so that it is perfectly symmetric we have:
\begin{equation}\label{eq:bmSeries}
    U^{(m)}=\frac{1}{m!}\sum_{i_1  \dots i_m =1}^d \frac{\left( \theta_{ i_1, \dots, i_m} \right)^2}{(1 +  \nu^{-1}C_m^{-1} d^{m(\alpha - 1)}(i_1 \dots i_m)^{-\alpha})^2}
\end{equation}
To proceed, we need to assume something about the specific shape of the tensor $\theta_{ i_1, \dots, i_m}$. In particular, we will assume:
\[
\theta_{ i_1, \dots, i_m} = \theta_m (i_1 \dots i_m)^{-\omega}
\]
where $\omega > 0$ is the target anisotropy parameter governing an equivalent source condition for the function, and $\theta_m$ is a shell-dependent scalar coefficient.

The total spectral energy contained in the $m$-th shell, denoted as $E(m)$, is given by the sum of the squared weights:
\begin{equation}\label{eq:TargetEnergyDef}
    E(m) = \frac{1}{m!} \sum_{i_1, \dots, i_m = 1}^d \left( \theta_{i_1, \dots, i_m} \right)^2 = \frac{\theta_m^2}{m!} \sum_{i_1, \dots, i_m = 1}^d (i_1 \dots i_m)^{-2\omega} = \frac{\theta_m^2}{m!} \left( \sum_{i=1}^d i^{-2\omega} \right)^m 
\end{equation}
We want to ensure the target function is square-integrable (i.e., its total energy is finite) as $d \to \infty$. Hence, depending on $\omega$ we will have two different regimes:
\paragraph{Fast decay target ($\boldsymbol{\omega > 1/2}$).} 
In this regime, the target coefficients decay rapidly enough to guarantee the convergence of the infinite series in the thermodynamic limit $d \to \infty$. As established in Eq.~\eqref{eq:TargetEnergyDef}, the spectral energy evaluates to:
\begin{equation}
	E(m) = \frac{\theta_m^2}{m!} \zeta(2\omega)^m \, ,
\end{equation}
implying that the prefactor $\theta_m$ scales as an intensive quantity, $\Theta_d(1)$.

Substituting the critical scaling ansatz $\nu = \xi d^{-\kappa}$ and the source condition into Eq.~\eqref{eq:bmSeries}, we obtain:
\begin{equation}\label{eq:bm1}
	U^{(m)} = \frac{\theta_m^2}{m!} \sum_{i_1, \dots, i_m = 1}^d \frac{(i_1 \cdots i_m)^{-2\omega}}{\big(1 +  \xi^{-1} C_m^{-1} d^{\kappa+m(\alpha - 1)}(i_1 \cdots i_m)^{-\alpha}\big)^2} \, .
\end{equation}
The asymptotic limit of this expression is governed by the exponent of $d$, namely $\kappa + m(\alpha-1)$. Since $\omega > 1/2$, the bounding series $\sum_{i_1, \dots, i_m} (i_1 \cdots i_m)^{-2\omega}$ intrinsically converges. Thus, we can invoke the Dominated Convergence Theorem (DCT) to exchange the limit $d \to \infty$ with the summation sign, using the bound:
\begin{equation}
	\frac{(i_1 \cdots i_m)^{-2\omega}}{\big(1 + \xi^{-1} C_m d^{\kappa+m(\alpha - 1)}(i_1 \cdots i_m)^{-\alpha}\big)^2} \leq (i_1 \cdots i_m)^{-2\omega} \, .
\end{equation}
Assuming we are not exactly at a transition boundary (i.e., $m \neq \frac{\kappa}{1-\alpha}$ for any $m \in \mathbb{N}$), we can evaluate the limit term by term:
\begin{itemize}
	\item \textbf{Unlearned features} $\big(m > \frac{\kappa}{1-\alpha}\big)$\textbf{:} The exponent of $d$ is strictly negative, causing the denominator to approach $1$. The series thus converges to the full spectral energy:
	\begin{equation}
		U^{(m)} \xrightarrow{d \to \infty} \frac{\theta_m^2}{m!} \sum_{i_1, \dots, i_m = 1}^\infty (i_1 \cdots i_m)^{-2\omega} = \frac{\theta_m^2}{m!}\zeta(2\omega)^m = E(m) \, .
	\end{equation}
	\item \textbf{Learned features} $\big(m < \frac{\kappa}{1-\alpha}\big)$\textbf{:} The exponent of $d$ is strictly positive. Consequently, the denominator diverges, yielding $U^{(m)} \xrightarrow{d \to \infty} 0$.
\end{itemize}
By defining $\kappa_{\text{eff}} = \frac{\kappa}{1-\alpha}$, we conclude:
\begin{equation}
	\sum_{m=0}^{\infty} U^{(m)} \xrightarrow{d \to \infty} \sum_{m > \kappa_{\text{eff}}} E(m) \, .
\end{equation}

\paragraph{Slow decay target $\bm{\omega < \frac{1}{2}}$}
In this regime, the unnormalized energy per shell grows to infinity:
\[
E(m) \sim \frac{\theta_m^2}{m!}(1-2\omega)^{-m} d^{m(1-2\omega)}
\]
so that $\theta_m^2$ needs to scale with $d$ in order to keep the full sum $\sum_m E(m) = 1$. We can asymptotically approximate the series in Eq.~\eqref{eq:bmSeries} with the corresponding Riemann integral:
\begin{equation}
	U^{(m)} = \frac{\theta_m^2}{m!} \int_{1}^d \cdots \int_{1}^d \mathrm{d}x_1 \cdots \mathrm{d}x_m  \frac{(x_1 \cdots x_m)^{-2\omega}}{(1+\nu^{-1} C_m^{-1} d^{m(\alpha -1)} (x_1 \cdots x_m)^{-\alpha} )^2} \, .
\end{equation}
Upon the substitution: 
\[
y_i = x_i \nu^{\frac{1}{m\alpha}} C_m^{\frac{1}{m\alpha}} d^{\frac{1-\alpha}{\alpha}}
\]
and using the scaling relation $\nu = \xi d^{-\kappa}$ we get:
\begin{equation}\label{eq:bmIntegralbis}
	\begin{aligned}
			U^{(m)} &= \frac{\theta_m^2}{m!} C_m^{\frac{2\omega-1}{\alpha}}\xi^{\frac{2\omega-1}{\alpha}} d^{\frac{1}{\alpha}(\kappa-m+m\alpha)(1-2\omega)} \int_{y_{\min}}^{y_{\max}} \cdots \int_{y_{\min}}^{y_{\max}} \mathrm{d}y_1 \cdots \mathrm{d}y_m \frac{(y_1 \cdots y_m)^{-2\omega}}{(1+(y_1 \cdots y_m)^{-\alpha})^2} = \\
			&= \frac{\theta_m^2}{m!} C_m^{\frac{2\omega-1}{\alpha}}\xi^{\frac{2\omega-1}{\alpha}} d^{\frac{1}{\alpha}(\kappa-m+m\alpha)(1-2\omega)} \mathcal{I}^{(m)}(\alpha,\omega) \, .
	\end{aligned}
\end{equation}
The asymptotic behavior of the upper and lower limits of the integral is the same as the one in Fig.~\ref{fig:MinMaxWA}. Again, we will divide the analysis according to the different values of $m$:

\begin{description}
	\item[Case a: $\bm{m < \kappa}$] Here both $y_{\min}, y_{\max}\to 0$. The integral $\mathcal{I}^{(m)}(\alpha,\omega)$ can thus be computed by re-writing:
	\[
\mathcal{I}^{(m)}(\alpha,\omega)= \int_{y_{\min}}^{y_{\max}} \cdots \int_{y_{\min}}^{y_{\max}} \mathrm{d}y_1 \cdots \mathrm{d}y_m \frac{(y_1 \cdots y_m)^{2(\alpha -\omega)}}{(1+(y_1 \cdots y_m)^{\alpha})^2}
	\]
	and asymptotically approximating:
	\begin{equation*}
		\begin{aligned}
		\mathcal{I}^{(m)}(\alpha,\omega) &\simeq \int_{y_{\min}}^{y_{\max}} \cdots \int_{y_{\min}}^{y_{\max}} \mathrm{d}y_1 \cdots \mathrm{d}y_m (y_1 \cdots y_m)^{2(\alpha -\omega)} = \\
			&= \Bigl(\int_{y_{\min}}^{y_{\max}} y^{2(\alpha-\omega)} dy \Bigr)^m = \\ 
			&= (1+2(\alpha-\omega))^{-m} (y_{\max}^{1+2(\alpha-\omega)} - y_{\min}^{1+2(\alpha-\omega)})^m \, .
		\end{aligned}
	\end{equation*}
	Now, the leading term in this last expression is determined by the sign of $1+2(\alpha-\omega)$. Since we are considering $\omega < \frac{1}{2}$, then $1+2(\alpha-\omega) > 0$, and thus the leading term is the one with $y_{\max}$:
	\[
	(1+2(\alpha-\omega))^{-m} (y_{\max}^{1+2(\alpha-\omega)} - y_{\min}^{1+2(\alpha-\omega)})^m \sim (1+2(\alpha - \omega))^{-m} y_{\max}^{m+2m(\alpha-\omega)}
	\] 
	which scales as $d^{\frac{1+2(\alpha - \omega)}{\alpha}(m-\kappa)}$. The full term $U^{(m)}$ will then scale like:
	\[
	d^{\frac{1+2(\alpha - \omega)}{\alpha}(m-\kappa)}d^{\frac{\kappa-m+m\alpha}{\alpha}(1-2\omega)} = d^{3m - 2\kappa - 2m\omega} \, .
	\]
	At a first glance, the sign of this exponent might not be immediately obvious. However, if we rewrite it as:
	\[
	3m - 2\kappa - 2m\omega = 2(m-\kappa) + m(1-2\omega)
	\]
	and consider this expression together with $\theta_m^2$ (which, in the case $\omega < \frac{1}{2}$, must scale with $d$ as well):
	\[
	U^{(m)} \asymp \theta_m^2 d^{2(m-\kappa) + m(1-2\omega)} = \theta_m^2 d^{m(1-2\omega)} d^{2(m-\kappa)} = E(m) d^{2(m-\kappa)}
	\]
	we can see that $U^{(m)} \to 0$ when $m < \kappa$.
	
	\item[Case b: $\bm{m > \frac{\kappa}{1-\alpha}}$] 
	Now $y_{\min}, y_{\max}\to \infty$. The integral $\mathcal{I}^{(m)}(\alpha,\omega)$ in Eq.~\eqref{eq:bmIntegralbis} becomes:
	\begin{equation*}
		\begin{aligned}
\mathcal{I}^{(m)}(\alpha,\omega) &\sim \int_{y_{\min}}^{y_{\max}} \cdots \int_{y_{\min}}^{y_{\max}} \mathrm{d}y_1 \cdots \mathrm{d}y_m (y_1 \cdots y_m)^{-2\omega} = \\
			&= \Bigl(\int_{y_{\min}}^{y_{\max}} y^{-2\omega} dy \Bigr)^m = (1-2\omega)^{-m}(y_{\max}^{1-2\omega}-y_{\min}^{1-2\omega})^m \, .
		\end{aligned}
	\end{equation*}
	Since $\omega < \frac{1}{2}$, then:
	\[
	(1-2\omega)^{-m}(y_{\max}^{1-2\omega}-y_{\min}^{1-2\omega})^m \sim (1-2\omega)^{-m}y_{\max}^{m-2m\omega}
	\]
	And, asymptotically:
	\begin{equation*}
		\begin{aligned}
				&\int_{y_{\min}}^{y_{\max}} \cdots \int_{y_{\min}}^{y_{\max}} \mathrm{d}y_1 \cdots \mathrm{d}y_m \frac{(y_1 \cdots y_m)^{2(\alpha -\omega)}}{(1+(y_1 \cdots y_m)^{\alpha})^2}  \sim (1-2\omega)^{-m}y_{\max}^{m-2m\omega} = \\
				&= (1-2\omega)^{-m} \xi^{\frac{1-2\omega}{\alpha}}C_m^{\frac{1-2\omega}{\alpha}}d^{\frac{1-2\omega}{\alpha}(m-\kappa)} \, .
		\end{aligned}
	\end{equation*}
	Overall, the quantity $U^{(m)}$ becomes at a leading order:
	\begin{equation*}
		\begin{aligned}
			U^{(m)} &= \frac{\theta_m^2}{m!} C_m^{\frac{2\omega-1}{\alpha}}\xi^{\frac{2\omega-1}{\alpha}} d^{\frac{1}{\alpha}(\kappa-m+m\alpha)(1-2\omega)} (1-2\omega)^{-m} \xi^{\frac{1-2\omega}{\alpha}}C_m^{\frac{1-2\omega}{\alpha}}d^{\frac{1-2\omega}{\alpha}(m-\kappa)} = \\ 
			&= \frac{\theta_m^2}{m!}(1-2\omega)^{-m} d^{m(1-2\omega)} \, .
		\end{aligned}
	\end{equation*}
	And since we are in the $\omega < \frac{1}{2}$ regime, then the latter can be simply recognized as the spectral energy, yielding:
	\[
	U^{(m)} = E(m) \, .
	\]
	
	\item[Case c: $\bm{\kappa < m \leq \frac{\kappa}{1-\alpha}}$] In this intermediate regime, the integration boundaries behave as $y_{\min} \to 0$ and $y_{\max} \to \infty$. As established in the previous case $m > \frac{\kappa}{1-\alpha}$, the divergence of the integral is strictly driven by the large-$y$ behavior near the upper boundary $y_{\max}$ (UV-divergence). Consequently, the contribution from the lower limit ($y_{\min} \to 0$) is asymptotically sub-leading and can be safely neglected. Evaluating the integral at leading order thus yields the exact same dimensional scaling derived in the previous case. Therefore, even in this intermediate regime, the bias saturates to the full shell energy:
	\begin{equation}
		U^{(m)} = E(m) \, .
	\end{equation}
	
	\item[Case d: $\bm{m = \kappa}$]  In this case, $y_{\min}\to 0$ and $y_{\max} = C_\kappa^{\frac{1}{\kappa\alpha}} \xi^{\frac{1}{\kappa\alpha}} = \Theta_d(1)$. We are not integrating over the pathological region $y \to\infty$ so the final integral should be of order 1. Indeed, we can transform the integral $\mathcal{I}^{(m)}(\alpha,\omega)$ in Eq.~\eqref{eq:bmIntegralbis} using the Irwin-Hall technique in Lemma~\ref{lemma:IrwinHall} to get:
	\begin{equation}
		\begin{aligned}
			\mathcal{I}^{(m)}(\alpha,\omega) &= \log(d)^{\kappa} \int_0^\kappa dx f_{IH}(x;\kappa) \frac{e^{(1-2\omega)(x\log(d)+m\log(y_{\min}))}}{(1+e^{-\alpha(x\log(d)+m\log(y_{\min}))})^2} = \\
			&= \log(d)^{\kappa} (\xi C_\kappa)^{(1-2\omega)/\alpha}\int_0^\kappa dx f_{IH}(x;\kappa) \frac{e^{(1-2\omega)(x-\kappa)\log(d)}}{(1+(C_\kappa \xi)^{-1} e^{-\alpha \log(d) (x-\kappa)})^2} \, .
		\end{aligned}
	\end{equation}
	Under the substitution $(\kappa-x) \log(d) = t$, we obtain:
	\[
	\mathcal{I}^{(m)}(\alpha,\omega) = \log(d)^{\kappa-1} (\xi C_\kappa)^{(1-2\omega)/\alpha} \int_0^{\kappa\log(d)} f_{IH}\left(\kappa - \frac{t}{\log(d)}; \kappa\right) \frac{e^{-(1-2\omega)t}}{(1+(C_\kappa \xi)^{-1}  e^{\alpha t})^2} dt
	\]
	As $d\to\infty$, we can asymptotically approximate the Irwin-Hall density as we did in Sec.~\ref{appendix:KernelStateEquation} with Eq.~\eqref{eq:gancio21} (it works since $1-2\omega-2\alpha > 0$). We will thus omit the full derivation and provide the final results:
	\[
	\lim_{d\to\infty} \mathcal{I}^{(m)}(\alpha,\omega) = \frac{1}{(\kappa-1)!}(\xi C_\kappa)^{(1-2\omega)/\alpha} \int_0^{\infty} t^{\kappa - 1} \frac{e^{-(1-2\omega)t}}{(1+(C_\kappa \xi)^{-1}  e^{\alpha t})^2} dt \, .
	\]
	Getting back to the formula for $U^{(\kappa)}$ from Eq.~\eqref{eq:bmIntegralbis}, we finally have:
	\[
	U^{(\kappa)} = \frac{\theta_\kappa^2}{\kappa!}  d^{\kappa(1-2\omega)} \frac{1}{(\kappa-1)!} \int_0^{\infty} e^{-t} t^{\kappa - 1} \frac{e^{2\omega t}}{(1+(C_\kappa \xi)^{-1}  e^{\alpha t})^2} dt \, .
	\]
	Considering $E(\kappa) = \frac{\theta_\kappa^2}{\kappa!} (1-2\omega)^{-\kappa} d^{\kappa(1-2\omega)}$, then finally:
	\begin{equation}
		U^{(\kappa)} = E(\kappa) \frac{(1-2\omega)^\kappa}{(\kappa-1)!} \int_0^{\infty} e^{-t} t^{\kappa - 1} \frac{e^{2\omega t}}{(1+(C_\kappa \xi)^{-1}  e^{\alpha t})^2} dt \, .
	\end{equation}
\end{description}
\qed
\newpage

%% file: Appendix/Strong_anisotropy.tex
\section{Strongly anisotropic regime}
In this section we aim at proving the formulas valid in the $\alpha > 1$ regime. In the strongly anisotropic regime, we can express the eigenvalues in Eq.~\eqref{eq:KernelSpectrum} as:
\[
\lambda_\beta = c \, h_{|\beta|}|\beta|! \sigma_1^{\beta_1} \cdots \sigma_d^{\beta_d},
\]
where $c$ is a multiplicative constant independent of $d$. Consequently, all the formulas derived hereafter will hold up to an overall multiplicative constant $C$.
\subsection{Kernel State Equation and Variance}
Following the same procedure adopted in Appendix \ref{appendix:KernelStateEquation}, we can write the self-consistency equation Eq.~\eqref{eq:SelfConsistencyEquation} as $\psi = \sum_m \psi^{(m)} + \frac{\lambda}{\nu}$ where:
\begin{equation}
    \psi^{(m)} = \frac{1}{m!} \sum_{i_1, \dots, i_m = 1}^d \frac{d^{-\kappa}}{1+c^{-1}\nu(i_1 \dots i_m)^\alpha (h_m m!)^{-1} r_\alpha(d)^{m}}
\end{equation}  
When $\alpha > 1$, we know that $r_\alpha^{m} = \zeta(\alpha)^m$ at leading order and, labeling $C_m = c^{-1} (h_m m!)^{-1} \zeta(\alpha)^{m}$ we get:
\[
\psi^{(m)} = \frac{1}{m!} \sum_{i_1, \dots, i_m = 1}^d \frac{d^{-\kappa}}{1+\nu(i_1 \dots i_m)^\alpha C_m} \simeq  \frac{1}{m!} \int_1^d \dots \int_1^d \frac{d^{-\kappa}}{1+\nu(x_1 \dots x_m)^\alpha C_m} dx_1 \dots dx_m
\]
using the integral asymptotic approximation. Changing variables $x_i = y_i (C_m \nu)^{-\frac{1}{m\alpha}}$, we finally have:
\begin{equation}\label{eq:PsiMSAIntegral}
   \psi^{(m)} = \frac{1}{m!}C_m^{-\frac{1}{\alpha}} d^{-\kappa}  \nu^{-\frac{1}{\alpha}}\int_{y_{\min}}^{y_{\max}} \dots \int_{y_{\min}}^{y_{\max}} dy_1  \dots dy_m \frac{1}{1+(y_1 \dots y_m)^{\alpha}} 
\end{equation}
We analyze the ridgeless $\lambda = 0$ case and $\lambda \asymp 1$ separately. But before, we claim the following lemma that will come n handy later: 
\begin{lemma}[\textit{Asymptotic Expansion of the Shell Integral}]\label{thm:IntegralDivergence}
	Let $\alpha > 1$ and consider the following $m$-dimensional integral: 
	\begin{equation}
		\mathcal{I}^{(m)} = \int_{y_{\min}}^{y_{\max}}\cdots\int_{y_{\min}}^{y_{\max}} \frac{\mathrm{d}y_1 \cdots \mathrm{d}y_m}{1+(y_1 \cdots y_m)^\alpha} \, ,
	\end{equation}
	where the integration bounds are strictly positive and obey the following scaling behavior in the high-dimensional limit ($d\to\infty$):
	\begin{subequations}
		\begin{align}
			y_{\min} &= \Theta\Big(d^{-\frac{\kappa}{m}} \hat{f}(d)\Big) \to 0 \, , \\
			y_{\max} &= d \, y_{\min} = \Theta\Big(d^{1-\frac{\kappa}{m}} \hat{f}(d) \Big) \to\infty) \, .
		\end{align}
	\end{subequations}
	Here, $m \in \mathbb{N}$, $\kappa \in\mathbb{R}, 0 < \kappa \leq m$, and $\hat{f}(d)$ is a strictly sub-polynomial function (i.e. $f(d) = o(d^\epsilon)$ for any $\epsilon > 0$). 
	
	Then, as $d \to \infty$, the leading-order asymptotic behavior of the integral resolves into two distinct regimes:
	
	\begin{itemize}
		\item \textbf{Super-critical shells} ($m > \kappa$): 
		The upper bound $y_{\max}$ strictly diverges as a power of $d$. The integral exhibits a poly-logarithmic divergence governed by the Irwin-Hall distribution:
		\begin{equation}
			\mathcal{I}^{(m)} \simeq  \log(d)^{m-1} f_{IH}(\kappa; m) \int_{-\infty}^{\infty} dz \frac{e^{z}}{ 1+e^{\alpha z}}  
		\end{equation}
		where $f_{IH}(\kappa; m)$ is the probability density function of the Irwin-Hall distribution evaluated at the critical threshold $\kappa$, explicitly given by:
		\begin{equation}
			f_{IH}(\kappa; m) = \frac{1}{(m-1)!} \sum_{j=0}^{\lfloor \kappa \rfloor} (-1)^j \binom{m}{j} (\kappa - j)^{m-1}
		\end{equation}
		
		\item \textbf{Resonant shell} ($m = \kappa$): 
The integral diverges sub-logarithmically as:
			\begin{equation}
				\mathcal{I}^{(\kappa)} \simeq \log \hat{f}(d)^{\kappa-1}  \frac{\kappa^{\kappa-1}}{(\kappa-1)!} \int_{-\infty}^{\infty} dz \frac{e^{z}}{ 1+e^{\alpha z}}    \, .
			\end{equation}
			
	\end{itemize}
\end{lemma}
\begin{proof}
    
To manipulate the integral $\mathcal{I}_m$, we will resort again to the Irwin-Hall trick in Lemma~\ref{lemma:IrwinHall}. We will then obtain a one-dimensional integral:
\[
\mathcal{I}^{(m)}  = \log(d)^{m} \int_{0}^{m} dx  \frac{e^{x\log(d)+m\log(y_{\min})}}{1+e^{x\alpha\log(d)+m\alpha\log(y_{\min})}} f_{IH}(x;m) \, .
\]
Since $y_{\min} =  \Theta(d^{-\kappa/m}\hat{f}(d))$, then $e^{m \log(y_{\min})} = e^{\Delta - \kappa\log(d)}$ where $\Delta = \Theta(m\log\hat{f}(d)) = o(\log(d))$ is a quantity that scales sub-logarithmically with $d$. At the end, we have:
\[
\mathcal{I}^{(m)}  = \log(d)^m \int_0^m dx \frac{e^\Delta e^{(x-\kappa)\log(d)}}{ 1+ e^{\alpha\Delta} e^{\alpha\log(d)(x-\kappa)}} f_{IH}(x;m) 
\]Let's change variables $t = (x-\kappa)\log(d)+\Delta$:
\begin{equation}\label{eq:appuybis}
	\mathcal{I}^{(m)} = \log(d)^{m-1} \int_{-\kappa \log(d)+\Delta}^{(m-\kappa)\log(d)+\Delta} dt \frac{e^{t}}{1+ e^{\alpha t}}f_{IH}\left( \kappa + \frac{t-\Delta}{\log(d)}; m\right) 
\end{equation}
We now analyze the behavior of this expression as $d \to \infty$ by leveraging the Dominated Convergence Theorem (DCT). Specifically, let us rewrite the integral using the indicator function $\mathbf{1}_{[\cdot]}(t)$:
\begin{equation}
	\begin{aligned}
		\mathcal{I}^{(m)} &= \log(d)^{m-1} \int_\mathbb{R} dt \, \mathbf{1}_{[-\kappa \log(d)+\Delta, (m-\kappa)\log(d)+\Delta ]}(t) \frac{e^{t}}{1+ e^{\alpha t}}f_{IH}\left( \kappa + \frac{t-\Delta}{\log(d)}; m\right) \\
		&= \log(d)^{m-1} \int_{\mathbb{R}} g_d(t) \, dt \, ,
	\end{aligned}
\end{equation}
To apply the theorem, we first compute the pointwise convergence of $g_d(t)$ for any fixed $t \in \mathbb{R}$:
\begin{equation}
	\lim_{d\to\infty} g_d(t) = \frac{e^{t}}{1+ e^{\alpha t}} f_{IH}( \kappa ; m ) \, ,
\end{equation}
as the upper and lower limit of the indicator function go to $\pm \infty$. This is true since $\Delta$ is sub-logarithmic, hence the leading order behavior is given by $-\kappa\log(d)\to\infty$ and $(m-\kappa)\log(d) \to \infty$.
Next, we construct an integrable bounding function $g(t)$ independent of $d$. Since the Irwin-Hall distribution $f_{IH}(z;m)$ is a continuous probability density function defined on a compact support ($[0,m]$), it is surely globally bounded by a finite constant $M_m$. Consequently, we can bound $g_d(t)$ as:
\begin{equation}
	g_d(t) \le M_m \frac{e^{t}}{1+ e^{\alpha t}} \equiv g(t) \, .
\end{equation}
The bounding function $g(t)$ is strictly integrable over $\mathbb{R}$ for $\alpha > 1$. Hence, by the Dominated Convergence Theorem, we can safely take the limit $d \to \infty$ inside the integral:
\begin{equation}
	\mathcal{I}^{(m)} \simeq f_{IH}(\kappa; m) \log(d)^{m-1} \int_{-\infty}^{\infty} dt \frac{e^{t}}{ 1+e^{\alpha t}} \, .
\end{equation}
The remaining integral is easy to compute and, overall, the integral simplifies to:
\begin{equation}\label{eq:f(d)}
	\mathcal{I}^{(m)} \sim f_{IH}(\kappa; m) \log(d)^{m-1} \frac{\pi}{\alpha \sin(\pi/\alpha)}
\end{equation}

Now let us consider the case $m = \kappa$. In this scenario, Eq.~\eqref{eq:appuybis} becomes:
\begin{equation}
	\mathcal{I}^{(\kappa)} = \log(d)^{\kappa-1} \int_{-\kappa \log(d)+\Delta}^{\Delta} dt \, \frac{e^{t}}{1+ e^{\alpha t}} f_{IH}\left( \kappa - \frac{\Delta-t}{\log(d)}; \kappa\right) \, .
\end{equation}
We can apply a reasoning similar to the previous case, but we must be more careful. If we were to apply the Dominated Convergence Theorem (DCT) directly, the integral alone would trivially collapse to $0$. While mathematically correct, this is not very instructive, as the diverging prefactor $\log(d)^{\kappa-1}$ would lead to an indeterminate form of the type $\infty \times 0$. To extract the correct scaling behavior of the integral as $d\to\infty$, let us recast the expression by factoring out $\Delta^{\kappa-1}$:
\begin{equation}
	\begin{aligned}
		\mathcal{I}^{(\kappa)} &= \Delta^{\kappa-1} \int_\mathbb{R} dt \, \mathbf{1}_{[-\kappa \log(d)+\Delta, \Delta]}(t) \left[ \log(d)^{\kappa-1} \Delta^{1-\kappa} \right] \frac{e^{t}}{1+ e^{\alpha t}} f_{IH}\left( \kappa - \frac{\Delta-t}{\log(d)}; \kappa\right) \\
		&\equiv \Delta^{\kappa-1} \int_{\mathbb{R}} g_d(t) \, dt \, .
	\end{aligned}
\end{equation}
Now we have to evaluate the pointwise limit of $g_d(t)$. Setting $\epsilon = \frac{\Delta-t}{\log(d)}$, when $\epsilon\to 0$ we can expand the Irwin-Hall function near its right edge:
\begin{equation}
	f_{IH}(\kappa - \epsilon; \kappa) = \frac{\epsilon^{\kappa-1}}{(\kappa-1)!} = \frac{(\Delta-t)^{\kappa-1}}{\log(d)^{\kappa-1} (\kappa-1)!} \, .
\end{equation}
Substituting this expansion into $g_d(t)$, the diverging logarithmic terms exactly cancel out:
\begin{equation}
	\begin{aligned}
		\lim_{d\to\infty} g_d(t) &= \lim_{d\to\infty} \Delta^{1-\kappa} \frac{e^{t}}{1+ e^{\alpha t}} (\Delta-t)^{\kappa-1} \frac{1}{(\kappa-1)!} \\
		&= \lim_{d\to\infty} \left( 1 - \frac{t}{\Delta} \right)^{\kappa-1} \frac{1}{(\kappa-1)!} \frac{e^{ t}}{1+e^{\alpha t}} = \frac{1}{(\kappa-1)!} \frac{e^{t}}{1+e^{\alpha t}} \, .
	\end{aligned}
\end{equation}
To rigorously apply the DCT, we can bound the sequence using the global upper limit for the Irwin-Hall distribution $f_{IH}(\kappa-x; \kappa) \leq x^{\kappa-1}$ (valid for all $x \geq 0$). Substituting this into the absolute value of $g_d(t)$ yields:
\begin{equation}
	|g_d(t)| \leq \left[\log(d)^{\kappa-1} \Delta^{1-\kappa}\right] \frac{e^{t}}{1+ e^{\alpha t}} \left( \frac{\Delta-t}{\log(d)} \right)^{\kappa-1} = \left( 1 - \frac{t}{\Delta} \right)^{\kappa-1} \frac{e^{t}}{1+ e^{\alpha t}} \, .
\end{equation}
Since $t \leq \Delta$ on the integration domain, the term $(1 - t/\Delta)$ is less than or equal to $1$ for $t \ge 0$, and bounded by $1 + |t|$ for $t < 0$. Therefore, we can uniformly bound $|g_d(t)|$ by a function $G(t)$ independent of $d$:
\begin{equation}
	|g_d(t)| \leq (1 + |t|)^{\kappa-1} \frac{e^{t}}{1+ e^{\alpha t}} \equiv G(t) \, .
\end{equation}
Because $\alpha > 1$, $G(t)$ is absolutely integrable over $\mathbb{R}$. By the Dominated Convergence Theorem, we finally obtain the leading-order behavior:
\begin{equation}
	\mathcal{I}^{(\kappa)} \simeq \Delta^{\kappa-1} \frac{1}{(\kappa-1)!} \int_{\mathbb{R}} \frac{e^{t}}{1+ e^{\alpha t}} dt = \Delta^{\kappa-1} \frac{1}{(\kappa-1)!} \frac{\pi}{\alpha \sin(\pi/\alpha)} \, .
\end{equation}
Recalling that $\Delta \simeq \kappa \log(\hat{f}(d))$ at leading order, we get:
\begin{equation}
	\mathcal{I}^{(\kappa)} \simeq \frac{\kappa^{\kappa-1}}{(\kappa-1)!} \big(\log \hat{f}(d)\big)^{\kappa-1} \frac{\pi}{\alpha \sin(\pi/\alpha)} \, .
\end{equation}
This demonstrates that, if $\hat{f}(d) \to \infty$, the integral in the case $m = \kappa$ diverges much slower ($\big(\log\hat{f}(d)\big)^{\kappa-1}$ instead of $\log(d)^{\kappa-1}$).
\end{proof}
\subsubsection{Ridgeless variance}
In this case, the self-consistency equation reduces to $\psi = \sum_m \psi^{(m)}$. We thus formulate the following ansatz for $\nu$:
\begin{equation}\label{eq:ScalingNuSA}
    \nu = \xi f(d) d^{-\kappa\alpha}
\end{equation}
where $\xi \asymp 1$ with respect to $d$ and $f(d)$ is a super-poly-logarithmic function in $d$, meaning that:
$$ \lim_{d\to\infty} \frac{f(d)}{d^{\eta}} = 0 \quad \forall \eta > 0 \quad \text{and} \quad \lim_{d\to\infty} \frac{f(d)}{\log(d)^\pi} = \infty \quad \forall \pi > 0 $$
The reasons for this specific scaling will become clear later on.  The upper and lower limit in Eq.~\eqref{eq:PsiMSAIntegral} become:
\begin{equation}
    y_{\min} = (\xi C_{m})^{1 / m\alpha} d^{-\kappa/m} f(d)^{1/m\alpha}, \quad\quad y_{\max} = (\xi C_{m})^{1 / m\alpha} d^{1-\kappa/m} f(d)^{1/m\alpha}
\end{equation}
\paragraph{$\bm{m < \kappa}$} We can use similar bounding arguments as the ones used in Appendix~\ref{appendix:m<kappa} to conclude that the terms $\psi^{(m)}$ are negligible in the full summation (more precisely, $\psi^{(m)} = \Theta(d^{m-\kappa})$).
\paragraph{$\bm{m \geq \kappa}$} Let us consider the integral in Eq.~\eqref{eq:PsiMSAIntegral}. We can use the Lemma~\ref{thm:IntegralDivergence} setting $\hat{f}(d) = f(d)$ (scaling sub-polynomialy as required). Overall, Eq.~\eqref{eq:PsiMSAIntegral} simplifies to:
\begin{equation}\label{eq:f(d)}
    \psi^{(m)} \sim \frac{1}{m!} C_m^{-1/\alpha} \xi^{-1/\alpha}f(d)^{-1/\alpha}  f_{IH}(\kappa; m) \log(d)^{m-1} \int_{-\infty}^{\infty} dz \frac{e^{z}}{ 1+e^{\alpha z}} 
\end{equation}
or, when $m = \kappa$:
\begin{equation}\label{eq:f(d)bis}
    \psi^{(m)} \sim \frac{1}{m!} C_m^{-1/\alpha} \xi^{-1/\alpha}f(d)^{-1/\alpha} \frac{\kappa^{\kappa-1}}{(\kappa-1)!} \log(f(d))^{m-1} \int_{-\infty}^{\infty} dz \frac{e^{z}}{ 1+e^{\alpha z}} 
\end{equation}
\begin{remark}
    Now it should be clearer why we introduced the scaling $f(d)$. Without it, the generic $\psi^{(m)}$ would diverge logarithmically, which would not make sense as $\psi^{(m)} < \psi \asymp 1$. To extract the specific shape of $f(d)$, one should solve $\sum_m \psi^{(m)} \asymp 1$ and $f(d)$ should balance the infinite summation of the logarithmic powers. However, this is not strictly necessary, as we prove now.
\end{remark}
Now consider the formula for the variance similar to the one obtained in Eq.~\eqref{eq:tauMDefinition}, where here:
\[
\tau^{(m)} = \frac{d^{-\kappa}}{\psi m!} \sum_{i_1, \dots, i_m}^d \frac{1}{(1+C_m (i_1 \dots i_m)^\alpha \nu)^2} \sim  \frac{d^{-\kappa}}{\psi m!} \int_1^d \dots \int_1^d dx_1 \dots dx_m \frac{1}{(1+C_m (x_1 \dots x_m)^\alpha \nu)^2}
\]
By manipulating this equation as detailed above, we eventually arrive at an integral analogous to the one in Lemma~\ref{thm:IntegralDivergence}, but featuring a squared denominator. It is straightforward to show that the same asymptotic characterization holds for this modified integral, requiring only the substitution:
\[
\int_{-\infty}^{\infty} \frac{e^{z}}{1+e^{\alpha z}} \, \mathrm{d}z \quad \to \quad \int_{-\infty}^{\infty}  \frac{e^{z}}{(1+e^{\alpha z})^2} \, \mathrm{d}z \, .
\]
Consequently, we obtain:
\begin{equation}\label{eq:tauMSA}
    \tau^{(m)} \sim \frac{1}{\psi m!} C_m^{-1/\alpha} \xi^{-1/\alpha}f(d)^{-1/\alpha}  f_{IH}(\kappa; m) \log(d)^{m-1} \int_{-\infty}^{\infty} dz \frac{e^{z}}{ (1+e^{\alpha z})^2} 
\end{equation}
But since:
\[
\int_{-\infty}^{\infty} dz \frac{e^{z}}{ (1+e^{\alpha z})^2} = \left(1-\frac{1}{\alpha}\right) \int_{-\infty}^{\infty} dz \frac{e^{z}}{1+e^{\alpha z}}
\]
Then, comparing Eq.~\eqref{eq:tauMSA} with Eq.~\eqref{eq:f(d)}:
\[
\tau^{(m)} \sim \left(1-\frac{1}{\alpha}\right) \frac{\psi^{(m)}}{\psi}
\]
When summing everything up, we obtain:
\[
\tau = \sum_{m=0}^\infty \tau^{(m)} = \left(1-\frac{1}{\alpha}\right) \frac{1}{\psi}\sum_{m=0}^\infty \psi^{(m)} = \left(1-\frac{1}{\alpha}\right) 
\]
\paragraph{Final formula for the variance}
Considering the variance $\mathsf{V}$ can be written as $\mathsf{V} = \sigma_\varepsilon^2 \frac{\tau}{1-\tau}$, we finally obtain:
\[
\mathsf{V} = \sigma_\varepsilon^2 \frac{\tau}{1-\tau} = \sigma_\varepsilon^2 (\alpha - 1) 
\]
\begin{remark}
    Note that for $\alpha \to 1^+$, the variance vanishes identically for $n$ large enough. This is in agreement with the limit $\alpha \to 1^-$ when using the weak anisotropy formulas in Appendix.~\ref{appendix:Variance}
\end{remark}
\subsubsection{Regularized case (\texorpdfstring{$\lambda \asymp 1$}{lambda = O(1)})}
In the presence of explicit ridge regularization, the self-consistency equation reads:
\[
\psi = \sum_m \psi^{(m)} + \frac{\lambda}{\nu}d^{-\kappa}
\]
Assuming $\lambda \asymp 1$ we must modify our scaling ansatz. We postulate:
\begin{equation}\label{eq:ScalingNuRidge}
    \nu = \xi d^{-\kappa}
\end{equation}
Under this scaling, the generic macroscopic term $\psi^{(m)}$ evaluated through the integral representation in Eq.~\eqref{eq:PsiMSAIntegral} will scale as:
\[
\psi^{(m)} \sim \frac{1}{m!}C_m^{-\frac{1}{\alpha}} d^{-\kappa} \nu^{-\frac{1}{\alpha}}\int_{y_{\min}}^{y_{\max}} \dots \int_{y_{\min}}^{y_{\max}} \frac{dy_1 \dots dy_m}{1+(y_1 \dots y_m)^{\alpha}} \asymp d^{-\kappa \left(1-\frac{1}{\alpha}\right)}\log(d)^{m-1}
\]
Since $\alpha > 1$, the exponent $(1-1/\alpha)$ is strictly positive. Consequently, $\psi^{(m)}$ gets power-law suppressed by $d^{-\kappa(1-1/\alpha)}$. Even when summing over the macroscopic shells $m$, the accumulated logarithmic factors $\log(d)^{m-1}$ yield at most a sub-polynomial growth $f(d)$. Hence, the entire kernel contribution vanishes in the thermodynamic limit ($\sum_{m} \psi^{(m)} = \tilde{\Theta}(d^{-\kappa (1-1/\alpha)}) \to 0$). The state equation thus trivially reduces to the regularization term:
\begin{equation}\label{eq:SolutionXi}
    \psi = \frac{\lambda}{\xi}
\end{equation}
This implies $\xi = \lambda/\psi$, which is an order one macroscopic constant, definitively validating our scaling ansatz.

Following the same analysis, the terms $\tau^{(m)}$ (Eq.~\eqref{eq:tauMSA}) undergo the exact same algebraic suppression:
\[
\tau^{(m)} \asymp d^{-\kappa} \nu^{-1/\alpha}\int_{y_{\min}}^{y_{\max}}\dots\int_{y_{\min}}^{y_{\max}} \frac{dy_1 \dots dy_m}{(1+ (y_1 \dots y_m)^\alpha)^2} \asymp d^{-\kappa\left(1-\frac{1}{\alpha}\right)} \log(d)^{m-1}
\]
Summing all target shell contributions, the total $\tau = \sum_m \tau^{(m)}$ asymptotically vanishes ($\tau = \tilde{\Theta}(d^{-\kappa\left(1-\frac{1}{\alpha}\right)}\to 0$). The variance therefore scales as:
\[
\mathsf{V} = \tilde\Theta\left(d^{-\kappa\left(1-\frac{1}{\alpha}\right)}\right)
\]
where the $\tilde\Theta$ notation absorbs the residual sub-polynomial factor $f(d)$.

\subsection{Bias}
We now turn our attention to the bias component. Recall that the deterministic equivalent $\mathsf{B}$ can be computed by:
\begin{equation*}
    \mathsf{B} = \frac{U}{1-\tau}
\end{equation*}
where the unlearned energy $U$ explicitly depends on the target function:
\begin{equation}\label{eq:appendixb}
        U = \nu^2 \langle \theta, (\Lambda + \nu I)^{-2}\theta\rangle
\end{equation}
In the previous section, we proved that the $\tau$ term evaluates to $\tau = 1 - \frac{1}{\alpha}$ in the ridgeless limit ($\lambda = 0$), whereas it vanishes ($\tau \to 0$) in the regularized regime ($\lambda = \Theta_d(1)$). We now proceed to compute the unlearned energy $U$ across both regularization regimes.
\subsubsection{Ridgeless limit}
The scaling of $\nu$ is given by $\nu = \xi d^{-\kappa\alpha}f(d)$. We can recast Eq.~\eqref{eq:appendixb} so that $U = \sum_m U^{(m)}$ and:
\begin{equation}
    U^{(m)} = \frac{1}{m!} \> \sum_{i_1, \dots, i_m = 1}^d \frac{\left( \theta_{ i_1, \dots, i_m} \right)^2}{(1 + (i_1 \dots i_m)^{-\alpha} \nu^{-1}C_m^{-1} )^2}\
\end{equation}
with $C_m = c^{-1}(h_m m!)^{-1} \> \zeta(\alpha)^{m}$. Using $\theta_{i_1, \dots, i_m} = \theta_m (i_1 \dots i_m)^{-\omega}$ we get:
\begin{equation}\label{eq:bmSA}
    U^{(m)} = \frac{\theta_m^2}{m!} \> \sum_{i_1, \dots, i_m = 1}^d \frac{(i_1 \dots i_m)^{-2\omega}}{(1 + (i_1 \dots i_m)^{-\alpha} \nu^{-1}C_m^{-1} )^2} = \frac{\theta_m^2}{m!} \> \sum_{i_1, \dots, i_m = 1}^d \frac{(i_1 \dots i_m)^{-2\omega}}{(1 + (i_1 \dots i_m)^{-\alpha} \nu^{-1}C_m^{-1} )^2} 
\end{equation}

The mathematical evaluation of Eq.~\eqref{eq:bmSA} naturally splits into two primary categories, with the first one requiring a further subdivision based on the magnitude of $\omega$ with respect to $\alpha$:
\begin{itemize}
	\item \textbf{The fast decay regime} ($\omega > \frac{1}{2}$). This range can be further partitioned into two distinct sub-categories:
	\begin{itemize}
		\item \textit{Ultra-fast decay} ($\omega > \alpha + \frac{1}{2}$);
		\item \textit{Moderately fast decay} ($\frac{1}{2} < \omega \leq \alpha + \frac{1}{2}$).
	\end{itemize}
	\item \textbf{The slow decay regime} ($\omega < \frac{1}{2}$).
\end{itemize}
\paragraph{Ultra-fast decay $\bm{\omega > \alpha + \frac{1}{2}}:$} 
Let us rewrite the unlearned energy $U^{(m)}$ as:
\[
U^{(m)} = \frac{\theta_m^2}{m!}\nu^2 \> \sum_{i_1, \dots, i_m = 1}^d \frac{(i_1 \dots i_m)^{-2\omega}}{(\nu + (i_1 \dots i_m)^{-\alpha}C_m^{-1} )^2} 
\]
We haven't yet made the scaling of $\nu$ explicit. As it turns out, it's not strictly required; we will only need to assume $\nu\to0$ as $d\to\infty$. When performing this high-dimensional limit, we can apply the Dominated Convergence Theorem, bounding each term using the absolutely summable sequence $\sum_{i_1, \dots, i_m} (i_1 \dots i_m)^{-2(\omega-\alpha)}$ (converging since $2(\omega-\alpha) > 1$). Taking the limit $\nu \to 0$, then at leading order:
\[
U^{(m)} \sim \frac{\theta_m^2}{m!}\nu^2 C_m^2 \sum_{i_1, \dots, i_m = 1}^\infty (i_1 \dots i_m)^{-2(\omega-\alpha)} \, .
\]
Now we have to consider the scaling behavior of each term appearing in the above equation. The infinite summation is converging to $\zeta(2(\omega-\alpha))$ and, since $\omega > \frac{1}{2}$, then $\theta_m$ is of order-one with respect to $d$. Hence the only asymptotic dependence is on $\nu$:
\[
U^{(m)} \asymp \nu^2 \, .
\]
We know the scaling of $\nu$ up to a sub-polynomial correction, so that we obtain
\[
U^{(m)} = \widetilde\Theta(d^{-2\kappa\alpha}) \, ,
\]
meaning that the kernel has learned every polynomial feature. The final bias will therefore scale similarly
\[
\mathsf{B} = \tilde\Theta(d^{-2\kappa\alpha})
\]
or, remembering $n \asymp d^\kappa$, then:
\[
\mathsf{B} = \tilde\Theta(n^{-2\alpha})
\] 

\paragraph{Moderately-fast decay $\bm{\frac{1}{2} < \omega \leq \alpha + \frac{1}{2}}:$}
We can asymptotically approximate the summation in Eq.~\eqref{eq:bmSA} as the corresponding Riemann integral:
\[
U^{(m)} \sim \frac{\theta_m^2}{m!}\int_{1}^d \dots \int_{1}^d  \frac{dx_1 \dots dx_m}{((x_1 \dots x_m)^{\omega} +  \nu^{-1}C_m^{-1} (x_1 \dots x_m)^{\omega-\alpha})^2} 
\]
and performing the change of variables:
\[
y_i = x_i \nu^{1/m\alpha} C_m^{1/m\alpha}
\]
so that now: 
\[
U^{(m)} =  \frac{\theta_m^2}{m! } C_m^{\frac{2\omega-1}{\alpha}} \nu^{\frac{2\omega-1}{\alpha}}\int_{y_{\min}}^{y_{\max}} \dots \int_{y_{\min}}^{y_{\max}} \frac{dy_1 \cdots dy_m}{((y_1 \dots y_m)^{\omega} + (y_1 \cdots y_m)^{\omega-\alpha})^2}   
\]
or
\begin{equation}\label{eq:bmIntegral2}
	\begin{aligned}
		U^{(m)} &= \frac{\theta_m^2}{m!} C_m^{\frac{2\omega-1}{\alpha}} \nu^{\frac{2\omega-1}{\alpha}}\int_{y_{\min}}^{y_{\max}} \dots \int_{y_{\min}}^{y_{\max}} \frac{(y_1 \dots y_m)^{2(\alpha-\omega)}}{((y_1 \dots y_m)^{\alpha} + 1)^2} dy_1 \cdots dy_m = \\
		&= \frac{\theta_m^2}{m!} C_m^{\frac{2\omega-1}{\alpha}} \nu^{\frac{2\omega-1}{\alpha}}\mathcal{I}^{(m)}(\alpha, \omega)
	\end{aligned}
\end{equation}
where the limits of the integral $\mathcal{I}^{(m)}(\alpha, \omega)$ are:
\begin{subequations}\label{eq:y_limits2}
	\begin{align}
		y_{\min} &= \nu^{\frac{1}{m\alpha}} C_m^{\frac{1}{m\alpha}} = C_m^{\frac{1}{m\alpha}} \xi^{\frac{1}{m\alpha}} f(d)^{\frac{1}{m\alpha}} d^{-\frac{\kappa}{m}} \label{eq:miny} \\
		y_{\max} &= d \nu^{\frac{1}{m\alpha}} C_m^{\frac{1}{m\alpha}} = C_m^{\frac{1}{m\alpha}} \xi^{\frac{1}{m\alpha}} f(d)^{\frac{1}{m\alpha}} d^{1-\frac{\kappa}{m}} \label{eq:maxy} \, .
	\end{align}
\end{subequations}
These limits are asymptotically represented in Fig.~\ref{fig:MinMaxSA}. The lower limit $y_{\min} = \widetilde{\Theta}(d^{-\frac{\kappa}{m}})$ always collapse to zero. The upper limit $y_{\max} = \widetilde{\Theta}(d^{1-\frac{\kappa}{m}})$ converges to $0$ when $m < \kappa$ and diverge to $\infty$ for $m > \kappa$. Note that when $m = \kappa$, the scaling $f(d)$ makes $y_{\max}$ diverge to $\infty$ instead of converging to an order-one constant.
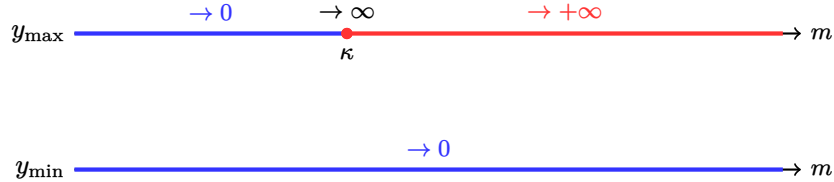
\begin{figure}[htbp]
	\centering
	\begin{tikzpicture}[
		scale=1.2,
		every node/.style={font=\small},
		dot/.style={circle, fill=black, inner sep=1.5pt}
		]
		\node[left, font=\bfseries] at (0, 1.5) {$y_{\max}$};
		\draw[thick, ->] (0, 1.5) -- (8, 1.5) node[right] {$m$};
		\draw[ultra thick, blue!80] (0, 1.5) -- (3, 1.5) node[midway, above=1pt] {$\to 0$};
		\draw[ultra thick, red!80] (3, 1.5) -- (7.8, 1.5) node[midway, above=1pt] {$\to +\infty$};
		\node[dot, label=below:{$\kappa$}, red!80] at (3, 1.5) {};
		\node[dot, label=above:{$\to\infty$}, red!80] at (3, 1.5) {};
		
		\node[left, font=\bfseries] at (0, 0) {$y_{\min}$};
		\draw[thick, ->] (0, 0) -- (8, 0) node[right] {$m$};
		\draw[ultra thick, blue!80] (0, 0) -- (7.8, 0) node[midway, above=1pt] {$\to 0$};
	\end{tikzpicture}
	\caption{\textit{Asymptotic behavior of the integration boundaries $y_{\min}$ and $y_{\max}$ derived in Eqs.~\eqref{eq:y_limits2}. At variance with Fig.~\ref{fig:MinMaxWA}, when $m = \kappa$ the respective boundary does not converge to a constant $\Theta(1)$ due to the diverging factor $f(d)$.}}
	\label{fig:MinMaxSA}
\end{figure}
As usual, we will solve Eq.~\eqref{eq:bmIntegral2} according to the different values of $m$: 
\paragraph{Case a: $\bm{m < \kappa}$} Here both $y_{\min}, y_{\max} \to 0$. The integral in Eq.~\eqref{eq:bmIntegral2} can therefore be asymptotically computed by:
\begin{equation}\label{eq:appuy3}
	\begin{aligned}
		\mathcal{I}^{(m)}(\alpha, \omega) = \int_{y_{\min}}^{y_{\max}} \dots \int_{y_{\min}}^{y_{\max}} \frac{(y_1 \dots y_m)^{2(\alpha-\omega)}}{((y_1 \dots y_m)^{\alpha} + 1)^2} dy_1 \dots dy_m &\sim 
		\int_{y_{\min}}^{y_{\max}} \dots \int_{y_{\min}}^{y_{\max}}(y_1 \dots y_m)^{2(\alpha-\omega)} 
	\end{aligned}
\end{equation}
When $\omega = \alpha + \frac{1}{2}$, then the integral in Eq.~\eqref{eq:appuy3} becomes logarithmic:
\begin{equation*}
	\begin{aligned}
		\int_{y_{\min}}^{y_{\max}} \dots \int_{y_{\min}}^{y_{\max}}(y_1 \dots y_m)^{-1} \> dy_1 \dots dy_m = \log^m\left(\frac{y_{\max}}{y_{\min}}\right) = \log(d)^m
	\end{aligned}
\end{equation*}
and, finally, $U^{(m)} \asymp \frac{\theta_m^2}{m!} \nu^2 \log(d)^m \asymp \frac{\theta_m^2}{m!} d^{-2\kappa\alpha} f(d)^{2} \log(d)^m = \tilde{\Theta}(d^{-2\kappa\alpha})\to 0$.

If $\frac{1}{2} < \omega < \alpha + \frac{1}{2}$, the integral $\mathcal{I}^{(m)}(\alpha, \omega)$ becomes:
\[
\int_{y_{\min}}^{y_{\max}} \dots \int_{y_{\min}}^{y_{\max}}(y_1 \dots y_m)^{2(\alpha-\omega)} dy_1 \dots dy_m \asymp (y_{\max}^{1+2(\alpha-\omega)} - y_{\min}^{1+2(\alpha-\omega)})^m
\]
since $1+2(\alpha-\omega) > 0$, the leading contribution is the one associated to $y_{\max}$, hence:
\[
U^{(m)} \asymp \frac{\theta_m^2}{m!} y_{\max}^{m(1+2(\alpha-\omega))} \nu^{\frac{2\omega-1}{\alpha}} =  \frac{\theta_m^2}{m!} d^{m+2m(\alpha-\omega)} \nu^{\frac{1+2(\alpha - \omega)}{\alpha}} \nu^{\frac{2\omega-1}{\alpha}} =\frac{\theta_m^2}{m!} d^{m+2m(\alpha-\omega)} \nu^{2}
\]
As $\frac{1}{2} < \omega < \frac{1}{2}+\alpha$, the target coefficient $\theta_m$ is of order 1 and:
\[
U^{(m)} \asymp \frac{\theta_m^2}{m!} d^{m+2m(\alpha-\omega)} \nu^{2} \asymp  d^{m+2m(\alpha-\omega)} d^{-2\kappa\alpha} f(d)^{2} = d^{m(1-2\omega)}d^{2\alpha(m-\kappa)} f(d)^{2} \to 0 \, ,
\]
since $m < \kappa$ and $1-2\omega < 0$. Therefore, for $m < \kappa$:
\begin{equation}
	U^{(m)} =
		\widetilde\Theta(d^{m(1-2\omega)+2\alpha(m-\kappa)}) \to 0
\end{equation}
and $\sum_{m<\kappa} U^{(m)}$ is negligible

\paragraph{Case b: $\bm{m > \kappa}:$} Using the technique in Lemma~\ref{lemma:IrwinHall}, we can rewrite the integral $\mathcal{I}^{(m)}(\alpha, \omega)$ in a one-dimensional form. Recalling that $m\log(y_{\min}) = -\kappa\log(d) + \Delta$ and $m\log(y_{\max}) = (m-\kappa)\log(d) + \Delta$, where $\Delta = \Theta(\alpha^{-1} \log(f(d)) = o(\log d)$, we obtain:
\begin{equation}\label{eq:appuy2}
	\mathcal{I}^{(m)}(\alpha, \omega) = \log(d)^{m-1} \int_{-\kappa\log(d) + \Delta}^{(m-\kappa)\log(d) + \Delta} dt \frac{e^{t(1+2\alpha-2\omega)}}{(1+e^{\alpha t})^2} f_{IH}\left(\kappa + \frac{t-\Delta}{\log(d)}; m\right) \, .
\end{equation}
We now analyze the behavior of this expression as $d \to \infty$. This is perfectly equivalent to the integral we analyzed in the proof of Lemma~\ref{thm:IntegralDivergence} since the integrand weight function $\eta(t) = \frac{e^{t(1+2\alpha-2\omega)}}{(1+e^{\alpha t})^2}$ is absolutely integrable on the real line for $\frac{1}{2} < \omega < \frac{1}{2}+\alpha$. By dominated convergence, we can therefore evaluate the Irwin-Hall density strictly at the critical threshold $\kappa$ and factor it out of the integral:
\begin{equation}
	\mathcal{I}^{(m)}(\alpha, \omega) \simeq \log(d)^{m-1} f_{IH}(\kappa;m) \int_{-\infty}^\infty dt \frac{e^{t(1+2\alpha-2\omega)}}{(1+e^{\alpha t})^2} \, .
\end{equation}
Finally, substituting this back into the expression for $U^{(m)}$, its value at leading order becomes:
\[
U^{(m)} \simeq \frac{\theta_m^2}{m!} C_m^{\frac{2\omega-1}{\alpha}} \nu^{\frac{2\omega-1}{\alpha}} \log(d)^{m-1} f_{IH}(\kappa;m) \int_{-\infty}^\infty dt \frac{e^{t(1+2\alpha-2\omega)}}{(1+e^{\alpha t})^2} \, .
\]
Since the remaining one-dimensional integral over $t$ converges to an order-one constant, and recalling the scalings $\nu = \tilde{\Theta}(d^{-\kappa\alpha})$ and $\theta_m = \Theta(1)$, we readily obtain:
\[
U^{(m)} = \tilde{\Theta}(d^{-\kappa(2\omega - 1)}) \, .
\]
Consequently, summing over the super-critical shells yields:
\[
\sum_{m>\kappa}^{\infty} U^{(m)} = \tilde{\Theta}(d^{-\kappa(2\omega - 1)}) \, .
\]
\paragraph{Case c: $\bm{m = \kappa}:$}
We use again Lemma~\ref{lemma:IrwinHall}. Recalling the boundary scalings, we have $\kappa\log(y_{\max}) = (m-\kappa)\log(d) + \Delta = \Delta$ and $\kappa\log(y_{\min}) = -\kappa\log(d) + \Delta$, where $\Delta = \Theta(\alpha^{-1} \log(f(d))) = o(\log(d)) \to \infty$. The integral becomes:
\begin{equation}
	\mathcal{I}^{(\kappa)}(\alpha, \omega) = \log(d)^{\kappa-1} \int_{-\kappa\log(d) + \Delta}^{\Delta} dt \frac{e^{t(1+2\alpha-2\omega)}}{(1+e^{\alpha t})^2} f_{IH}\left(\kappa + \frac{t-\Delta}{\log(d)}; \kappa\right) \, .
\end{equation}

This integral is formally equivalent to the one analyzed in the proof of Lemma~\ref{thm:IntegralDivergence} in the case $m = \kappa$, the only difference being in the absolutely integrable integrand weight function. Leveraging the DCT and expanding the Irwin-Hall function near its rightmost boundary, we obtain:
\begin{equation}
	\mathcal{I}^{(\kappa)}(\alpha, \omega) \simeq \frac{\Delta^{\kappa-1}}{(\kappa-1)!} \int_{-\infty}^{\infty} dt \frac{e^{t(1+2\alpha-2\omega)}}{(1+e^{\alpha t})^2} \asymp \Delta^{\kappa-1} \asymp \big(\log \hat{f}(d)\big)^{\kappa-1} \, .
\end{equation}
Finally, substituting this sub-logarithmic growth back into the macroscopic unlearned energy evaluated in Eq.~\eqref{eq:bmIntegral2}, and recalling that $\nu = \widetilde{\Theta}(d^{-\kappa\alpha})$, we obtain:
\begin{equation}
	U^{(\kappa)} \asymp \nu^{\frac{2\omega-1}{\alpha}} \mathcal{I}^{(\kappa)}(\alpha, \omega) \asymp d^{-\kappa(2\omega-1)} \big(\log \hat{f}(d)\big)^{\kappa-1} = \widetilde{\Theta}\Big(d^{-\kappa(2\omega-1)}\Big) \, .
\end{equation}
Hence the resonant shell $m = \kappa$ decays at the exact same polynomial rate as the super-critical shells $m > \kappa$, being immune to the $\log(d)^{\kappa-1}$ explosion.

Combining all the individual regimes together, we finally arrive at the complete asymptotic picture showing that the unlearned energy, and consequently the macroscopic bias, vanishes in the high-dimensional limit. In particular, the overall convergence rate is dictated by the contributions of the deep shells ($m > \kappa$), yielding:
\begin{equation}
	\mathsf{B} = \widetilde{\Theta}\left(d^{-\kappa(2\omega-1)}\right) = \widetilde{\Theta}\left(n^{-(2\omega-1)}\right) \, .
\end{equation}

\paragraph{Slow decay $\bm{\omega < \frac{1}{2}}:$}
Similarly to what we did in the moderately-fast regime, we will split our analysis into three distinct parts. The quantity we want to compute is always:
\[
U^{(m)} =  \frac{\theta_m^2}{m! } C_m^{\frac{2\omega-1}{\alpha}} \nu^{\frac{2\omega-1}{\alpha}}\int_{y_{\min}}^{y_{\max}} \dots \int_{y_{\min}}^{y_{\max}} \frac{dy_1 \cdots dy_m}{((y_1 \dots y_m)^{\omega} + (y_1 \cdots y_m)^{\omega-\alpha})^2}   =  \frac{\theta_m^2}{m! } C_m^{\frac{2\omega-1}{\alpha}} \nu^{\frac{2\omega-1}{\alpha}} \mathcal{I}^{(m)}(\alpha, \omega)
\]
\paragraph{Case a: $\bm{m < \kappa}$:} 
Following the same procedure as the one in the moderately-fast regime, we approximate the summation as a Riemann integral. Since both limits vanish ($y_{\min}, y_{\max} \to 0$), the integral scales as:
\begin{equation}
	\begin{aligned}
		\mathcal{I}^{(m)}(\alpha, \omega) &\simeq \int_{y_{\min}}^{y_{\max}}\dots\int_{y_{\min}}^{y_{\max}} (y_1 \cdots y_m)^{2(\alpha-\omega)} dy_1 \cdots dy_m \\
		&\asymp \big(y_{\max}^{1+2\alpha-2\omega} - y_{\min}^{1+2\alpha-2\omega}\big)^m \asymp y_{\max}^{m(1+2\alpha-2\omega)} \, .
	\end{aligned}
\end{equation}
Consequently, we obtain the scaling for the unlearned energy:
\begin{equation}
	U^{(m)} \asymp \frac{\theta_m^2}{m!} y_{\max}^{m(1+2\alpha-2\omega)} \nu^{\frac{2\omega-1}{\alpha}} \asymp \frac{\theta_m^2}{m!} d^{m+2m(\alpha-\omega)} \nu^{2} \, .
\end{equation}
The fundamental difference here is that in the slow decay regime ($\omega < 1/2$), the target coefficients $\theta_m$ must scale with $d$ to maintain a finite total energy for the target function $f_\star$. Specifically, we established that their squared norm scales as $\theta_m^2 = \Theta(d^{-m(1-2\omega)})$. Therefore, replacing this into our expression alongside $\nu^2 = \widetilde{\Theta}(d^{-2\kappa\alpha})$, the unlearned energy for the $m <\kappa$ shells elegantly reduces to:
\begin{equation}
	U^{(m)} \asymp d^{-m(1-2\omega)} d^{m+2m(\alpha-\omega)} \nu^{2} = \widetilde{\Theta}\big(d^{2\alpha(m-\kappa)}\big) \to 0 \, .
\end{equation}
Since $m < \kappa$, the exponent is strictly negative, meaning that the kernel perfectly resolves the sub-critical shells.

\paragraph{Case b: $\bm{m > \kappa}$}
We will change the variables in Eq.~\eqref{eq:appuy2} with $t = m\log(y_{\max}) - S = mL - S$:
\begin{equation}
	\begin{aligned}
		\mathcal{I}^{(m)}(\alpha, \omega) &= \log(d)^{m-1} \int_0^{m\log(d)} dt \, \frac{e^{(mL - t)(1+2\alpha-2\omega)}}{(1+e^{\alpha (mL - t)})^2} f_{IH}\left(m - \frac{t}{\log(d)};m\right) \\
		&= \log(d)^{m-1} e^{mL(1+2\alpha-2\omega)}\int_0^{m\log(d)} dt \, \frac{e^{- t(1+2\alpha-2\omega)}}{(1+e^{m\alpha L}e^{-\alpha t})^2} f_{IH}\left(m - \frac{t}{\log(d)};m\right) \\
		&= \log(d)^{m-1} y_{\max}^{m(1+2\alpha-2\omega)} \int_0^{m\log(d)} dt \, \frac{e^{- t(1+2\alpha-2\omega)}}{(1+y_{\max}^{m\alpha}e^{-\alpha t})^2} f_{IH}\left(m - \frac{t}{\log(d)};m\right) \\
		&= y_{\max}^{m(1-2\omega)} \int_0^\infty dt \, \boldsymbol{1}_{[0,m\log(d)]}(t) \log(d)^{m-1} \frac{e^{- t(1+2\alpha-2\omega)}}{(y_{\max}^{-m\alpha} + e^{-\alpha t})^2} f_{IH}\left(m - \frac{t}{\log(d)};m\right) \, .
	\end{aligned}
\end{equation}
Now we can use the Dominated Convergence Theorem (DCT) on the sequence of functions:
\[
g_d(t) = \boldsymbol{1}_{[0,m\log(d)]}(t) \log(d)^{m-1} \frac{e^{- t(1+2\alpha-2\omega)}}{(y_{\max}^{-m\alpha} + e^{-\alpha t})^2} f_{IH}\left(m - \frac{t}{\log(d)};m\right) \, .
\]
Similarly to what we have already shown (expanding the Irwin-Hall function near its boundary in the limit $d\to\infty$), this pointwise converges to:
\[
\lim_{d \to \infty} g_d(t) = \frac{1}{(m-1)!}e^{-t(1-2\omega)}t^{m-1} \, .
\]
To rigorously bound the sequence, we note that $(y_{\max}^{-m\alpha} + e^{-\alpha t})^2 > e^{-2\alpha t}$. Combined with the global bound $f_{IH}(m-x; m) \leq x^{m-1}$, we obtain:
\[
|g_d(t)| \leq t^{m-1} \frac{e^{-t(1+2\alpha-2\omega)}}{e^{-2\alpha t}} = t^{m-1} e^{-t(1-2\omega)} \equiv G(t) \, ,
\]
which is absolutely integrable on $\mathbb{R}^+$ since $\omega < \frac{1}{2}$. By the DCT, we get at leading order:
\begin{equation}\label{eq:gancio18}
	\begin{aligned}
		\mathcal{I}^{(m)}(\alpha, \omega) &\simeq \frac{1}{(m-1)!} y_{\max}^{m(1-2\omega)} \int_0^{\infty} dt \> t^{m-1} e^{-t(1-2\omega)} \\
		&= y_{\max}^{m(1-2\omega)} (1-2\omega)^{-m} \, .
	\end{aligned}
\end{equation}
Finally, substituting this back, we obtain:
\[
U^{(m)} \sim \frac{\theta_m^2}{m!} C_m^{\frac{2\omega-1}{\alpha}} \nu^{\frac{2\omega-1}{\alpha}}y_{\max}^{m(1-2\omega)} (1-2\omega)^{-m} = \frac{\theta_m^2}{m!} d^{m(1-2\omega)} (1-2\omega)^{-m} \, ,
\]
and since $\omega < \frac{1}{2}$, this last expression is precisely the leading term of the energy for the $m$-th shell, $E(m)$:
\[
U^{(m)} \to E(m) \, .
\]
Hence, we find that when $m > \kappa$, the unlearned energy saturates to the full spectral energy: $U^{(m)} \to E(m)$. 

\paragraph{Case c: $\bm{m = \kappa}$}
Due to the sub-polynomial divergence of $f(d)$ as $d\to\infty$ in Eqs.~\eqref{eq:y_limits2}, the asymptotic evaluation of the resonant shell $m=\kappa$ identically follows the deep-shell regime $(m>\kappa)$, yielding $U^{(\kappa)} = E(\kappa)$. This property is entirely driven by the presence of $f(d)$; in its absence, the $\kappa$-th shell would exhibit a different behavior. From a learning perspective, the condition $U^{(\kappa)}= E(\kappa)$ demonstrates that the resonant shell remains completely unlearned, retaining its full initial target energy.

Synthesizing the results across all polynomial shells, we obtain the definitive asymptotic limit for the total unlearned energy in the slow decay regime ($\omega < 1/2$).
\begin{equation}
	U = \sum_{m = \kappa}^{\infty} E(m) \, .
\end{equation}

\paragraph{Summary for unregularized regime}
Putting everything together, we have
\begin{itemize}
    \item When $\omega > \alpha + \frac{1}{2}$:
    \[
    U = \sum_{m< \kappa} U^{(m)} + U^{(\kappa)} + \sum_{m>\kappa}U^{(m)} = \tilde\Theta(d^{-2\kappa\alpha})
    \]
    \item  When $\frac{1}{2} < \omega \leq \alpha + \frac{1}{2}$:
    \[
    U = \sum_{m< \kappa} U^{(m)} + U^{(\kappa)} + \sum_{m>\kappa}U^{(m)} = \tilde\Theta(d^{-\kappa(2\omega-1)})
    \]
    \item When $\omega < \frac{1}{2}$:
    \[
    U = \sum_{m< \kappa} U^{(m)} + U^{(\kappa)} + \sum_{m>\kappa}U^{(m)} = \sum_{m\geq\kappa}E(m)
    \]
\end{itemize}
and the full bias is given by $\mathsf{B} = \frac{1}{1-\tau}U = \alpha U$
\subsubsection{Regularized regime}
We now turn to the computation of the unlearned energy of the $m$-th shell in the dense setting:
\begin{equation}\label{eq:gancio12}
	U^{(m)} = \frac{\theta_m^2}{m!} \sum_{i_1, \dots, i_m = 1}^d \frac{(i_1 \cdots i_m)^{-2\omega}}{\big(1 + C_m^{-1} \nu^{-1} (i_1 \cdots i_m)^{-\alpha}\big)^2} \, ,
\end{equation}
where we introduce the macroscopic regularization $\lambda = \Theta_d(1)$. As established in Eq.~\eqref{eq:ScalingNuRidge}, this assumption imposes a strict scaling on the resolvent parameter $\nu$:
\begin{equation}
	\nu = \xi d^{-\kappa} = \frac{\lambda}{\psi}d^{-\kappa} \, .
\end{equation}
Aside from this modified scaling, the functional structure of $U^{(m)}$ remains unchanged, allowing us to adapt the analytical computations performed for the ridgeless setting ($\lambda = 0$). To systematically evaluate the asymptotic behavior of these sums, we transition to the multi-dimensional integral representation of Eq.~\eqref{eq:bmIntegral2}:
\begin{equation}
	U^{(m)} \simeq \frac{\theta_m^2}{m!} C_m^{\frac{2\omega-1}{\alpha}} \nu^{\frac{2\omega-1}{\alpha}} \mathcal{I}^{(m)}(\alpha, \omega) = \frac{\theta_m^2}{m!} C_m^{\frac{2\omega-1}{\alpha}} \xi^{\frac{2\omega-1}{\alpha}} d^{-\kappa\left(\frac{2\omega-1}{\alpha}\right)} \mathcal{I}^{(m)}(\alpha, \omega) \, ,
\end{equation}
where the upper and lower integration boundaries, previously defined in Eqs.~\eqref{eq:y_limits2}, are now updated to reflect the macroscopic ridge scaling of $\nu$:
\begin{subequations}\label{eq:y_limitsRidge}
	\begin{align}
		y_{\min} &= (\nu C_m)^{\frac{1}{m\alpha}} = C_m^{\frac{1}{m\alpha}} \xi^{\frac{1}{m\alpha}} d^{-\frac{\kappa}{m\alpha}} \, , \label{eq:miny3} \\
		y_{\max} &= d y_{\min} = C_m^{\frac{1}{m\alpha}} \xi^{\frac{1}{m\alpha}} d^{1-\frac{\kappa}{m\alpha}} \, . \label{eq:maxy3}
	\end{align}
\end{subequations}
By introducing the effective capacity threshold $\kappa_{\text{eff}} = \frac{\kappa}{\alpha} < \kappa$, these boundaries can be elegantly rewritten as:
\begin{subequations}\label{eq:y_limitsRidge2}
	\begin{align}
		y_{\min} &=  \Theta\Big(d^{-\frac{\kappa_{\text{eff}}}{m}}\Big) \, , \label{eq:miny4} \\
		y_{\max} &= \Theta\Big(d^{1-\frac{\kappa_{\text{eff}}}{m}}\Big) \, . \label{eq:maxy4}
	\end{align}
\end{subequations}
When $m < \kappa_{\text{eff}}$, the exponent $1 - \kappa_{\text{eff}}/m$ becomes strictly negative. Consequently, both the upper and lower integration limits asymptotically converge to zero. We can thus safely leverage the analytical machinery developed for the ridgeless setting: the sub-critical regime $m < \kappa_{\text{eff}}$ under a macroscopic ridge penalty is formally equivalent to the regime $m < \kappa$ in the unregularized case. In both scenarios, the integral vanishes, confirming that the sub-critical shells ($m < \kappa_{\text{eff}}$) are perfectly learned by the kernel, yielding $U^{(m)} \to 0$.

A parallel reasoning applies to the super-critical shells ($m > \kappa_{\text{eff}}$), which directly map to the unregularized $m > \kappa$ scenario. While the ridgeless setting required the modulating sub-polynomial factor $f(d)$ to regularize the summations, here the strong polynomial decay inherently dominates the asymptotic behavior, rendering the two problems formally equivalent. By directly evaluating the integral, we thus obtain the definitive asymptotic behavior for the unlearned energy of the deep shells ($m > \kappa_{\text{eff}}$):
\begin{equation}
	U^{(m)} \asymp
	\begin{cases}
		E(m) & \text{for } \omega < \frac{1}{2} \, , \\[8pt]
		\log(d)^{m-1} \nu^{\frac{2\omega-1}{\alpha}} = \widetilde{\Theta}\Big(n^{-\frac{2\omega-1}{\alpha}}\Big) & \text{for } \frac{1}{2} < \omega \leq \alpha + \frac{1}{2} \, , \\[8pt]
		\nu^2 = \Theta(n^{-2}) & \text{for } \omega > \alpha + \frac{1}{2} \, .
	\end{cases}
\end{equation}
The only main departure from the ridgeless setting comes out when we analyze the resonant shell $m = \kappa_{eff}$ (equivalent to the ridgeless $m = \kappa$). Assuming there exists an integer such that $m = \kappa_{\text{eff}}$, the integration boundaries defined in Eq.~\eqref{eq:y_limitsRidge2} behave differently. In the unregularized case, the presence of the modulating factor $f(d)$ forced the upper limit $y_{\max}$ to slowly diverge to infinity for the resonant shell. Here, however, the strict absence of $f(d)$ implies that the term $1 - \kappa_{\text{eff}}/m$ is exactly zero and, as a direct consequence, the upper bound $y_{\max}$ does not diverge, but instead rigorously settles on a finite, macroscopic $\Theta_d(1)$ constant as $d \to \infty$.

For fast-decaying targets ($\omega > 1/2$), the mathematical treatment maps almost identically to the ridgeless scenario, with the crucial distinction that the multi-dimensional integral now evaluates to a pure, order-one constant, completely devoid of any sub-logarithmic modulations. The resonant shell thus seamlessly joins the power-law scaling of the super-critical modes, yielding an unlearned energy that vanishes as the super-critical shells $m > \kappa_{eff}$:
\begin{equation}
	U^{(\kappa_{\text{eff}})} = \Theta\left(d^{-\kappa \min\left(\frac{2\omega-1}{\alpha}, 2\right)}\right) \, .
\end{equation}

It remains to evaluate the critical shell $m = \kappa_{\text{eff}}$ in the slow-decay regime ($\omega < 1/2$), under the assumption that the data anisotropy $\alpha$ yields an integer effective threshold, $\kappa_{\text{eff}} \in \mathbb{N}$. Unlike the unregularized limit, the absence of the diverging sub-polynomial factor $f(d)$ stabilizes the upper integration boundary $y_{\max}$ to an $\Theta_d(1)$ constant. 

The analytical procedure closely mirrors the techniques developed in the previous sections. We aim to compute the core integral $\mathcal{I}^{(\kappa_{\text{eff}})}(\alpha, \omega)$. Applying Lemma~\ref{lemma:IrwinHall} and evaluating the integration limits, we obtain:
\begin{equation}
\mathcal{I}^{(\kappa_{\text{eff}})}(\alpha, \omega) = \log(d)^{\kappa_{\text{eff}}-1} \int_{-\kappa_{\text{eff}}\log(d) + \Delta}^{\Delta} dt \frac{e^{t(1+2\alpha-2\omega)}}{(1+e^{\alpha t})^2} f_{IH}\left(\kappa_{\text{eff}} + \frac{t-\Delta}{\log(d)}; \kappa_{\text{eff}}\right) \, ,
\end{equation}
where the constant integration bound is defined as $\Delta = \frac{1}{\alpha} \log(C_{\kappa_{\text{eff}}} \xi)$ or, equivalently, $e^{\Delta / \kappa_{\text{eff}}} = C_{\kappa_{\text{eff}}}^{\frac{1}{\kappa_{\text{eff}}\alpha}} \xi^{\frac{1}{\kappa_{\text{eff}}\alpha}}$. To rigorously evaluate the high-dimensional limit ($d \to \infty$), we apply the Dominated Convergence Theorem (DCT). We define the sequence of integrands:
\[
g_d(t) = \boldsymbol{1}_{[-\kappa_{\text{eff}}\log(d) + \Delta, \Delta]}(t) \log(d)^{\kappa_{\text{eff}}-1} \frac{e^{t(1+2\alpha-2\omega)}}{(1+e^{\alpha t})^2} f_{IH}\left(\kappa_{\text{eff}} + \frac{t-\Delta}{\log(d)}; \kappa_{\text{eff}}\right) \, .
\]
By asymptotically expanding the Irwin-Hall distribution near its rightmost boundary, the polynomial $\log(d)$ divergence is perfectly canceled, yielding the point-wise limit:
$$
\lim_{d\to\infty} g_d(t) = \frac{e^{t(1+2\alpha-2\omega)}}{(1+e^{\alpha t})^2} \frac{(\Delta - t)^{\kappa_{\text{eff}}-1}}{(\kappa_{\text{eff}}-1)!} \, ,
$$
which does not depend on $d$ (recall that $\Delta$ is now a constant). Leveraging the usual property of the Irwin-Hall distribution, this sequence is globally uniformly bounded by the dominating function:
$$
G(t) = \frac{e^{t(1+2\alpha-2\omega)}}{(1+e^{\alpha t})^2} (\Delta - t)^{\kappa_{\text{eff}}-1}\, ,
$$
which is strictly integrable over $(-\infty, \Delta]$. Thus, passing the limit inside the integral yields:
\[
\mathcal{I}^{(\kappa_{\text{eff}})}(\alpha, \omega) =  \frac{1}{(\kappa_{\text{eff}}-1)!} \int_{-\infty}^{\Delta} dt \frac{e^{t(1+2\alpha-2\omega)}}{(1+e^{\alpha t})^2} (\Delta - t)^{\kappa_{\text{eff}}-1} \, .
\]
To evaluate this, we introduce the change of variables $z = \Delta - t$, obtaining:
\[
\mathcal{I}^{(\kappa_{\text{eff}})}(\alpha, \omega) =  \frac{e^{\Delta(1+2\alpha-2\omega)}}{(\kappa_{\text{eff}}-1)!} \int_{0}^{\infty} dz \frac{e^{-z(1+2\alpha-2\omega)}}{(1+e^{\alpha\Delta} e^{-\alpha z})^2} z^{\kappa_{\text{eff}}-1} \, .
\]
Substituting $e^{\alpha\Delta} = C_{\kappa_{\text{eff}}} \xi$, the integral assumes its final explicit form:
\[
\mathcal{I}^{(\kappa_{\text{eff}})}(\alpha, \omega) = \frac{1}{(\kappa_{\text{eff}}-1)!} (C_{\kappa_{\text{eff}}} \xi)^{\frac{1+2\alpha-2\omega}{\alpha}}  \int_{0}^{\infty} dz \frac{e^{-z(1+2\alpha-2\omega)}}{(1+C_{\kappa_{\text{eff}}} \xi e^{-\alpha z})^2} z^{\kappa_{\text{eff}}-1} \, .
\]
We can now reinsert this result into the general expression for $U^{(\kappa_{\text{eff}})}$. Renaming the dummy integration variable to $t$, we obtain:
\[
U^{(\kappa_{\text{eff}})} = \frac{\theta_{\kappa_{\text{eff}}}^2}{\kappa_{\text{eff}}!} d^{-\kappa_{\text{eff}}(2\omega-1)} \frac{(C_{\kappa_{\text{eff}}}\xi)^{2}}{(\kappa_{\text{eff}}-1)!} \int_0^{\infty} dt \> \frac{e^{-t (1+2\alpha-2\omega)}}{(1+ C_{\kappa_{\text{eff}}} \xi e^{-\alpha t} )^2} t^{\kappa_{\text{eff}}-1} \, .
\]
We can further restructure the integral by factoring out $e^{-2\alpha t}$ from the denominator, which simplifies the exponential terms and explicitly recovers the unlearned energy definition $E(\kappa_{\text{eff}})$:
\[
U^{(\kappa_{\text{eff}})} = E(\kappa_{\text{eff}}) \frac{(1-2\omega)^{\kappa_{\text{eff}}}}{(\kappa_{\text{eff}}-1)!} \int_0^{\infty} dt \> e^{-t} t^{\kappa_{\text{eff}}-1} \frac{e^{2t\omega}}{(1+ C_{\kappa_{\text{eff}}}^{-1} \xi^{-1} e^{\alpha t} )^2} \, .
\]
Notice that the functional structure of this integral perfectly mirrors the one derived for the residual fraction $\eta$ in the weakly anisotropic regime (Eq.~\eqref{eq:eta}). By substituting the explicit value of $C_{\kappa_{\text{eff}}}^{-1}$, and recalling that $\xi^{-1} = \frac{\psi}{\lambda}$ under macroscopic regularization, we obtain the final exact form for the critical shell's unlearned energy:
\[
U^{(\kappa_{\text{eff}})} = E(\kappa_{\text{eff}}) \frac{(1-2\omega)^{\kappa_{\text{eff}}}}{(\kappa_{\text{eff}}-1)!} \int_0^{\infty} dt \> e^{-t} t^{\kappa_{\text{eff}}-1} \frac{e^{2t\omega}}{\left(1+ h_{\kappa_{\text{eff}}} \kappa_{\text{eff}}! \zeta(\alpha)^{-\kappa_{\text{eff}}} \frac{\psi}{\lambda} e^{\alpha t} \right)^2} = \eta' E(\kappa_{\text{eff}}) \, ,
\]
\qed
\newpage

%% file: Appendix/Single_Index_Models.tex
\section{Single-Index Models}
In this section we wish to prove the formulas in the main body regarding the single-index models case. We will consider a target function whose form is:
$$
f_\star(x) = g (v^\top x)
$$
where $v \in \mathbb{R}^d$ is the main defining property of the single-index function and we assume the link function $g \in L_2(\mathcal{N}(0,1))$, meaning that:
\[
\|g\|^2 = \int_\mathbb{R} \frac{1}{\sqrt{2\pi}}e^{-\frac{1}{2}t^2} g(t)^2 dt = 1 \> .
\]
Since $x \sim \mathcal{N}(0, \Sigma)$ where $\Sigma = \text{diag}(\sigma_1, \dots, \sigma_d)$ is the usual diagonal covariance matrix power-law distributed, then it will be convenient to redefine:
$$
z = \Sigma^{-1/2} x
$$
so that $ z \sim \mathcal{N}(0, I_d)$. To ensure that $f_\star \in L_2(p_X) = L_2(\mathcal{N}(0, \Sigma))$ and $\|f_\star\|_{L_2} = 1$ as usual, one can prove we need to assume:
\[
\mathbb{E}[(v^\top x)^2] = 1   \iff \sum_k v_{k}^2 \sigma_k = 1
\]
together with $\|g\|^2 = 1$. Hence, we can decompose $f_\star$ it in the Hermite multi-variate orthonormal basis $H_\beta$ obtaining the same expression of Eq.~\eqref{eq:appendixBeta}
\begin{equation}\label{eq:stand}
    f_\star(x) = \sum_{\beta\in\mathbb{Z}^d_{\ge 0}} C_{\bm\beta} \> H_{\beta}(z) = \sum_{m=0}^{\infty} \sum_{1 \leq i_1 \leq \dots \leq i_m \leq d} \theta_{i_1, \dots, i_m} H_{i_1, \dots, i_m}(z)
\end{equation}
Let us now consider for a moment the link function $g(t)$. Since we assumed $g \in L_2(\mathcal{N}(0,1))$, we can decompose $g$ in the 1-dimensional orthonormal Hermite polynomials basis obtaining:
$$g(t) = \sum_{m=0}^{\infty} g_m h_m(t)$$
where $h_m(t) = \frac{1}{\sqrt{m!}}\He_m(t)$. At this point, the target function becomes:
$$
f_\star(x) = \sum_{m=0}^{\infty} g_m h_m(v^\top x)
$$
or, using $ z = \Sigma^{-1/2} x$:
$$
f_\star(x) = \sum_{m=0}^{\infty} g_m h_m(w^\top z)
$$
having defined $w = \Sigma^{1/2}v$ as the transformation of the model weights due to the data anisotropy. To proceed, we will use a well-known addition identity of the Hermite polynomials  \citep[Sec.~18.18(ii)]{DLMF_Chap18}. Since by design $\|w\|^2 = 1$, we have: 
$$
h_m(w^\top z) = \sum_{|\beta| = m} \sqrt{m \choose \beta} w_{1}^{\beta_1} \dots w_{d}^{\beta_d} H_{\beta}(z)
$$
where, for convenience, we will define ${m \choose \beta} = {m \choose \beta_1 \dots \beta_d} = \frac{m!}{\beta_1! \dots \beta_d!}$. Remembering now that $w = \Sigma^{1/2}v$ and thus $w_{j} = \sigma_j^{1/2} v_{j}$, we will have:
$$
h_m(w^\top z) = \sum_{|\beta| = m} \sqrt{m \choose \beta} (\sigma_1^{\frac{1}{2}\beta_1} \dots \sigma_d^{\frac{1}{2}\beta_d}) (v_{1}^{\beta_1} \dots v_{d}^{\beta_d}) H_{\beta}(z)
$$
It is now convenient to move to the reparametrization $\beta \to (i_1, \dots, i_m), \beta_j = \sum_k \delta_{i_k, j}$, since we are only considering values of $\beta$ whose total modulus is fixed. We can therefore write:
$$
h_m(w^\top z) = \sum_{1 \leq i_1 \leq \dots \leq i_m \leq d} \sqrt{m \choose \beta} (\sigma_{i_1}^{1/2} \dots \sigma_{i_m}^{1/2}) (v_{i_1} \dots v_{i_m}) H_{i_1, \dots, i_m}(z)
$$
where $\beta$ must be identified once the parameters $m, i_1, \dots, i_m$ are known. At this point, since, by definition, $\sigma_j = r_\alpha^{-1} j^{-\alpha}$, then we will have:
$$
h_m(w^\top z) =  \sum_{1 \leq i_1 \leq \dots \leq i_m \leq d} \sqrt{m \choose\beta} \>  r_\alpha^{-m/2} (i_1 \dots i_m)^{-\alpha/2} (v_{i_1} \dots v_{i_m})  H_{i_1, \dots, i_m}(z)
$$
Putting everything together, we obtain:
$$
f_\star(x) = \sum_{m = 0}^{\infty} \sum_{1 \leq i_1 \leq \dots \leq i_m \leq d} g_m \sqrt{m \choose \beta} \>  r_\alpha^{-m/2} (i_1 \dots i_m)^{-\alpha/2} (v_{i_1} \dots v_{i_m})  H_{i_1, \dots, i_m}(z)
$$
Comparing with Eq.~\eqref{eq:stand}, it is easy to obtain the parameters we are interested in:
$$
\theta_{i_1, \dots, i_m} = g_m \sqrt{m \choose \beta} \> r_\alpha^{-m/2} (i_1 \dots i_m)^{-\alpha/2} (v_{i_1} \dots v_{i_m})
$$
In general, if the function is dense enough, then we can neglect the contribution of the diagonal terms approximating asymptotically:
$$
\theta_{i_1, \dots, i_m} = g_m \sqrt{m!} \> r_\alpha^{-m/2} (i_1 \dots i_m)^{-\alpha/2} (v_{i_1} \dots v_{i_m})
$$
\qed

%% file: Appendix/HermiteAnsatzNumerical.tex
\section{Hermite Ansatz Conjecture}\label{appendix:HermiteAnsatz}
In this section, we provide numerical evidence supporting the Hermite Ansatz introduced in Conjecture~\ref{conjecture:hermite_ansatz}. To this end, we proceed in two distinct ways.

\paragraph{Cumulative Target Function Comparison} First, we numerically compare the cumulative target energies $A(u)$ and $\widehat{A}(u)$, defined as:
\[
A(u) = \sum_{\beta : \lambda_\beta \leq u} \theta_\beta^2,
\]
\[
\widehat{A}(u) = \sum_{\beta : \lambda_\beta \leq u} c_\beta^2,
\]
where $\theta_\beta$ denotes the projection coefficient of the target function $f_\star$ onto the kernel eigenbasis, and $c_\beta$ denotes its coefficient when decomposed along the Hermite orthonormal basis. We focus on Single-Index Models as target functions, as this class allows us to leverage the exact Hermite decomposition given in Eq.~\eqref{eq:conversionSparse} (which holds even for finite $d$). In particular, to numerically compute $\widehat{A}(u)$:
\begin{itemize}
    \item We define a Single-Index Model target function using a random direction $v$ and a fixed Hermite expansion for the link function $g$.
    \item Using Eq.~\eqref{eq:conversionSparse}, we compute the coefficients $c_\beta$ indexed by $\beta$. We then sort these target weights according to the kernel eigenvalues.
    \item To obtain the kernel eigenvalues, we rely on the asymptotic spectrum expression in Eq.~\eqref{eq:SpectrumSpec}. While Eq.~\eqref{eq:conversionSparse} is valid for finite $d$, Eq.~\eqref{eq:SpectrumSpec} technically applies in the large-$d$ limit; nonetheless, we demonstrate that it provides a remarkably good approximation to the finite-$d$ spectrum.
\end{itemize}
To build $A(u)$, we rely on the \textit{empirical diagonalization} procedure. This Monte Carlo-inspired technique allows us to obtain an estimate of both the kernel spectrum and the target coefficients $\theta_\beta$ on the kernel eigenbasis. The comparison is illustrated in Fig.~\ref{fig:ComparisonCumulative} for different values of $\alpha$, providing strong numerical evidence for the validity of our conjecture in both the weakly and strongly anisotropic regime.
\begin{figure}[htbp]
    \centering
    \includegraphics[width=0.9\linewidth]{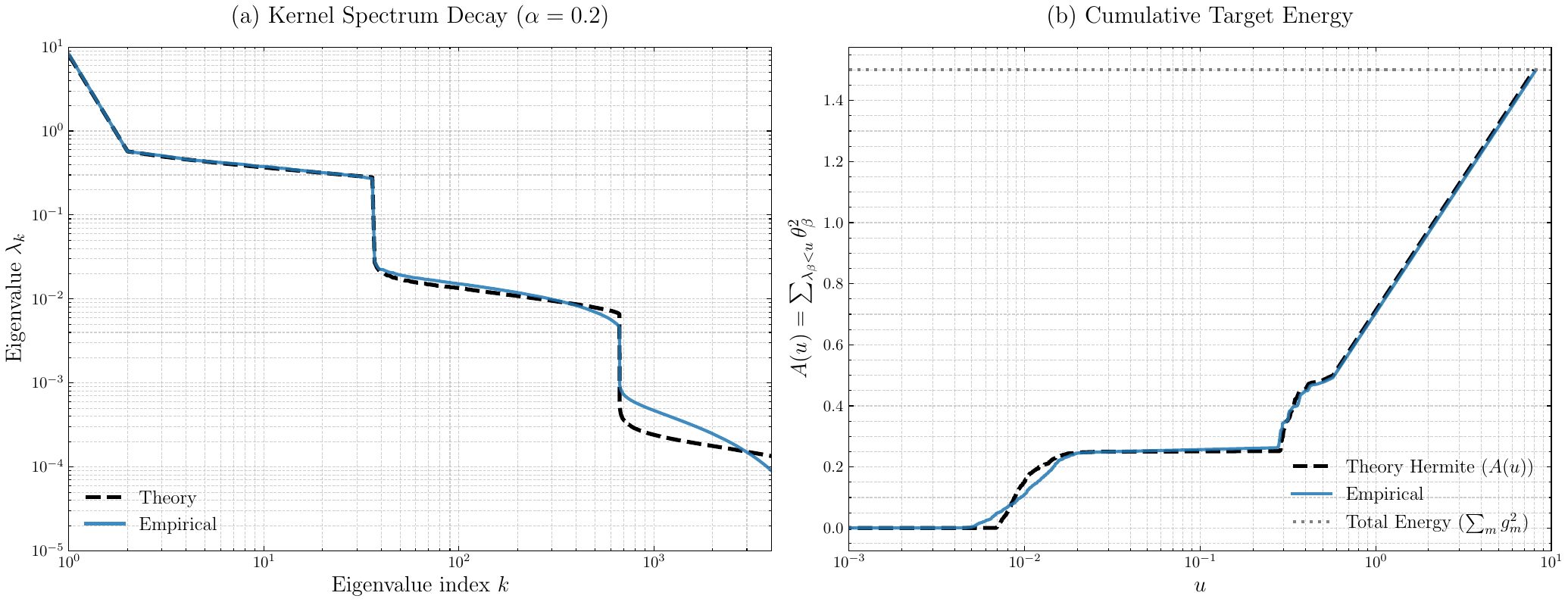}
    \includegraphics[width=0.9\linewidth]{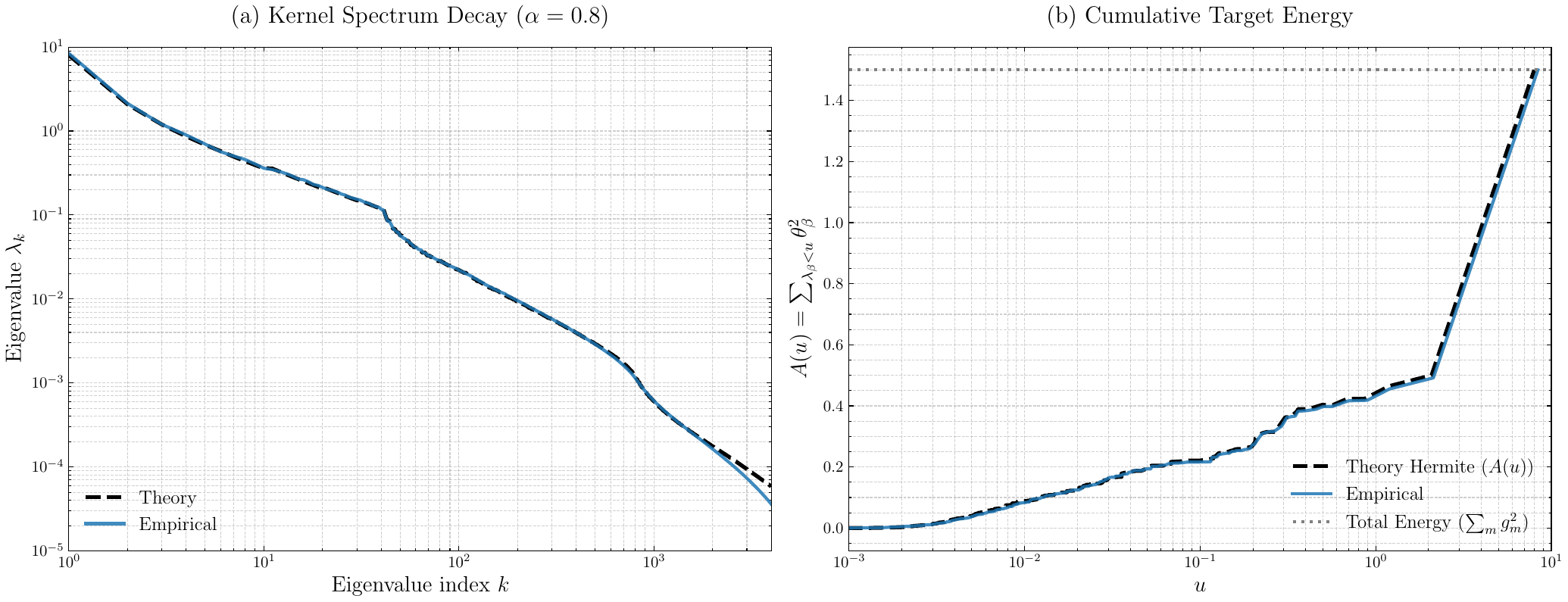}
    \includegraphics[width=0.9\linewidth]{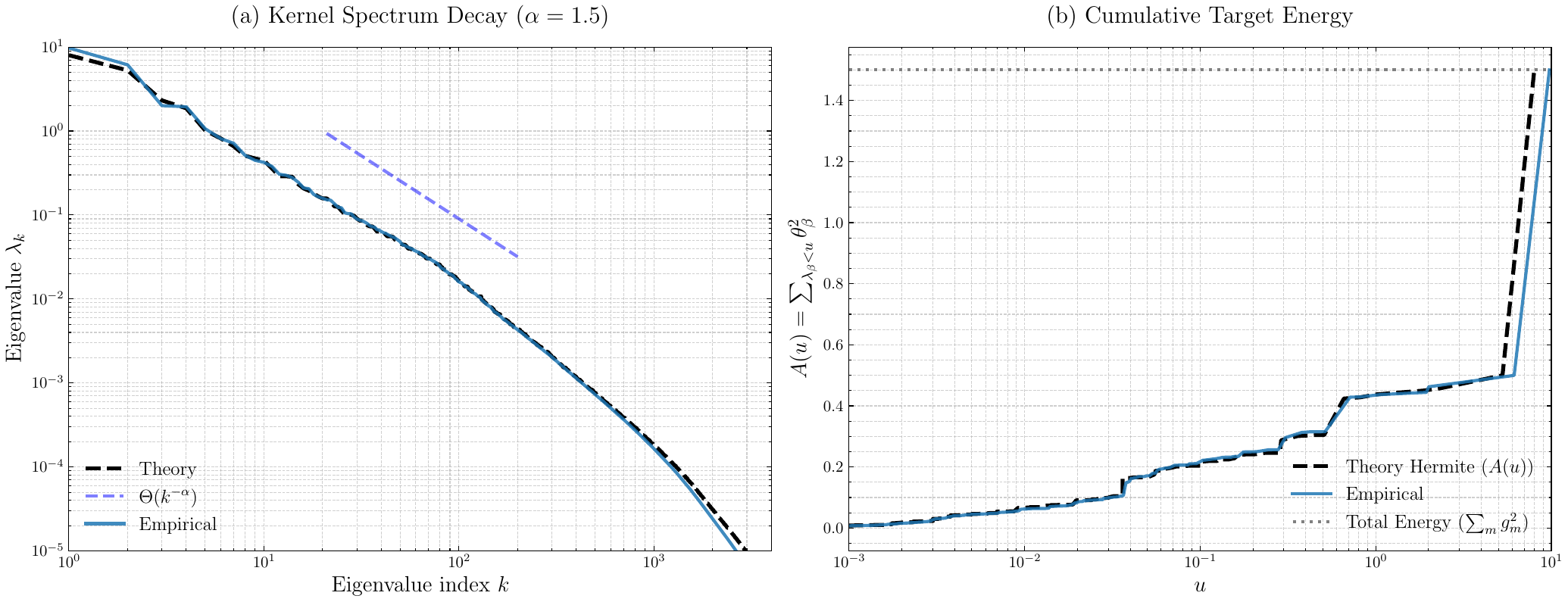}
    \caption{\textit{Kernel spectrum (left panel) and cumulative target function comparison (right panel) for different values of $\alpha$. In the left panel, we observe that the asymptotic theoretical spectrum obtained from Eq.~\eqref{eq:SpectrumSpec} closely matches the spectrum empirically extracted via empirical diagonalization. In the right panel, we see that the cumulative curves $A(u)$ computed on the Hermite basis and on the kernel eigenbasis are in close agreement. Simulation parameter: $d = 35$, polynomial kernel $k(x, x') = (2+\langle x, x'\rangle)^3$, random index $v$. The empirical diagonalization was performed using $M = 10000$ Monte-Carlo samples}}
    \label{fig:ComparisonCumulative}
\end{figure}

\paragraph{Learning SIMs}
As a second test, we can leverage the Hermite Ansatz to evaluate the theoretical predictions from Sec.~\ref{ssec:LearningSIM}, which are valid if we assume that the kernel eigenbasis coincides with the Hermite orthonormal basis in which the single-index target function is naturally decomposed. To do so, we empirically consider a KRR learning task trained over an anisotropic dataset and extract the empirical bias (setting $\sigma_\varepsilon = 0$ to eliminate the variance component). 

We first consider the case of a sparse single-index model with $j = 1$. The empirical versus asymptotic results are illustrated in Fig.~\ref{fig:SIMSparse} for both the weakly anisotropic regime ($\alpha = 0.35 < 1$) and the strongly anisotropic regime ($\alpha = 1.5 > 1$). In the former case, as $d \to \infty$, the bias converges to a non-zero value given by Eq.~\eqref{eq:BiasSparse_HeadWA}. In the $\alpha = 1.5$ case, we expect the bias to converge to zero according to the scaling laws theoretically obtained in Eqs.~\eqref{eq:BiasSparse_HeadSA} and \eqref{eq:BiasSparse_HeadSA1}, depending on the strength of the regularization.
\begin{figure}[htbp]
    \centering
    \includegraphics[width=\linewidth]{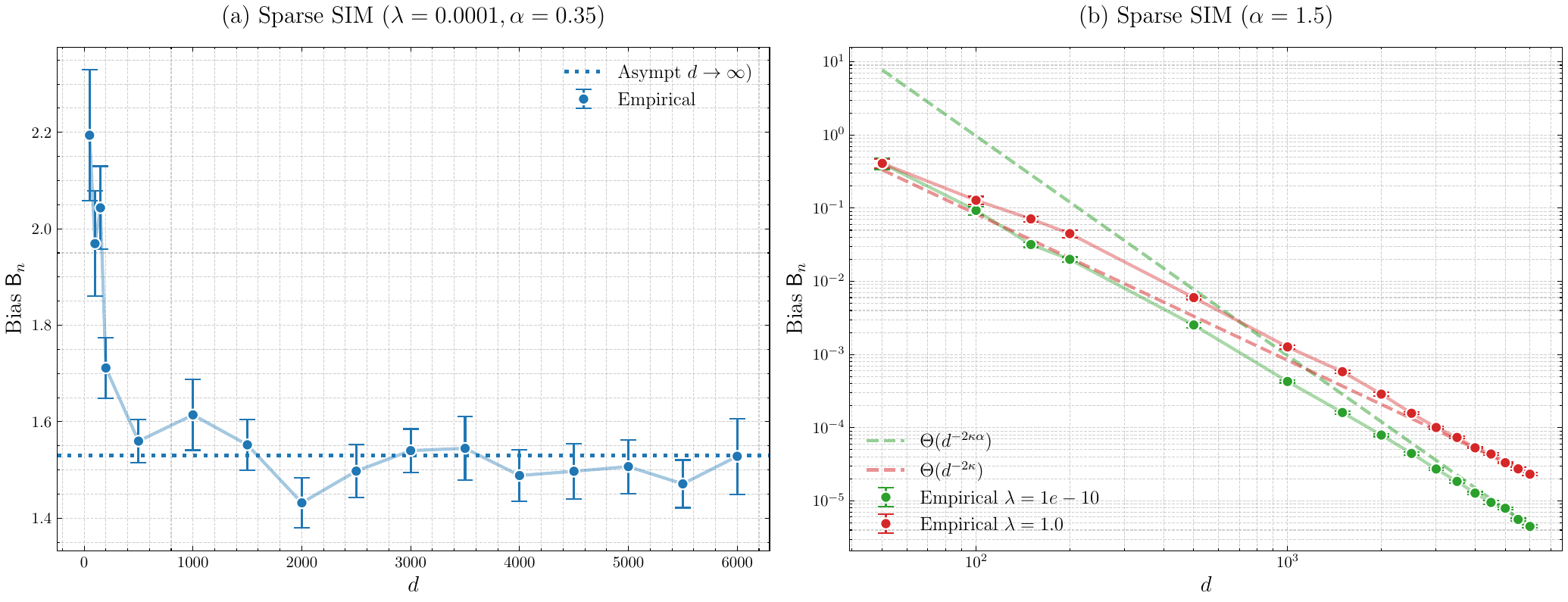}
    \caption{\textit{\textbf{Left:} Asymptotic versus empirical bias for a sparse SIM in the weakly anisotropic regime. \textbf{Right:} Asymptotic versus empirical bias for a sparse SIM in the strongly anisotropic regime. Simulation parameters: exponential kernel truncated at the fourth order, $\kappa = 1$, $\psi = 0.8$. Each empirical KRR run is repeated $20$ times and averaged, with errors representing the standard deviations.}}
    \label{fig:SIMSparse}
\end{figure}

Lastly, we consider the Single-Index power-law decaying case. In Fig.~\ref{fig:SIMPLWA}, we examine the weakly anisotropic regime ($\alpha < 1$). Since the effective parameter is defined as $\omega = \frac{1}{2}(\alpha + \gamma)$, depending on $\gamma$ we can fall into either the $\omega < \frac{1}{2}$ or $\omega > \frac{1}{2}$ regime. When $\gamma = 0.2$ (left plot), we have $\omega < \frac{1}{2}$, and the kernel learns up to the $\kappa$-th feature of the target function. Asymptotically, the empirical curves converge toward the theoretical and asymptotic predictions of Eq.~\eqref{eq:BiasWA>}. When $\gamma = 1$ or $\gamma = 2$ (right plot), we have $\omega > \frac{1}{2}$. From Eq.~\eqref{eq:BiasWA<}, we expect the kernel to learn up to the $\frac{\kappa}{1-\alpha} = \frac{1}{1-0.8} = 5$-th order features of the target function. In our simulation, the link function was set to $g(t) = \sum_{m=0}^3\frac{1}{m!} g_m He_m(t)$ (meaning $g_5 = 0$), so we expect the kernel to learn the entire function, and the bias to decay to zero. Indeed, we observe that for both values of $\gamma$, the bias drops to zero (note that we cannot observe the power-law decay here, as it corresponds to a sub-leading correction, whereas we are only tracking the leading term).
\begin{figure}[htbp]
    \centering
    \includegraphics[width=\linewidth]{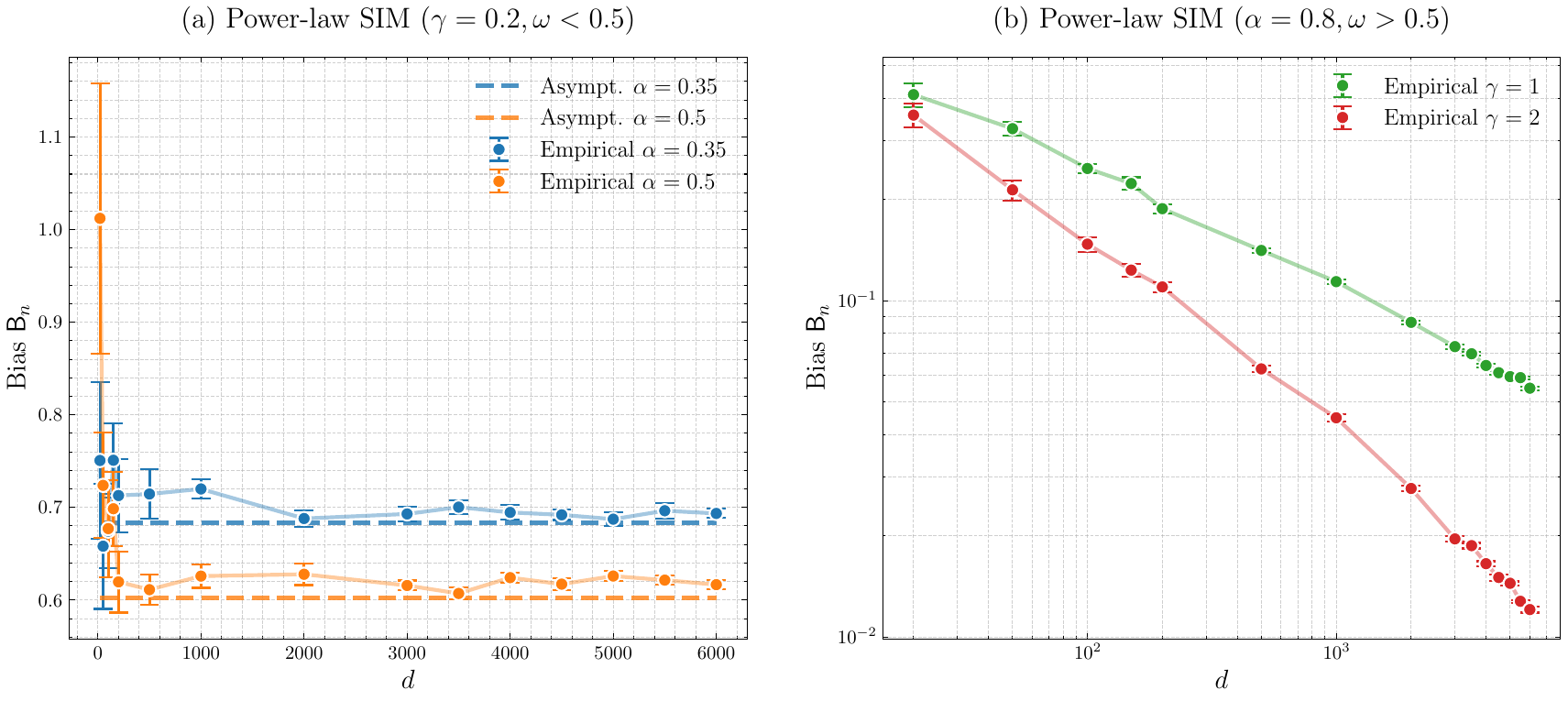}
    \caption{\textit{\textbf{Left:} Asymptotic versus empirical bias for a power-law SIM in the weakly anisotropic regime with an effective $\omega < 0.5$. \textbf{Right:} Asymptotic versus empirical bias for a power-law SIM in the weakly anisotropic regime for an effective $\omega > 0.5$. The link function $g$ is chosen such that its expansion is truncated at degree $3$ ($g(t) = \sum_{m=0}^3 g_m \text{He}_m(t)$), allowing the kernel to fully learn it in the $\omega > \frac{1}{2}$ regime. All other parameters match those of Fig.~\ref{fig:SIMSparse}.}}
    \label{fig:SIMPLWA}
\end{figure}

Finally, we analyze the strongly anisotropic regime (Fig.~\ref{fig:SIMPLSA}). Here, provided that $\gamma > 0$, we always operate in the fast-decay $\omega > \frac{1}{2}$ regime, meaning we should observe a scaling behavior of the bias, converging to zero. We consider two distinct scenarios: one where $\omega > \alpha + \frac{1}{2}$ and another where $\omega < \alpha + \frac{1}{2}$, allowing us to verify the different scaling behaviors predicted by Theorem~\ref{thm:bias:strong}, for $\lambda \to 0$ and $\lambda = 1$. Once again, we obtain visually compelling evidence supporting our conjecture.
\begin{figure}[htbp]
    \centering
    \includegraphics[width=\linewidth]{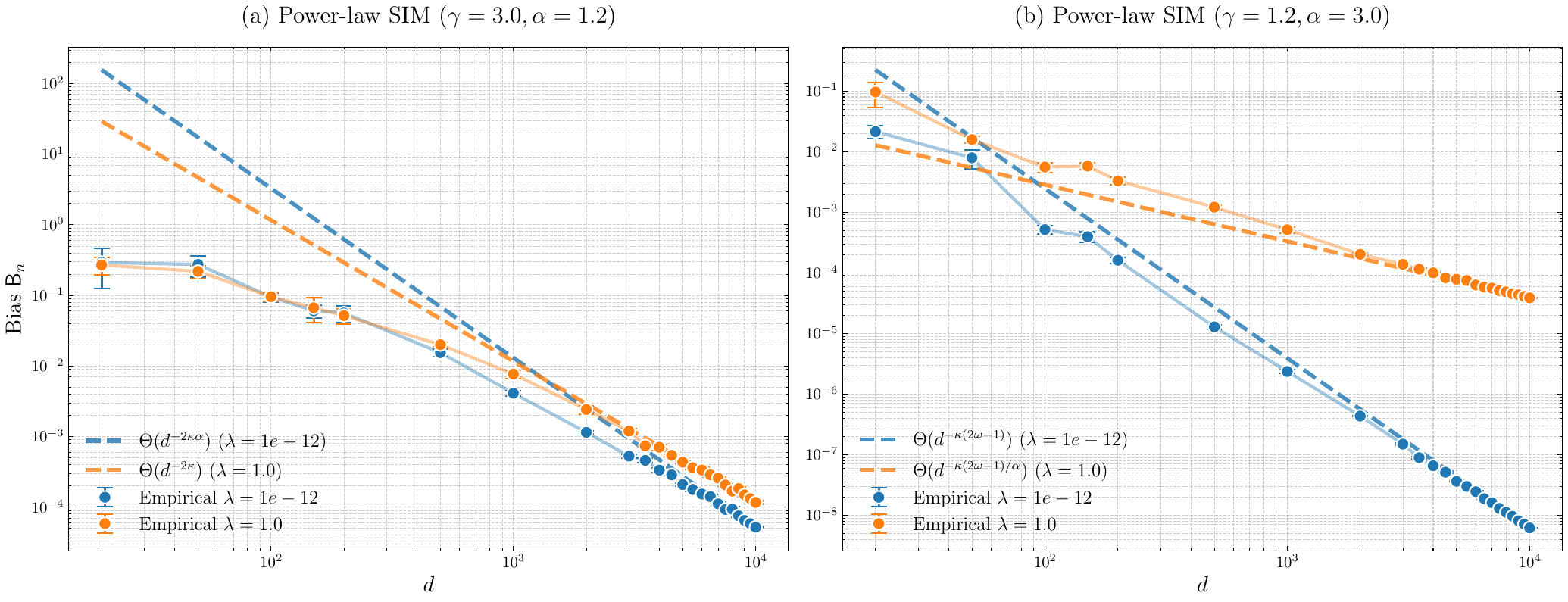}
    \caption{\textit{\textbf{Left:} Asymptotic versus empirical bias for a power-law SIM in the strongly anisotropic regime when $\omega > \alpha + \frac{1}{2}$. \textbf{Right:} Asymptotic versus empirical bias for a power-law SIM in the strongly anisotropic regime when $\omega < \alpha + \frac{1}{2}$. All other parameters match those of Fig.~\ref{fig:SIMSparse}.}}
    \label{fig:SIMPLSA}
\end{figure}

%% file: Appendix/Technical_Lemmas.tex
\newpage 
\section{Technical and Auxilliary Lemmas}

\begin{lemma} Let $\ell \in \NN$. Then there exists a constant $C_{\ell}$, independent of $d$, such that
\[
\left \|\| x \|^{2} \right \|_{L^{\ell}} = \mathbb{E}_{x \sim \mathcal{N}(0,\Sigma)} \left [ \left ( \| x \|^{2\ell} \right )^{\frac{1}{\ell}}\right ] \leq C_{\ell}. 
\]
\label{lemma:moments_norm}
\end{lemma}
\begin{proof}
    Indeed, we have: 
    \begin{align}
        \mathbb{E}_{x \sim \mathcal{N}(0,\Sigma)} \left [ \| x \|^{2\ell} \right ]^{\frac{1}{\ell}}  & = \mathbb{E}_{z \sim \mathcal{N}(0,I_{d})} \left [ \left ( \sum_{j=1}^{d} \sigma_j z_j^{2}\right )^{\ell}\right ]^{\frac{1}{\ell}} \\
        & \leq \sum_{j=1}^{d} \sigma_j \mathbb{E}_{z \sim \mathcal{N}(0,1)} \left [ z_j^{2\ell }\right ]^{\frac{1}{\ell}} \leq C_{\ell} \Tr(\Sigma) = C_\ell,
    \end{align}
    where in the second line we use the triangular inequality for the $L^\ell$- norm. 
\end{proof}

\begin{lemma}
    Let $m \in \NN$, and denote: 
    \[
    F_m(x) = K_{>m}(x, x), \quad \text{ and } \Tr(K_{>m}) = \sum_{j \geq m+1} \lambda_j.
    \]
    Then 
    \begin{equation}
        \mathbb{E} \left [ F_m(x)^{2}\right ]^{\frac{1}{2}} \leq C_D \Tr(K_{>m}), \quad \mathbb{E} \left [ (F_m - \Tr(K_{>m}))^{2}\right ]^{\frac{1}{2}} \leq (C_D+ 1) \Tr \left ( K_{>m}\right ). 
    \end{equation}
    \label{lemma:L2_bounds_diagonal}
\end{lemma}

\begin{proof}
    For the first inequality, we compute directly: 
    \begin{align}
        \mathbb{E} \left [ F_m(x)^{2}\right ] & = \mathbb{E} \left [ \left (\sum_{j > m} \lambda_j \Phi_j(x)^{2} \right )^{2}\right ] \\ 
        & = \sum_{j, \ell \geq m} \lambda_j \lambda_\ell \mathbb{E} \left [\Phi_j(x)^{2} \Phi_\ell(x)^{2}\right ] \\
        & \leq \sum_{j, \ell \geq m} \lambda_j \lambda_\ell \mathbb{E} \left [\Phi_j(x)^{4} \right ]^{\frac{1}{2}} \mathbb{E}\left [\Phi_\ell(x)^{4} \right ]^{\frac{1}{2}},
    \end{align}
    by Cauchy-Schwarz. Since the eigenfunctions are polynomials, hypercontractivity gives: 
    \begin{equation}
         \mathbb{E} \left [ F_m(x)^{2}\right ] \leq C_{D} (\sum_{j, \ell \geq m} \lambda_j )^{2} = C_{D} \Tr(K_{>m})^2. 
    \end{equation}
    Taking the square-root, we conclude the first inequality. For the second inequality, we have: 
    \begin{equation}
        \mathbb{E} \left [ (F_m - \Tr(K_{>m}))^{2}\right ]^{\frac{1}{2}} \leq \mathbb{E} \left [ F_m^{2}\right ]^{\frac{1}{2}} + \Tr(K_{>m}),
    \end{equation}
    and applying the first inequality we conclude: 
    \begin{equation}
        \mathbb{E} \left [ (F_m - \Tr(K_{>m}))^{2}\right ]^{\frac{1}{2}}  \leq (C_D +1 )\Tr(K_{>m}). 
    \end{equation}
\end{proof}

\begin{lemma}[Gaussian Poincaré Inequality, (\cite{boucheron_concentration2013}, Theorem 3.20]
Let $x \sim \mathcal{N}(0,\Sigma)$, and let $f:\RR^{d} \to \RR$ be a differentiable function. Then: 
\[
\mathrm{Var}(f(x)) \leq \|\Sigma\|_\mathrm{op}  \mathbb{E} \| \nabla f(x)\|_2^{2} . 
\]
\label{eq:Gaussian_Poincare}  
\end{lemma}